\documentclass[final,12pt]{colt2022} 

\title[How catastrophic can catastrophic forgetting be in linear regression?]{How catastrophic can catastrophic forgetting be in linear regression?}

\usepackage{times}
\coltauthor{%
 \Name{Itay Evron} \Email{itay@evron.me} \\
 \addr Department of Electrical and Computer Engineering, Technion, Haifa, Israel
 \AND
 \Name{Edward Moroshko} \Email{edward.moroshko@gmail.com} \\
 \addr Department of Electrical and Computer Engineering, Technion, Haifa, Israel
 \AND
 \Name{Rachel Ward} \Email{rward@math.utexas.edu} \\
 \addr Oden Institute for Computational Engineering and Sciences, University of Texas, Austin, TX
 \AND
 \Name{Nati Srebro} \Email{nati@ttic.edu} \\
 \addr Toyota Technological Institute at Chicago, Chicago IL, USA
 \AND
 \Name{Daniel Soudry} \Email{daniel.soudry@gmail.com} \\
 \addr Department of Electrical and Computer Engineering, Technion, Haifa, Israel
}

\usepackage{tikz,nicefrac}
\usepackage{cancel}
\usepackage{enumitem}
\usepackage{soul}
\usepackage{xcolor}
\usepackage{caption}
\usepackage{graphicx}
\usepackage{mathtools}

\setlist[enumerate]{label*=\arabic*.}
\setlist{leftmargin=5.5mm} 

\newcommand{\ie}{\textit{i.e., }}
\newcommand{\eg}{\textit{e.g., }}
\newcommand{\wrt}{\textit{w.r.t. }}

\newcommand{\vect}[1]{\boldsymbol{#1}}

\newcommand{\doubleN}{\mathbb{N}}

\newcommand{\naturals}{\doubleN^{+}}

\DeclareMathOperator*{\argmin}{argmin}
\DeclareMathOperator*{\argmax}{argmax}  
\DeclareMathOperator*{\range}{range}     
\DeclareMathOperator*{\kernel}{null}              
\newcommand{\vw}{\vect{w}}
\newcommand{\1}{\vect{1}}
\newcommand{\0}{\mat{0}}
\newcommand{\teacher}{\vw^{\star}}
\newcommand{\mat}[1]{\boldsymbol{#1}}
\newcommand{\X}{\mat{X}}

\newcommand{\M}{\mat{M}}
\newcommand{\U}{\mat{U}}
\newcommand{\Q}{\mat{Q}}
\newcommand{\V}{\mat{V}}
\newcommand{\I}{\mat{I}}
\newcommand{\A}{\mat{A}}
\newcommand{\B}{\mat{B}}
\newcommand{\x}{\vect{x}}
\newcommand{\vv}{\vect{v}}
\newcommand{\vu}{\vect{u}}
\newcommand{\y}{\vect{y}}
\newcommand{\w}{\vw}

\newcommand{\rank}{\operatorname{rank}}
\newcommand{\tr}{\operatorname{tr}}
\newcommand{\trace}{\tr}
\def\mSigma{{\mat{\Sigma}}}
\def\mLambda{{\mat{\Lambda}}}
\def\mP{{\mat{P}}}
\newcommand\smallcdots{\cdotp\!\hskip.6pt\cdotp\!\hskip.6pt\cdotp}

\def\reals{\mathbb{R}}
\def\complex{\mathbb{C}}

\newcommand{\norm}[1]{\left\Vert{#1}\right\Vert}
\newcommand{\abs}[1]{\left\vert{#1}\right\vert}
\newcommand{\cnt}[1]{\left[{#1}\right]}
\newcommand{\explain}[1]{\left[\substack{#1}\right]}
\newcommand{\expectation}{\mathop{\mathbb{E}}}

\newcommand{\prn}[1]{\left({#1}\right)}
\newcommand{\bigprn}[1]{\big({#1}\big)}
\newcommand{\Bigprn}[1]{\Big({#1}\Big)}
\newcommand{\biggprn}[1]{\bigg({#1}\bigg)}
\newcommand{\tprn}[1]{({#1})}
\newcommand{\smallnorm}[1]{\Vert{#1}\Vert}
\newcommand{\tnorm}[1]{\smallnorm{#1}}
\newcommand{\bignorm}[1]{\big\Vert{#1}\big\Vert}
\newcommand{\Bignorm}[1]{\Big\Vert{#1}\Big\Vert}
\newcommand{\biggnorm}[1]{\bigg\Vert{#1}\bigg\Vert}
\newcommand{\hop}{*}
\newcommand{\ang}[2]{\theta_{\tau\left({#1,#2}\right)}}

\newcommand{\tsum}{{\sum}}
\newcommand{\ball}[1]{\mathcal{B}^{d}}

\usepackage{xspace}

\newcommand{\cf}{catastrophic forgetting\xspace}
\newcommand{\Coll}{\mathcal{S}_T}
\newcommand{\coll}{S}
\newcommand{\sol}[1]{\mathcal{W}_{#1}}
\newcommand{\itr}{t}
\newcommand{\bigO}{\mathcal{O}}

\newcommand{\expfor}{\bar{F}}
\newcommand{\maxrank}{r_{\max}}

\newcommand{\avgrank}{r_{\operatorname{avg}}}
\newcommand{\decrease}[2]{\Delta_{#1}(#2)}

\newtheorem{claim}[theorem]{Claim}
\newtheorem{assumption}[theorem]{Assumption}

\newtheorem{property}{Property}

\newcommand{\remove}[1]{REMOVE!}

\newenvironment{proof-sketch}{\noindent{\bf Proof sketch.}}{}

\def\figref#1{Figure~\ref{#1}}

\def\secref#1{Section~\ref{#1}}
\def\Secref#1{Section~\ref{#1}}

\def\eqref#1{Eq.~(\ref{#1})}

\def\lemref#1{Lemma~\ref{#1}}

\def\thmref#1{Theorem~\ref{#1}}
\def\Thmref#1{Theorem~\ref{#1}}
\def\corref#1{Corollary~\ref{#1}}
\def\clmref#1{Claim~\ref{#1}}
\def\defref#1{Definition~\ref{#1}}
\def\propref#1{Prop.~\ref{#1}}

\def\appref#1{Appendix~\ref{#1}}
\def\appsref#1{App~\ref{#1}}

\definecolor{residuals}{RGB}{200,27,80}

\definecolor{itay}{RGB}{210,30,220}
\definecolor{rachel}{RGB}{40,180,220}
\definecolor{red2}{RGB}{240,17,17}

\definecolor{snote}{RGB}{50,200,50}

\newcommand{\unnotice}[1]{#1}

\makeatletter
\newenvironment{recall}[1][\proofname]{\par
\normalfont \topsep6\p@\@plus6\p@\relax
\trivlist
\item\relax
{\bfseries
Recall #1}%
{\bfseries\@addpunct{.}}\hspace\labelsep\ignorespaces
}
\makeatother

\long\def\supptitle#1{

   \gdef\@runningheadingerrortitle{0}

   \ifnum\statePaper=0
    {
     \gdef\@runningtitle{Manuscript under review by AISTATS \@conferenceyear}
    }
   \fi

   \ifnum\statePaper=1
   {
   \ifx\undefined\@runningtitle
    {
    \gdef\@runningtitle{#1}
    }
   \fi
   }
   \fi

   \ifnum\@runningheadingerrortitle=0
         {
         \global\setbox\titrun=\vbox{\small\bfseries\@runningtitle}%
         \ifdim\wd\titrun>\textwidth%
            {\gdef\@runningheadingerrortitle{2}
             \gdef\@messagetitle{Running heading title too long}
            }%
         \else\ifdim\ht\titrun>10pt
              {\gdef\@runningheadingerrortitle{3}
              \gdef\@messagetitle{Running heading title breaks the line}
              }%
              \fi
          \fi
         }
    \fi

   \ifnum\@runningheadingerrortitle>0
     {
        \fancyhead[CE]{\small\bfseries\@messagetitle}
        \ifnum\@runningheadingerrortitle>1
           \typeout{}%
           \typeout{}%
           \typeout{*******************************************************}%
           \typeout{Running heading title exceeds size limitations for running head.}%
           \typeout{Please supply a shorter form for the running head}
           \typeout{with \string\runningtitle{...}\space just after \string\begin{document}}%
           \typeout{*******************************************************}%
           \typeout{}%
           \typeout{}%
        \fi
     }
  \else
     {
          \fancyhead[CE]{\small\bfseries\@runningtitle}
     }
  \fi

  \hsize\textwidth
  \linewidth\hsize \toptitlebar {\centering
  {\Large\bfseries #1 \par}}
 \bottomtitlebar
}

\begin{document}

 \maketitle

\begin{abstract}
To better understand catastrophic forgetting, we study fitting an overparameterized linear model to a sequence of tasks with different input distributions.
We analyze how much the model forgets the true labels of earlier tasks after training on subsequent tasks, obtaining exact expressions and bounds.
We establish connections between continual learning in the linear setting and two other research areas -- 
alternating projections and the Kaczmarz method.
In specific settings, we highlight differences between forgetting and convergence to the offline solution as studied in those areas. %
In particular, when $T$ tasks in $d$ dimensions are presented cyclically for $k$ iterations, we prove an upper bound of $T^2\min\{1/\sqrt{k},d/k\}$ on the forgetting.
This stands in contrast to the convergence to the offline solution, which can be arbitrarily slow according to existing alternating projection results.
We further show that the $T^2$ factor can be lifted when tasks are presented in a random ordering.

\end{abstract}

\section{Introduction}

Continual learning or lifelong learning is a machine learning setting where data from different tasks are presented sequentially to the learner.
The goal is to adapt the model to new tasks while preserving its performance on previously-learned tasks \citep{schlimmer1986case,thrun1995lifelong,parisi2019continual}. 
A key challenge in continual learning is the \textit{catastrophic forgetting} phenomenon
\citep{mccloskey1989catastrophic,ratcliff1990rehearsal,goodfellow2013empirical,ramasesh2020anatomy}, wherein adaptation of models to fit to new tasks often (unsurprisingly) leads to degradation in performance on previous tasks.

Despite recent advances in theoretical understanding of continual learning \citep{doan2021NTKoverlap,bennani2020generalisationWithOGD,KnoblauchHD20,lee2021taskSimilarity,Asanuma_2021}, 
catastrophic forgetting is not \emph{fully} understood even in simple models. 
Consider sequentially learning from a stream of tasks. 
One should ask: 
what are the best and worst-case sequences of tasks? 
\linebreak
How does the similarity between tasks effect catastrophic forgetting?
What can be said analytically about the benefits to revisiting (\ie replaying) tasks? 
When does forgetting truly becomes ``catastrophic", so that it is impossible to learn \emph{all} tasks sequentially?

In this work, we aim to theoretically characterize the \emph{worst-case} catastrophic forgetting
in overparameterized linear regression models.\footnote{ We believe it is necessary to do so before moving to more complex models. 
Moreover, any linear regression result can be applied to complex models (e.g., deep networks) in the neural kernel regime (NTK), as in \cite{doan2021NTKoverlap,bennani2020generalisationWithOGD}.}
To this end, we sequentially fit a linear model to tasks observed in some ordering.
We analyze the forgetting convergence of linear regressors obtained by GD/SGD which is trained \emph{to convergence} on data from the current task.

\pagebreak

We explain how in this linear setting, 
continual learning repeatedly \emph{projects} previous solutions onto the solution spaces of newer tasks.
Interestingly, we show that, under a realizability assumption, these projections contract the distance to the minimum norm offline solution that solves \emph{all} tasks.
We analyze this contraction and the resulting convergence.

A setting similar to ours has been extensively studied in the \emph{alternating projections} literature. 
There, a vector is iteratively projected onto closed subspaces (or convex sets in general),
in order to find a solution in their intersection. 
For instance, \citet{kayalar1988error} studied Halperin's cyclic setting \cite{halperin1962product}
and analyzed the convergence of 
$\bignorm{
\tprn{\mP_2 \mP_1}^{n}\!-\!\mP_{1\cap\,2}}^{2}$,
where $\mP_1$, $\mP_2$, and $\mP_{1\cap\,2}$ are orthogonal projections onto subspaces $H_1$, $H_2$, and $H_1\!\cap\!H_2$ (respectively).
In contrast, our goal is to analyze the projection \emph{residuals} (\ie the forgetting). 
That is, we mainly focus on
$\bignorm{
\tprn{\I-\mP_1}
\tprn{\mP_2 \mP_1}^{n}}^{2}$,
rather than on the convergence to 
$H_1\!\cap\!H_2$
(\ie an offline solution).
Due to this difference,
we are able to derive \emph{uniform} data-independent upper bounds on forgetting, even when convergence to the offline solution is arbitrarily slow and ``traditional'' bounds become trivial.

Moreover, our fitting procedure
can also be seen as a Kaczmarz method \citep{karczmarz1937angenaherte},
where one solves a linear equation system $\A\x\!=\!\vect{b}$
by iteratively solving subsets (\ie tasks) of it.
Here also, typical bounds involve data-dependent properties like the spectrum of $\A$,
which are trivial in the worst case.

\paragraph{Our Contributions.} 
We thoroughly analyze catastrophic forgetting in linear regression
\unnotice{optimized by the plain memoryless (S)GD algorithm (\ie without actively trying to mitigate forgetting), and:}
\vspace{-.5mm}
\begin{itemize}[leftmargin=5mm]\itemsep.8pt
    \item Identify cases where there is no forgetting.
    \item Show that without any restrictions, one can construct task sequences where forgetting is maximal and essentially \emph{catastrophic}.
    \item Connect catastrophic forgetting
    to a large body of research on alternating projections and the Kaczmarz method.
    Then, we use this perspective to investigate
    the \emph{worst-case} forgetting when $T$ tasks in $d$ dimensions are seen \emph{repeatedly} for $k$ iterations, 
    under two different orderings:
    \begin{itemize}[leftmargin=5mm,topsep=0pt]\itemsep.8pt
        \item \textbf{Cyclic task orderings.} 
        %
        For $T\!\ge\!3$, we uniformly upper bound the forgetting by $\tfrac{T^2}{\sqrt{k}}$ and $\tfrac{T^2d}{k}$, and lower bound it by $\tfrac{T^2}{k}$.
        To the best of our knowledge,
        our analysis uncovers novel bounds for residuals in the cyclic block Kaczmarz setting \citep{elfving1980block} as well.

        \item \textbf{Random task orderings.} 
        We upper bound the expected forgetting by $\tfrac{d}{k}$,
        independently of $T$.
    \end{itemize}
    \item Analyze the effect of similarity, or angles, between two consecutive tasks on the forgetting.
    We find that after seeing these tasks \emph{once} ---
    intermediate angles are most prone to forgetting,
    but after \emph{repeating} the tasks ---
    small angles (\ie nearly-aligned tasks) cause the highest forgetting.
\end{itemize}

\section{Problem setting}
\label{sec:setting}
\paragraph{Motivating example.} Before we formalize our problem, we give a motivating example to keep in mind. Suppose we are interested to learn a predictor for pedestrian detection in an autonomous car. 
This detector is required to operate well in $T$ geographically distant environments 
(\eg different countries), or ``tasks". 
To train the detector, we need to drive the car around, but can do this only in one environment at a time. We do this until the detector has good performance in that environment. 
Then we ship the car to another environment, drive it there to train the detector, and so on (potentially revisiting past environments). Notably, while the overall problem remains constant (pedestrian detection), each time we need to solve it for a potentially  different landscape (\eg city, forest, desert), which can radically change the input distribution. 
The question we aim to answer is: 
\linebreak when would the detector be able to work well in all the environments it visited during training, even though it only observed them sequentially?


\paragraph{General setting.}
\unnotice{
We consider fitting a linear model
$f(\x)\!=\!\w^{\top}\x$, 
parameterized by $\w\!\,\in\!\,\reals^d$, 
\linebreak
on a sequence of tasks
originating from an arbitrary set of $T$ regression problems
$\big\{\big(\X_{m},\y_{m}\big)\big\}_{m=1}^T$.
\linebreak
During~$k$ iterations, 
tasks are seen
according to a \emph{task ordering}
$\tau\!\!:\naturals\!\!\to\!\left[T\right]$.
That is, 
the learner is presented with a sequence 
$\big(\X_{\tau(1)},\y_{\tau(1)}\big),
\big(\X_{\tau(2)},\y_{\tau(2)}\big),
\dots,
\big(\X_{\tau(k)},\y_{\tau(k)}\big)$.
Starting from ${\w_{0}=\vect{0}_d}$,
the model is fitted sequentially, 
yielding a sequence of iterates $\w_{1},...,\w_{k}$.  
At each iteration~$\itr$ we obtain $\w_{\itr}$ by fitting the \emph{current} task 
$\big(\X_{\tau(\itr)},\y_{\tau(\itr)}\big)$, \emph{without} access to other tasks.
}

\paragraph{Tasks.}
Each \emph{task} $m=1,2,\dots,T$ corresponds to a regression problem
$\big(\X_{m},\y_{m}\big)$,
\ie 
to a data matrix $\X_{m}\in\reals^{n_m\times d}$ with $n_m$ samples in $d$ dimensions
and a vector
$\y_{m}\!\in\!\reals^{n_m}$
of the corresponding labels.
We focus on the overparametrized regime, 
where every task $\big(\X_{m},\y_{m}\big)$ is rank deficient, 
\ie $r_m\!\,\triangleq\!\,\rank{\tprn{\X_m}}\!\,<\,\!d$, thus admitting infinitely-many linear models that perfectly fit the data.

\paragraph{Notation.} Denote by $\mat{A}^{+}$ the Moore–Penrose inverse of $\mat{A}$. 
Denote by $\mP_m$ the orthogonal projection onto the \emph{null space} of a matrix $\X_m$, 
\ie $\mP_m=\I-\X_m^{+}\X_m$. 
Denote the $\ell^2$ norm of a vector by $\norm{\vect{v}}$,
and the \emph{spectral} norm of a matrix
 by $\norm{\mat{A}}$.
Denote the $d$-dimensional closed Euclidean ball as
 $\ball{d} \triangleq \!
 \big\{\vv\in\reals^d \,\,\big| \norm{\vv} \le 1\big\}$.
To avoid ambiguity, we explicitly exclude zero from the set of natural numbers and use
$\naturals\!\!\triangleq\!\doubleN\setminus\!\left\{0\right\}$.
Finally, denote the natural numbers from $1$ to $T$ by 
$\cnt{T}\!\triangleq\!\left\{1,\ldots,T\right\}$.


\subsection{Task collections}

A \emph{task collection} is an (unordered) set of $T$ tasks,
\ie
$\coll=\big\{\big(\X_{m},\y_{m}\big)\big\}_{m=1}^T$.
Throughout the paper, 
we focus on task collections from the set:
\begin{align*}
\begin{split}
    \Coll
    \triangleq
    \Big\{
        \big\{
            \big(\X_{m},\y_{m}\big)
        \big\}_{m=1}^{T}
        ~\Big|~
        \underbrace{
            \forall m\!:~ 
            \Vert{\X_{m}}\Vert \le1
        }_{\text{Assumption \ref{assume:bounded_data}}},~
        \underbrace{
        \exists \w\in\ball{d}\!:~
        \y_{m}\!=\!\X_{m}\w,~
        \forall m\in\cnt{T}
        }_{\text{Assumption \ref{assume:realizability}}}
        \Big\}.
\end{split}
\end{align*}
That is, we make the two following
assumptions on the task collections.
\begin{assumption}[Bounded data]
\label{assume:bounded_data}
Singular values of all data matrices are bounded (w.l.o.g., by $1$).
\end{assumption}
\begin{assumption}[Realizability]
\label{assume:realizability}
Tasks are jointly realizable by a linear predictor
with a bounded norm (w.l.o.g.,  bounded by $1$).
\end{assumption}

Assumption~\ref{assume:bounded_data} is merely a technical assumption on the data scale, 
simplifying our presentation.
Assumption~\ref{assume:realizability} on joint realizability of tasks
is essential to our derivations, especially for \eqref{eq:affine-rule}. 
We note that this is a reasonable assumption in a highly overparameterized models like wide neural networks in the NTK regime, or in noiseless settings with an underlying linear model.
%


\subsection{Forgetting and catastrophic forgetting}
%
%
%
%
We fit the linear model sequentially on tasks $\tau(1),\ldots,\tau(k)$.
Starting from the solution to its preceding task,
each task is fitted \emph{without} access to previous tasks.
%
%
The goal in continual learning is to not \emph{forget} what we learned on previous tasks. 
For example, in the motivating example given in the beginning of the section, we want the pedestrian detector to still be able to work well in previous environments after we train it in a new environment. 

\pagebreak

Formally,
like in a recent work
\citep{doan2021NTKoverlap},
we measure the forgetting of $\w$ on the task presented at iteration $\itr$ by the squared loss 
$\mathcal{L}_{\itr}\left(\w\right)
\!\triangleq\!
\bignorm{\X_{\tau(\itr)}\w\!-\!\y_{\tau(\itr)}}^2$, 
and define the 
forgetting as follows.
\begin{definition}[Forgetting]
\label{def:forgetting}
Let $\coll\!=\!\big\{\!\big(\X_{m},\y_{m}\big)\!\big\}_{m=1}^T\!\in\!\Coll$ 
be a task collection fitted according to an ordering 
$\tau\!\!:\naturals\!\!\to\!\left[T\right]$.
The~forgetting at iteration $k$
is the average
loss
of already~seen~tasks,
\ie
\vspace{-1mm}
\begin{align*}
F_{\tau,\coll}\left(k\right)
=
\frac{1}{k}\tsum_{\itr=1}^{k}
\mathcal{L}_{\itr}\left(\w_k\right)
=
\frac{1}{k}\tsum_{\itr=1}^{k}
\Bignorm{\X_{\tau(\itr)}\w_{k}-\y_{\tau(\itr)}}^2 ~.
\end{align*}
\end{definition}
In words, we say that the model ``forgets'' the labels of a task seen at iteration $\itr$ after fitting $k$ tasks, 
if $\w_{k}$ does not perfectly predict
$\y_{\tau(\itr)}$ from $\X_{\tau(\itr)}$, 
\ie $\mathcal{L}_\itr(\w_{k})>0$. 
\unnotice{We note in passing that
\emph{forgetting} should not be confused with \emph{regret} since they can exhibit different behaviors
(see \secref{sec:related}).}

In the following, we capture the worst possible convergence behavior under a given ordering.
\begin{definition}[Worst-case forgetting]
\label{def:worst-case}
The worst-case forgetting of a task ordering $\tau$ after $k$ iterations is the 
maximal forgetting at iteration $k$ on \emph{any} task collection in $\Coll$, \ie
$
\sup_{\coll \in \Coll}F_{\tau,\coll}\left(k\right)
$.
\end{definition}
For specific task orderings, we are able to bound this worst-case forgetting by non-trivial uniform (data-independent) bounds that converge to 0 with the number of iterations $k$.
When the worst-case forgetting does not converge to 0, we say that forgetting is \emph{catastrophic}.
More formally,
\begin{definition}[Catastrophic forgetting of a task ordering]
Given an ordering $\tau$ over $T$ tasks, 
the forgetting is \emph{not} catastrophic if 
$\lim_{k\to\infty} 
\tprn{\sup_{\coll\in \Coll}
{F_{\tau,\coll}}(k)} = 0$.
\end{definition}
In this paper we focus on the quantities $F_{\tau,\coll}(k)$ and $\sup_{\coll \in \Coll}\!F_{\tau,\coll}(k)$, 
aiming to answer the following questions: 
under what conditions on $\tau$ and $\coll$ there is no forgetting, 
\ie $F_{\tau,\coll}\left(k\right)=0$? 
What can we say about 
$\sup_{\coll \in \Coll}\!F_{\tau,\coll}(k)$
for general task orderings?
Can forgetting be catastrophic when fitting a finite number of tasks in a cyclic or random ordering?

\subsection{Fitting procedure}
\label{sec:fitting_procedure}

%
%
Our ultimate goal is to minimize the forgetting (\defref{def:forgetting}). 
To this end, in this paper we analyze the following fitting procedure of tasks, corresponding to the simplest continual learning setting.
At each iteration $\itr=1,..,k$, 
we start from the previous iterate $\w_{\itr-1}$
and run (stochastic) gradient \linebreak
descent \emph{to convergence}
so as to minimize the squared loss
on the \emph{current} task,
\linebreak
\ie 
${\mathcal{L}_{\itr}\left(\w\right)
=
\bignorm{\X_{\tau(\itr)}\w-\y_{\tau(\itr)}}^2}$,
obtaining the new iterate $\w_{\itr}$.

Since we work in the overparameterized regime, where $r_{\itr}<d$
for all $\itr$, each $\w_{\itr}$ perfectly fits its corresponding task $\tau(\itr)$, \ie $ \X_{\tau(\itr)}\w_{\itr}=\y_{\tau(\itr)}$. 
In addition, for each task, (S)GD at convergence returns the unique solution which implicitly minimizes the $\ell^2$ distance to the initialization \citep{zhang2017understanding,gunasekar2018characterizing}, hence we can express $\w_{\itr}$ as the unique solution to the following optimization problem:
\begin{align}
\begin{split}
\label{min_l2}
    \w_{\itr}
    =
    \argmin_{\w}&\left\Vert\w-\w_{\itr-1}\right\Vert~, ~~~~~
    \text{s.t.}~~\X_{\tau(\itr)}\w=\y_{\tau(\itr)}~.
\end{split}
\end{align}

\paragraph{Our iterative update rule.}
The solution to the above optimization problem 
is given by
%
\begin{align}
\label{eq:update_rule}
    \w_{\itr}
    =
    \w_{\itr-1}+\X_{\tau(\itr)}^{+}(\y_{\tau(\itr)}-\X_{\tau(\itr)}\w_{\itr-1})~.
\end{align}
%
When $\X_{\tau(\itr)}$ contains one sample only (\ie its rank is 1), 
our update rule is equivalent to those of the Kaczmarz method \citep{karczmarz1937angenaherte} and the NLMS algorithm \citep{slock1993convergence,Haykin:2002}.
Moreover, it can be seen as an SGD update
on the forgetting from \defref{def:forgetting},
\ie
${\w_{\itr}
=
\w_{\itr-1} \!+\!
\tfrac{1}{\smallnorm{\x_{\tau(\itr)}}^2}
\bigprn{
y_{\tau(\itr)}\!-
\x_{\tau(\itr)}^{\top}\w_{\itr-1}}
\x_{\tau(\itr)}}$.
Given~tasks with \emph{many} samples (rank greater than $1$), our rule is equivalent to that of the block Kaczmarz method \cite{elfving1980block}.
%
%
We discuss these connections in \secref{sec:related}.
%


\paragraph{Minimum norm offline solution.}
Under the realizability assumption~\ref{assume:realizability}
there might be infinitely-many \emph{offline solutions} that perfectly fit \emph{all} tasks.
To facilitate our discussion, 
we focus on the minimum~$\ell^2$-norm offline solution $\teacher$ (often referred to as the \emph{offline solution} for brevity), 
\ie
\vspace{-1mm}
\begin{align}
\begin{split}
\label{eq:minimum-norm-solution}
    \teacher 
    ~\triangleq~
    {\argmin}_{\w \in \reals^{d}}\,
    \tfrac{1}{2}\norm{\w}^2,
    ~~\text{s.t.}~~
    \X_{m}\w=\y_{m},~
    \forall m\in\cnt{T}
    ~.
\end{split}
\end{align}

\paragraph{Task solution spaces.}
Finally, from \eqref{min_l2} we see that at the end of the $\itr$-th iteration, the iterate $\w_\itr$ must lie in the \emph{solution space} of task $\tau(\itr)$, which is an affine subspace defined as follows
\vspace{-1mm}
\begin{align}
\label{eq:solution_spaces}
\sol{\tau(\itr)}
\triangleq
\big\{
\w~
\big|\,
\X_{\tau(\itr)}\w=\y_{\tau(\itr)}
\big\}
=
\teacher + \kernel(\X_{\tau(\itr)})~.
\end{align}

\section{Forgetting dynamics}
\label{sec:forgetting_in_linear}
To analyze forgetting, 
we emphasize the projective nature of learning.
We rewrite the update rule from
\eqref{eq:update_rule}
by employing the realizability Assumption~\ref{assume:realizability}
and plugging in $\X_m \teacher\!=\!\y_m$ into
that equation.
Then, we subtract $\teacher$ from both sides,
and reveal an equivalent affine update rule, \ie
%
\begin{align}
    \w_{\itr}-\teacher
    &=
    \big(
    \I -
    \X_{\tau(\itr)}^{+}\X_{\tau(\itr)}
    \big)
    \w_{\itr-1}
    +
    \X_{\tau(\itr)}^{+}\X_{\tau(\itr)}\w^\star
    -\w^\star
    ~\triangleq~
    \mP_{\tau(\itr)}
    \big(\w_{\itr-1}-\teacher\big)~,
\label{eq:affine-rule}    
\end{align}
where we remind that
$\mP_{m}
\triangleq
\I-\X_{m}^+\X_{m}$ is the projection operator on the solution space $\sol{m}$.

\paragraph{Geometric interpretation.}
Using properties of pseudo-inverses and operator norms we get that
$\forall m\!\in\!\cnt{T}\!:\,\bignorm{\X_m\vu}^2
=\bignorm{\X_m\X_m^+\X_m\vu}^2
\le\bignorm{\X_m}^2\bignorm{\X_m^+\X_m\vu}^2
=
\bignorm{\X_m}^2\bignorm{\bigprn{\I-\mP_m}\vu}^2$. 
Then, we recall that $\big\Vert{\X_{m}}\big\Vert\!\le\!1$ 
(Assumption~\ref{assume:bounded_data}),
and reveal that the forgetting can be seen as the mean of the squared \emph{residuals} from projecting $\bigprn{\w_{k}\!-\!\teacher}$
onto previously-seen solution spaces. That is,
\vspace{-2mm}
\begin{align}
\label{eq:geometric}
    F_{\tau,S}\left(k\right)
    &= 
    \frac{1}{k}\,
    \tsum_{\itr=1}^{k}
    \bignorm{\X_{\tau(\itr)}  
    \bigprn{\w_{k}-\teacher}}^2 
    \le
    \frac{1}{k}\,
    \tsum_{\itr=1}^{k}
    \bignorm{
    \bigprn{\I-\mP_{\tau(\itr)}}
    \bigprn{\w_{k}-\teacher}}^2~.
\end{align}
\vspace{-2mm}
\begin{figure}[h!]
  \begin{minipage}[c]{0.62\textwidth}
    \caption{
    \small
    \textbf{Projection illustration.}
    \\
    According to \eqref{eq:affine-rule},
    ${(\w_{k-1}\!-\!\teacher)}$~is given by projecting 
    ${(\w_{k-2}\!-\!\teacher)}$
    onto the solution space of the $(k-1)$th task,
    \ie $\sol{\tau(k-1)}$,
    which in this figure is a rank-1 affine subspace in $\reals^2$.
    In turn, ${(\w_{k-1}\!-\!\teacher)}$ is projected onto $\sol{\tau(k)}$ to obtain ${(\w_{k}\!-\!\teacher)}$,
    and so on.
    Overall, the solution is continually getting closer to~$\teacher$.
    Moreover, the forgetting
    (\textcolor{residuals}{magenta}) is the
    mean of the squared residuals
    from projecting $(\w_{k}-\teacher)$
    onto previously seen solution spaces,
    as can be seen from \eqref{eq:geometric}.
    }
    \label{fig:contraction}
  \end{minipage}
  \hfill
  \begin{minipage}[c]{0.34\textwidth}
    \includegraphics[width=.95\textwidth]{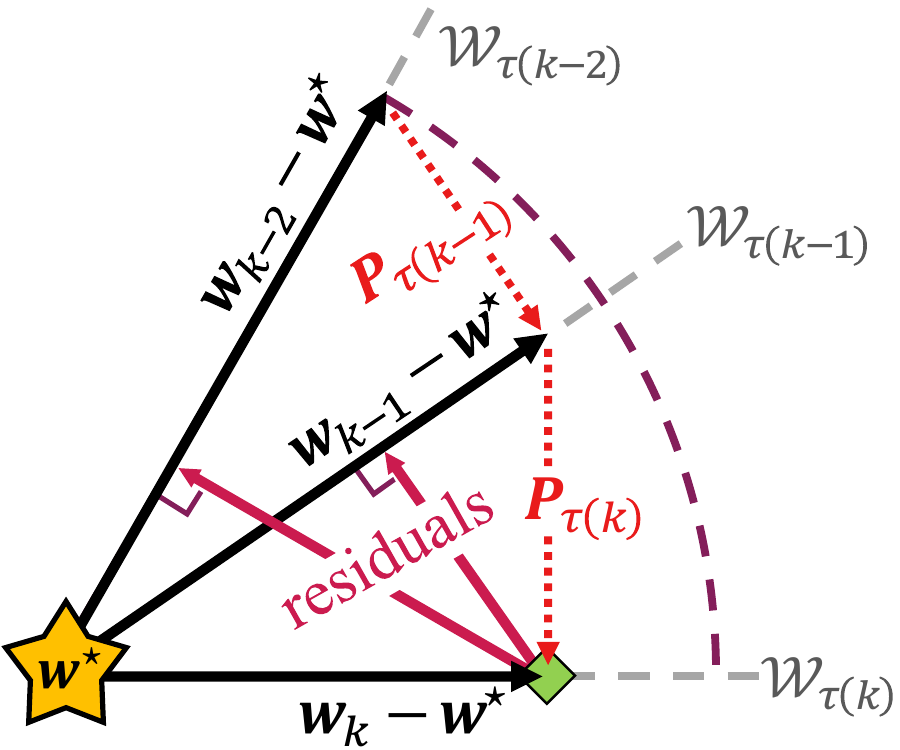}
  \end{minipage}
\end{figure}
\vspace{-3mm}
\paragraph{Learning as contracting.}
Recursively,
the affine update rule from \eqref{eq:affine-rule}
provides a closed-form expression for the 
distance between the iterate and the offline solution 
(recall $\w_{0}=\0$):
\vspace{-1mm}
\begin{align}
\begin{split}
    \w_{\itr}-\teacher
    &=
    \mP_{\tau(\itr)}\cdots\mP_{\tau(1)}
    \big(
    \,\cancel{\w_{0}}
    -
    \teacher~\big)~.
\label{total_contraction}
\end{split}
\end{align}
Since orthogonal projections are non-expansive operators, it also follows that
\vspace{-1mm}
\begin{align*}
    \forall{\itr\!\in\!\cnt{T}\!}:~
    \tnorm{\w_{\itr} - \teacher}
    \leq 
    \tnorm{\w_{\itr-1} - \teacher}
    \leq 
    \dots
    \leq 
    \tnorm{\,\cancel{\w_{0}} - \teacher}
    =
    \tnorm{\teacher}~,
\end{align*}
hinting at a possible convergence towards the offline solution,
as depicted in \figref{fig:contraction}.

\pagebreak

Combining \eqref{eq:geometric} and (\ref{total_contraction}), 
we express the average forgetting (\defref{def:forgetting}) 
of an ordering $\tau$ over a task collection $S=\{(\X_{m},\y_{m})\}_{m=1}^T$ as follows,
\begin{align}
    \label{eq:linear_forgetting}
    F_{\tau,S}\left(k\right)
    &= 
    \frac{1}{k}
    \sum_{\itr=1}^{k}
    \bignorm{\X_{\tau(\itr)}  
    {\mP_{\tau(k)}
    \cdots
    \mP_{\tau(1)}}
    \teacher}^2
    \le
    \frac{1}{k}
    \sum_{\itr=1}^{k}
    \bignorm{
    \bigprn{\I\!-\!\mP_{\tau(\itr)}}
    {\mP_{\tau(k)}
    \cdots
    \mP_{\tau(1)}}
    \teacher}^2~.
\end{align}

\paragraph{Worst-case formulation.}
So far, 
${\tnorm{\X_m\X_m^+\X_m\vu}^2
\le\tnorm{\X_m^+\X_m\vu}^2}$
is the only inequality we used
(at \eqref{eq:geometric}, relying on Assumption~\ref{assume:bounded_data}).
Importantly, this inequality saturates when all non-zero singular values of $\X_m, \forall m\!\in\!\cnt{T}$ are $1$.
Consequentially, 
the \emph{worst-case} forgetting in \defref{def:worst-case}
can be simply expressed in terms of $T$ projection matrices
$\mP_{1},\dots,\mP_{T}\in\reals^{d\times d}$, 
as
\begin{align}
\label{eq:forgetting-projection}
    \sup_{
    \substack{
        \coll\in\mathcal{S}_{T}
    }
    }
    \!
    F_{\tau,\coll}(k)
    \,
    =
    \sup_{
    \substack{
        \mP_{1},\dots,\mP_{T}
    }}
    \frac{1}{k}\,
    \tsum_{\itr=1}^{k}
    \bignorm{
    \bigprn{\I-\mP_{\tau(\itr)}}
    \mP_{\tau(k)}
    {\cdots}\mP_{\tau(1)}
    }^2~,
\end{align}
where we also used Assumption~\ref{assume:realizability}
that $\norm{\teacher}\le 1$.

Throughout this paper,
we mainly analyze \eqref{eq:linear_forgetting}~and~(\ref{eq:forgetting-projection}) from different perspectives.
Multiplying from the left by $\left\{\I-\mP_{\tau(\itr)}\right\}_{\itr=1}^{k}$ 
is what distinguishes our quantity of interest --- the forgetting --- 
from quantities studied in the area of alternating projections.

\paragraph{Principal angles between two tasks.}
Finally, we briefly present \emph{principal angles}, which affect forgetting dynamics, as discussed throughout this paper.
Much of the research on alternating projections has focused on establishing notions of angles between subspaces \citep{deutsch1995angle,oppenheim2018angle}.
%
\emph{Principal~angles},
also known as canonical angles,
are a popular choice for angles between two linear subspaces
\citep{bjorck1973numerical,bargetz2020angles},
having many applications in numerical analysis
(e.g., in the generalized eigenvalue problem \cite{ge2016generalized}).
These angles geometrically describe a pair of subspaces,
by recursively taking the smallest angle between any two vectors in these subspaces that are orthogonal to previously chosen vectors.
We elaborate on the definition and the role of these angles in \appsref{app:principal_angles}.
There, we visualize these angles and explain that the non-zero principal angles between the row spaces of two tasks,
\ie $\range(\X_1^\top),\range(\X_2^\top)$,
are identical to those between the corresponding 
solution spaces $\sol{1},\sol{2}$.

\paragraph{The rest of our paper.}
We study forgetting 
under different task orderings.
In \Secref{sec:arbitrary} we consider arbitrary orderings
and show when there is provably \emph{no} forgetting,
and when forgetting is arbitrarily high,
\ie catastrophic. 
We analyze cyclic and random orderings in Sections~\ref{sec:cyclic}~and~\ref{sec:random}.
For both these orderings, 
we derive convergence guarantees 
and prove
forgetting \emph{cannot} be catastrophic.
%


\section{Arbitrary task orderings}\label{sec:arbitrary}
\paragraph{Identity ordering.}
In this section we consider arbitrary sequences of tasks, \ie we do not impose any specific ordering. 
To this end, we take $k\!=\!T$ and an \emph{identity ordering} $\tau$ s.t. 
$
{\tau\!\left(t\right)\!=\!t,\,
\forall t\!\in\!\naturals}$.
To simplify notation, in this section only,
we suppress $\tau$ 
and use $\X_{\tau(m)}\!=\!\X_{m}$
interchangeably.

\subsection{No forgetting cases}
\label{sec:no_forgetting}
Consider learning two tasks sequentially:
$\big(\X_{1},\y_{1}\big)$
and then $\big(\X_{2},\y_{2}\big)$.
Right after learning the second task,
we have
$\mathcal{L}_{2}(\w_2)
\!=\!
\tnorm{\X_2 \w_2 \!-\! \y_2}^2 
\!=\! 0$.
Thus, the forgetting from \eqref{eq:linear_forgetting} becomes 
${
    F_{\tau,S}(2)
    =
    \tfrac{1}{2}
    \norm{
    \X_1
    \mP_{2}\mP_{1}
    \w^\star}^2
}
$.
We now derive sufficient and necessary conditions for no forgetting.

\pagebreak

\begin{theorem}[No forgetting in two-task collections]
\label{thm:noforgetting}
Let
$\coll\!=\!\left\{\big(\X_{1},\y_{1}\big),\big(\X_{2},\y_{2}\big)\right\}
\!\in\!\mathcal{S}_{T=2}$ 
be a task collection with $2$ tasks,
fitted under an identity ordering $\tau$,
\ie$\big(\X_{1},\y_{1}\big)$~and then $\big(\X_{2},\y_{2}\big)$.
Then the following conditions are equivalent:
\begin{enumerate}\itemsep.5pt
    \item 
    For \emph{any} labeling $\y_1,\y_2$
    (or equivalently, any minimum norm solution $\teacher$),
    after fitting the second task, the model does not ``forget'' the first one.
    That is,
    $F_{\tau,\coll}(2)=0$.
    \item
    It holds that
    $\X_1 \mP_2 \mP_1 = 
    \0_{n_1 \times d}$.
    \item Each principal angle between the tasks, \ie $\range({\X_1^\top})$ and
    $\range({\X_2^\top})$, is either $0$ or $\nicefrac{\pi}{2}$.
\end{enumerate}
\end{theorem}
The proof is given in App.~\ref{app:no_forgetting}.
The exact definition of angles between tasks is given in App.~\ref{app:principal_angles}.

For instance, the above conditions hold when
$\range(\X_{1}^\top)\subseteq \range(\X_{2}^\top)$
(or vice-versa),
\ie there is no forgetting when tasks have maximum overlap in the row span of inputs.
The third condition aligns with the empirical observation in \citet{ramasesh2020anatomy}, wherein catastrophic forgetting in overparameterized neural networks is small when the tasks are either very similar or very distinct. %
On the other hand, this seemingly contradicts the conclusions in \citet{doan2021NTKoverlap} that similar tasks are \textit{potentially} bad for forgetting.
However, their conclusions are based on a loose upper bound on the forgetting.
In \secref{sec:two_tasks},
we carefully analyze the forgetting dynamics of two tasks seen repeatedly in cycles
and show that increasingly-similar tasks \emph{can} be worse for forgetting, 
but only after multiple cycles.
We elaborate on these connections in \appsref{app:compare_to_doan}.

\subsection{Maximal forgetting cases: can $F_{\tau,\coll}(k) \rightarrow 1$~?}
\label{sec:maximal_forgetting}

Now, we present an adversarial task collection that yields arbitrarily high forgetting:
at the end of learning, the learner almost completely ``forgets'' previously-seen tasks. 
We use this opportunity to build further intuition on two factors causing high forgetting.

Our construction is intuitively based on 
the geometric interpretation from \eqref{eq:geometric} and \figref{fig:contraction},
that the forgetting is the \emph{mean} of the squared residuals from projecting $(\w_k\!-\!\teacher)$ onto the solution spaces of previously-seen tasks,
\ie
$F_{\tau,S}\left(k\right) 
    \!\le\!
    \tfrac{1}{k}
    \tsum_{\itr=1}^{k}
    \bignorm{
    \bigprn{\I\!-\!\mP_{\tau(\itr)}}
    \bigprn{\w_{k}\!-\!\teacher}}^2$.

\vspace{1mm}

\begin{figure}[h!]
  \begin{minipage}[c]{0.58\textwidth}
    \small
    \caption{
    \small
    \textbf{Illustrating the adversarial construction.}
    For the discussed residuals to be large,
    our construction ensures that:
    }
    \vspace{-2mm}
    
    \begin{enumerate}
    \item \textbf{The iterates are kept afar from the $\teacher$.}
    \linebreak
    Since
    ${F_{\tau,S}\left(k\right)
    \le
    \tnorm{\w_{k}\!-\!\w^{\star}}^2
    \le
    \tnorm{\w_{\itr}\,-\,\w^{\star}}^2},
    \forall \itr\!\le\!k$,
    it is important to maintain a large $\tnorm{\w_{t}\!-\!\w^{\star}}$ 
    in all iterations.
    We achieve this by using similar \emph{consecutive} tasks.
    \item 
    \textbf{\emph{Most} solution spaces are orthogonal to the last one.}
    For the averaged residuals to be large, 
    the last 
    $\bigprn{\w_{k}\!-\!\teacher}$
    should be orthogonal to as many previous solution spaces as possible.
    For this, we ``huddle'' most of the tasks near the first one, almost orthogonally to the last.
    \end{enumerate}
    \label{fig:worst_case}
  \end{minipage}
  \hfill
  \begin{minipage}[c]{0.4\textwidth}
    \includegraphics[width=0.95\textwidth]{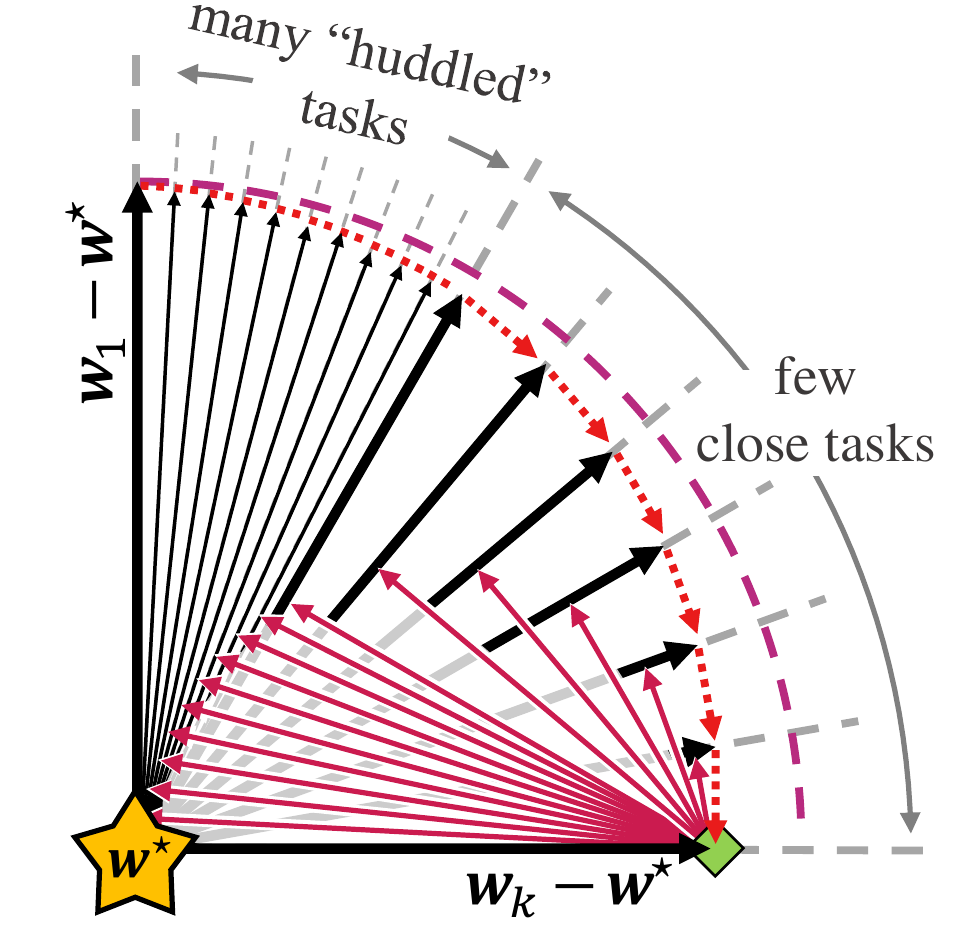}
  \end{minipage}
\end{figure}

\vspace{-3mm}

\begin{theorem}[Forgetting can be arbitrarily bad]
\label{thm:worst_case}
When using the identity ordering 
(\ie $\tau\!\left(\itr\right)\!=\itr$),
thus seeing each task once,
the worst-case forgetting after $k$ iterations is arbitrarily bad,
\ie
\begin{align*}
    1-\!
    \sup_{
    \substack{
        \coll\in\mathcal{S}_{T=k}
    }
    }
    \!\!
    F_{\tau,\coll}(k)
    \le
    \bigO\left(\nicefrac{1}{\sqrt{k}}\right)~.
\end{align*}
\vspace*{-4mm}
\end{theorem}
The exact construction details and the proof are given in \appsref{app:maximal_forgetting}.

\pagebreak

By now, we understand that under arbitrary task orderings, 
there exist task sequences where
the learner almost \emph{completely} forgets previously-learned expertise
and forgetting is indeed \emph{catastrophic}.
In the sections to follow, we show that cyclic and random orderings \emph{do not} suffer from this flaw.

\section{Cyclic task orderings}
\label{sec:cyclic}
We again consider collections of $T$ tasks $\left\{\big(\X_{m},\y_{m}\big)\right\}_{m=1}^T$,
but now we study the forgetting when tasks are presented in a cyclic ordering $\tau$,
\ie
$
\tau\left(\itr\right)=
\big(\left(\itr-1\right)\mathrm{mod}~ T~\big)+1,~
\forall \itr\!\in\!\naturals
$.
\linebreak
For example, suppose we want to train a pedestrian detector continuously during different times of the day (morning, noon, evening, and night), so that the task order forms a fixed cycle.
\linebreak 
Such cyclic settings also arise in search engines, e-commerce, and social networks,
where tasks (\ie distributions)
are largely influenced by events that recur
either weekly (\eg weekdays vs. weekends),
monthly (\eg paydays),
annually (\eg holidays),
and so on.

\bigskip

Under cyclic orderings,
the forgetting from \eqref{eq:linear_forgetting}
after $n$ cycles, becomes
\vspace{-2mm}
\begin{align}
\label{eq:cyclic-forgetting}
F_{\tau,\coll}(k=nT)
=
\frac{1}{T}
\sum_{m=1}^{T}
\bignorm{
\X_m
\bigprn{\mP_T \,\smallcdots \mP_1}^{n}
\teacher
}^2
\le 
\frac{1}{T}
\sum_{m=1}^{T}
\bignorm{
(\I\!-\!\mP_{m})
\bigprn{\mP_T \,\smallcdots \mP_1}^{n}
}^2.
\end{align}

\subsection{Warm up: Exact forgetting analysis with $T=2$ tasks}
\label{sec:two_tasks}

Exploiting the connection we made to the field of alternating projections,
we first analyze the convergence to the minimum norm offline solution
in terms of the \emph{Friedrichs angle} 
\citep{friedrichs1937angles}
between the tasks,
\ie their minimal \emph{non-zero} principal angle 
$\theta_F\!\triangleq\!
\min_{i: \theta_i\neq0} \theta_i > 0$ 
(as explained in \appsref{app:principal_angles}).

\begin{theorem}[Convergence to the minimum norm offline solution]
\label{thm:2_tasks_distance}
For any task collection 
of two distinct tasks
fitted in a~cyclic ordering $\tau$, 
the distance from the offline solution
after $k\!=\!2n$ iterations ($n$ cycles) is
\emph{tightly} upper bounded by
$\big\Vert
{\w_{k} - \teacher}
\big\Vert^2
\le
\left(
\cos^2 \theta_F\right)^{k-1}\norm{\teacher}^2$,
where $\theta_F$
is the Friedrichs angle
between the given tasks, as defined above.
\end{theorem}
\begin{proof}
Plugging in
the cyclic ordering definition into
the recursive form of \eqref{total_contraction}, 
we obtain
${\bignorm{
{\w_{k}\!-\!\teacher}}
=
\bignorm{\mP_{\tau(k)}\cdots\mP_{\tau(1)}
\teacher}
=\bignorm{
\tprn{\mP_2 \mP_1}^{n}
\teacher}}
$.

A known alternating projection result by
\citet{kayalar1988error} (Theorem~2 therein) states that
${\bignorm{
\tprn{\mP_2 \mP_1}^{n}-\mP_{1\cap\,2}}^{2}
=
\left(\cos^2 \theta_F\right)^{k-1}}$,
where in our context,
$\mP_{1\cap\,2}$ projects onto 
the null spaces' intersection, \ie 
$\kernel({\X_1})\cap\kernel({\X_2})$.
Since the \emph{minimum norm} solution $\teacher$ must lie in 
$\range\!\prn{\X_1^\top}\cup\range\!\prn{\X_2^\top}$,
then by properties of orthogonal complements we have
${\mP_{1\cap\,2}\teacher=\0}$.
Then, 
we see that
$
\tnorm{{\w_{k} \!-\! \teacher}}^2
=
\bignorm{\bigprn{\mP_2 \mP_1}^{n}\teacher}^{2}
\!\!=\!
\bignorm{\bigprn{\tprn{\mP_2\mP_1}^{n}
\!-\!
\mP_{1\cap\,2}}\teacher}^{2}$,
and conclude:
$${
\tnorm{{\w_{k} \!-\! \teacher}}^2
\!=\!
\bignorm{\bigprn{\tprn{\mP_2\mP_1}^{n}
\!-\!
\mP_{1\cap\,2}}\teacher}^{2}
\!\le\!
\bignorm{{\tprn{\mP_2\mP_1}^{n}
\!-\!
\mP_{1\cap\,2}}}^{2}
\bignorm{\teacher}^{2}
=\!
\left(\cos^2 \theta_F\right)^{k-1}
\!
\bignorm{\teacher}^{2}
}\,.$$
Clearly, a carefully chosen $\teacher$
(induced by $\y_1,\y_2$)
can saturate the inequality, making it tight.
\end{proof}
Note that the rate of convergence to $\w^{\star}$ can be arbitrarily slow when
the Friedrichs angle ${\theta_F\to 0}$. 
Importantly, this means that there can be no data-independent convergence guarantees to $\teacher$.
\linebreak
One might think that this implies that the forgetting
(\ie the residuals)
is also only trivially bounded, however a careful analysis shows that this is not the case.

\pagebreak

In contrast to the above \Thmref{thm:2_tasks_distance}, we now show that the forgetting 
\emph{is} non-trivially bounded.
\begin{lemma}[Angles' effect on forgetting $T=2$]
\label{lem:2_tasks_forgetting}
For any task collection $\coll\in\mathcal{S}_{T=2}$ of two tasks,
the forgetting after $k=2n$ iterations (\ie $n$ cycles) is tightly upper bounded by
\begin{align*}
\begin{split}
F_{\tau, \coll}\left(k\right)
\le~
&
\frac{1}{2}
\max_{i}
\left\{
\left(\cos^2 \theta_i\right)^{k-1}
\left(1-\cos^2 \theta_i\right)
\right\}~,
\end{split}
\end{align*}
where 
$\{\theta_i\}_i \subseteq
\left(0, \tfrac{\pi}{2}\right]$
are the non-zero principal angles between the two tasks in $\coll$.
Moreover, the above inequality saturates when all non-zero singular values of the first task
(\ie of $\X_{1}$)
are $1$s.
\end{lemma}

In contrast to $\cos^2 \theta_F$ that bounds the distance to the offline solution (see \thmref{thm:2_tasks_distance}) and can be arbitrarily close to $1$, the quantities $\left(\cos^2 \theta_i\right)^{k-1}
\left(1-\cos^2 \theta_i\right)$ in \lemref{lem:2_tasks_forgetting} 
are upper bounded \linebreak
\emph{uniformly} for any $\theta_i$, which allows deriving a data-independent expression for the worst-case forgetting in the next theorem.
\begin{theorem}[Worst-case forgetting when $T=2$]
\label{thm:2_tasks_worst_forgetting}
For~a~cyclic ordering $\tau$ of two tasks, 
the worst-case forgetting after $k=2n$ iterations
(\ie $n$ cycles),
is
\vspace{-2mm}
\begin{align*}
    \sup_{
    \substack{
        \coll\in\mathcal{S}_{T=2}
    }
    }
    \!\!
    F_{\tau,\coll}
    \left(k\right)
    =
    \frac{1}{2e\left(k-1\right)} -
    \frac{1}{4e\left(k-1\right)^2} +
    \bigO\left(\frac{1}{k^3}\right)~.
\end{align*}
\end{theorem}
The proofs for both 
\lemref{lem:2_tasks_forgetting} and
\thmref{thm:2_tasks_worst_forgetting}
are given in \appsref{app:two_tasks}.

\paragraph{Demonstration.}
\figref{fig:two_tasks_super} demonstrates
our analysis for the worst-case forgetting 
on $T\!=\!2$ tasks.
We consider a simplistic case where both tasks are of rank $d\!-\!1$,
\ie 
$\rank(\X_1)\!=\!\rank(\X_2)\!=\!d\!-\!1$,
thus having solution spaces of rank $1$ with
a straightforward single angle $\theta$ between them.

\begin{figure}[h!]
    \subfigure[\small 
    Effect of task similarity on forgetting.] {
        \includegraphics[width=.48\columnwidth]{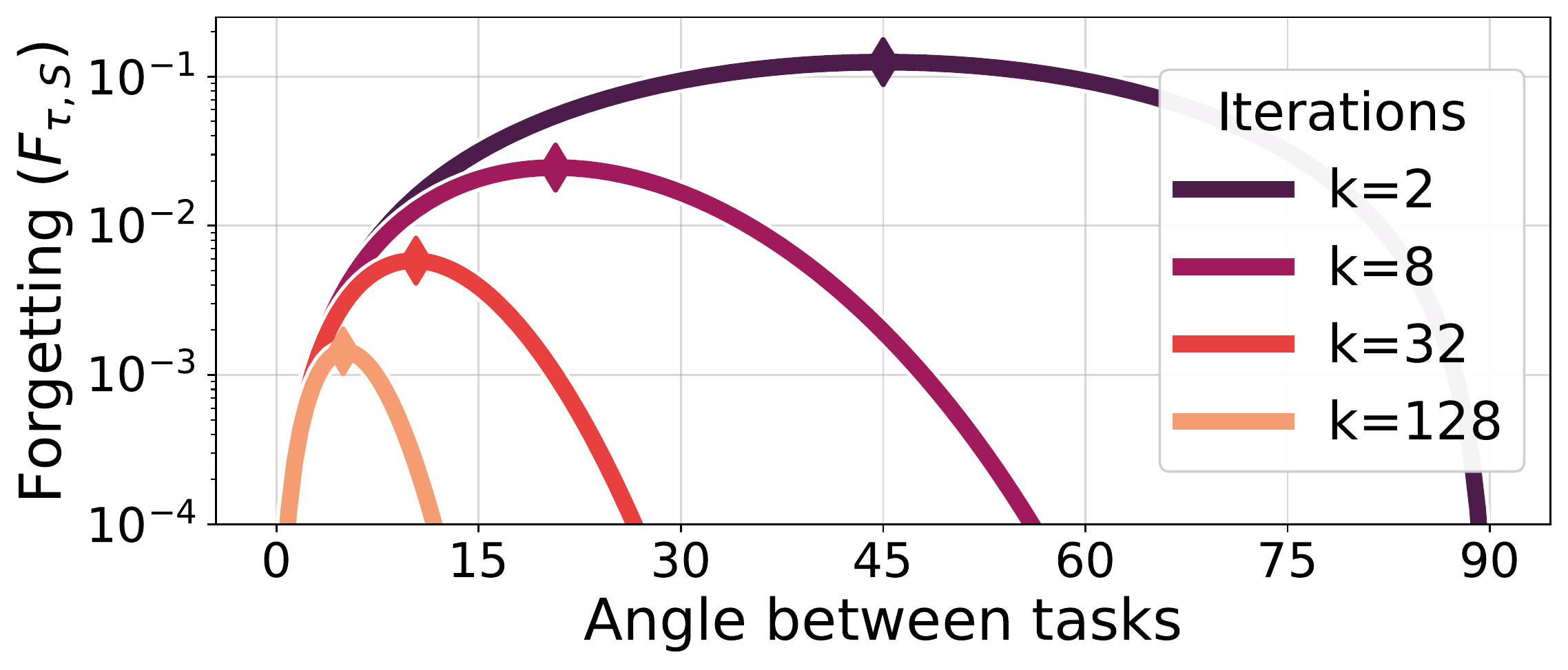}
        \label{fig:exact_2_tasks}
    }
    \hfill
    \subfigure[\small A sharp uniform bound for the forgetting.]
    {
        \includegraphics[width=.48\columnwidth]{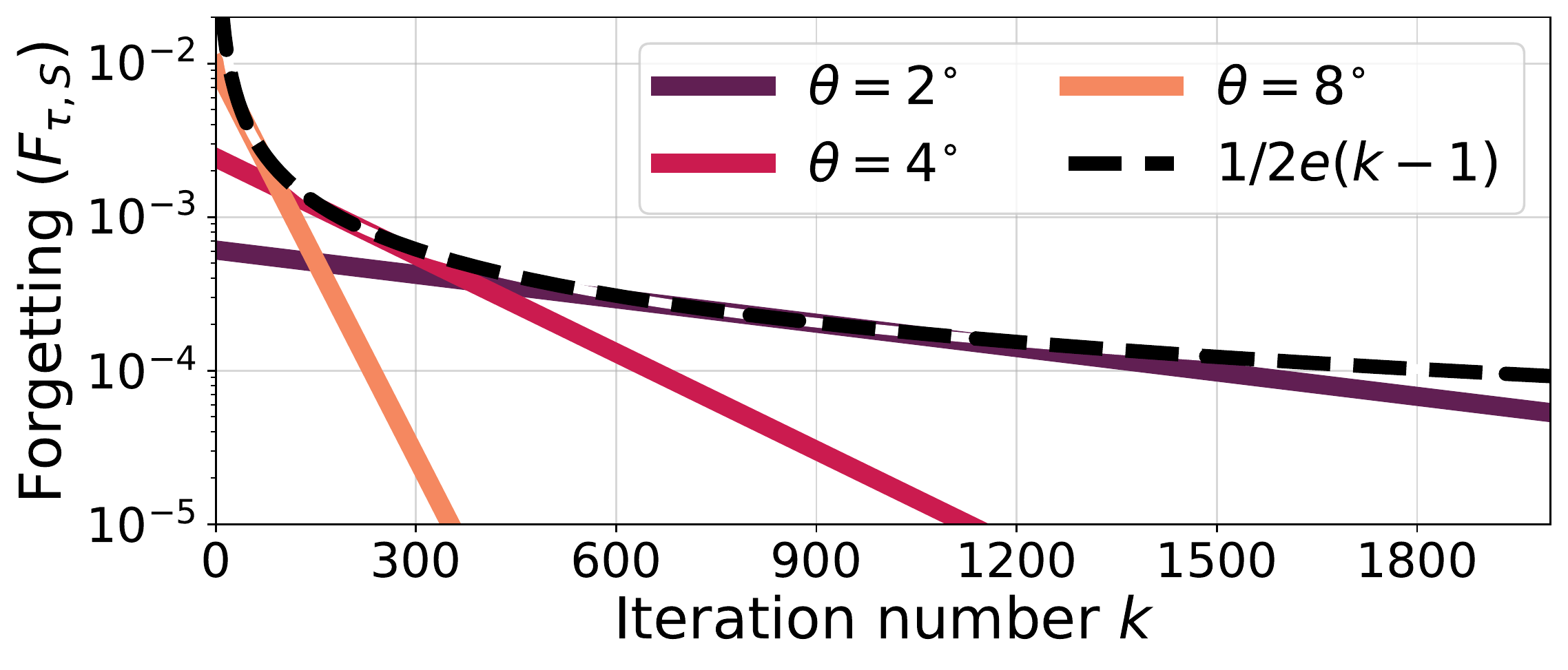}
        \label{fig:2tasks_bound}
    }
    \vspace{-1mm}
    \caption{
    \small
    {Demonstration of forgetting (\lemref{lem:2_tasks_forgetting}) and worst-case forgetting (\Thmref{thm:2_tasks_worst_forgetting}) for $T=2$.}
    \label{fig:two_tasks_super}
    }
\end{figure}

\figref{fig:exact_2_tasks}
demonstrates the
analytical effect of task angles,
\ie 
$(\cos^2 \theta)^{k-1}
(1\!-\!\cos^2 \theta)$
from 
\lemref{lem:2_tasks_forgetting}.
After one cycle ($k\!=\!2$), 
small and nearly-orthogonal angles induce low forgetting, 
while intermediate angles are troublesome.
However, as $k$ increases, 
the angle that maximizes forgetting goes to zero.
Importantly,
we find that the effect of task similarity depends on the number of cycles!

\figref{fig:2tasks_bound} 
demonstrates the worst-case analysis of \thmref{thm:2_tasks_worst_forgetting}. 
As tasks become \emph{more} similar,
\ie $\theta$ decreases, the initial forgetting is smaller, yet convergence is slower since the contraction is small at each iteration. 
Conversely, larger angles lead to larger initial forgetting but also to faster convergence due to a more significant contraction.

Our findings after seeing each task \emph{once} (\ie $k\!=\!2$) resemble findings from \citet{lee2021taskSimilarity} 
that intermediate task similarity causes the most forgetting
(however, their setup and notion of similarity are different).
Like we mentioned in \secref{sec:no_forgetting},
our analysis contradicts a corollary from \citet{doan2021NTKoverlap}
implying a higher risk of forgetting when two tasks are \emph{more} aligned.
This discrepancy stems from an upper bound in \citep{doan2021NTKoverlap} 
being looser than the tight bounds we derive
(see \appsref{app:compare_to_doan}).

\pagebreak

\subsection{Main Result: Worst-case forgetting with $T\ge 3$ tasks}
\label{sec:forgetting_many_tasks}

For the general cyclic case, 
we provide two upper bounds --
a dimension-dependent bound, and more importantly, a dimension-independent one.
Both follow a power-law \wrt  the iteration number $k$.
\begin{theorem}[Worst-case forgetting when $T\ge3$]
\label{thm:worst_case_T3}
For~any number of tasks $T\ge 3$ under a cyclic ordering $\tau$, the worst-case forgetting
after $k=nT\ge T^2$ iterations (\ie $n\ge T$ cycles), is
\vspace{-.1mm}
\begin{align*}
    \frac{T^2}{24ek}
    ~~\le
    \!\!\!\!\!\!\!\!\!\!
    \sup_{
    \substack{
        \coll\in\mathcal{\coll}_{T\ge 3}: \\
        \forall m:\,\rank(\X_m) \le \maxrank
    }
    }
    \!\!\!\!\!\!\!\!\!\!\!\!\!\!\!\!
    F_{\tau,\coll}(k)
    ~~\le~~
    {
        \min\left\{
        \frac{T^2}{\sqrt{k}}
        ,~~
        \frac{T^2\prn{d-\maxrank}}{2k}
        \right\} 
    }~.
\end{align*}
%

Moreover, 
if the cyclic operator $\bigprn{\mP_T \cdots \mP_1}$ from \eqref{eq:cyclic-forgetting} is symmetric (\eg in a \emph{back-and-forth} setting  where tasks $m$ and $(T\!-\!m\!+\!1)$ are identical $\forall m\!\in\!\cnt{T}$),
then the worst-case forgetting is \emph{sharply}
${
\sup_{
\substack{
    \coll\in\mathcal{\coll}_{T\ge 3}
}
}
F_{\tau,\coll}(k)
=
\Theta\left(\nicefrac{T^2}{k}\right)
}
$.
\end{theorem}
%
\paragraph{Proof sketch for the upper bound.}
We briefly portray our proof for the above result (given fully in \appsref{app:proofs_many_tasks}).
For brevity, denote the cyclic operator as
$\M\triangleq\mP_T \dots \mP_1$.
Our proof revolves around the 
maximal decrease at the $n$th cycle,
\ie $\decrease{n}{\vu}
\triangleq\bignorm{
\M^{n}
\vu}^2
-
\bignorm{
\M^{n+1}
\vu}^2$.
\linebreak
We start by showing that the worst-case forgetting on the \emph{first task} (see \eqref{eq:cyclic-forgetting}) is upper bounded by the maximal decrease.
That is,
\vspace{-1mm}
\begin{align*}
\forall \vu\!\in\!\complex^{d}\!:~
\bignorm{
(\I\!-\!\mP_{1})
\M^{n}
\vu}^2
\stackrel{\text{(*)}}{=}
\bignorm{
\M^{n}
\vu}^2
\!-
\bignorm{
\mP_{1}
\M^{n}
\vu}^2
\stackrel{\text{(**)}}{\le}
\!
\bignorm{
\M^{n}
\vu}^2
\!-
\bignorm{
\M^{n+1}
\vu}^2
\!\!=\!
\decrease{n}{\vu},
\end{align*}
where (*) stems from the idempotence of $\mP_1$, 
and (**) is true since projections are non-expansive operators, meaning
$
\bignorm{
\mP_{1}
\M^{n}
\vu}^2
\ge 
\bignorm{
\mP_{T}\cdots\mP_{1}
\M^{n}
\vu}^2
\triangleq
\bignorm{
\M^{n+1}
\vu}^2
$.

More generally, we show that
$
\tnorm{
(\I-\mP_{m})
\M^{n}
\vu}^2
\!\le\!
m \decrease{n}{\vu}$,
yielding the overall bound:
\vspace{-1mm}
\begin{align*}
\tfrac{1}{T}
\tsum_{m=1}^{T}
\bignorm{
(\I-\mP_{m})
\M^{n}
}^2
\!\le\!
\tfrac{T-1}{2}
\cdot\!\!\!\!
\max_{\vu:\left\Vert \vu\right\Vert _{2}=1}\!
    \!\!
    \decrease{n}{\vu}~.
\end{align*}

Then, we prove that
$\max_{\vu:\left\Vert \vu\right\Vert _{2}=1}
    \decrease{n}{\vu}
    \le
    2\tnorm{\M^{n}\!-\!\M^{n+1}}\le
    2\sqrt{\nicefrac{T}{n}}$,
using telescoping sums on elements of
$\{\tnorm{\left(\M^\itr\!-\!\M^{\itr+1}\right)\!\vv}^2\}_{\itr=1}^{n}$.
Finally, we prove
$\max_{\vu:\left\Vert \vu\right\Vert _{2}=1}
    \decrease{n}{\vu}
\le \tfrac{d-\maxrank}{n}
$
by using telescoping sums on the traces of matrices 
$\left\{
\tprn{\M^{\itr}}^{\top}
\M^{\itr}
\!-
\tprn{\M^{\itr+1}}^{\top}
\M^{\itr+1}
\right\}_{\itr=1}^{n}.~
\blacksquare$

\section{Random task orderings}
\label{sec:random}
So far, we saw in \secref{sec:arbitrary} that arbitrary task orderings provide no convergence guarantees and might forget  \emph{catastrophically}.
In \secref{sec:cyclic} we saw that cyclic orderings \emph{do not} suffer from catastrophic forgetting, since their forgetting converges to zero like a power law.
We now analyze \emph{random} task ordering,
and show that they also have uniform (data-independent) convergence guarantees.

We consider a random task ordering $\tau$ that matches a uniform probability to any task at any iteration, \ie
$
\forall m\!\in\!\cnt{T},
\forall \itr\!\in\!\naturals\!\!:~
    \Pr\left[~\tau\!\left(\itr\right)=m~\right]\!=\!
    \nicefrac{1}{T}
$.
Below, we adjust the forgetting definitions in~\ref{def:forgetting}~and~\ref{def:worst-case}
to the random setting by defining the expected forgetting.
\begin{definition}[Expected forgetting of a task collection]
\label{def:expected-forgetting}
After $k$ iterations, the \emph{expected forgetting} on a specific task collection $\coll\in\Coll$
is defined as
\vspace{-1mm}
\begin{align*}
\expfor_{\tau,\coll} \left(k\right)
&\triangleq
\expectation_{\tau}
\left[
F_{\tau,\coll}\left(k\right)
\right]
=
\expectation_{\tau} 
\Big[
\tfrac{1}{k}
\tsum_{\itr=1}^{k}
\norm{\X_{\tau(\itr)}\w_k-\y_{\tau(\itr)}}^2
\Big]~.
\end{align*}
\end{definition}

Our main result in this section is a uniform bound on the expected forgetting under the uniform random task ordering. 
The proof is given in \appsref{app:proofs6}.
\begin{theorem}[Worst-case expected forgetting]
\label{thm:expected_forgetting}
Under the uniform i.i.d. task ordering $\tau$,
the worst-case expected forgetting  
after $k$ iterations is
\vspace{-2mm}
\begin{align*}
    \sup_{
    \substack{
        \coll\in\mathcal{\coll}_{T}: \\
        \tfrac{1}{T}
        \tsum_{m=1}^{T} \rank(\X_m) = \avgrank
    }
    }
    \!\!\!\!\!\!\!\!\!\!\!\!\!\!\!\!\!\!\!
    \expfor_{\tau,\coll}
    \left(k\right)
    \le
    \frac{
    9\left(d-\avgrank\right)}{k}~.
\end{align*}
\end{theorem}
\vspace{-1mm}

\paragraph{Demonstration.}
The following \figref{fig:many_task_bound} demonstrates
the worst-case forgetting under cyclic task orderings (\thmref{thm:worst_case_T3})
and the worse-case expected forgetting under random orderings (\thmref{thm:expected_forgetting}).
%
We consider a specific $2$-d task collection $\coll$ consisting of $T\!=\!128$ rank-one tasks. 
\linebreak
The collection is ``adversarial'' in the sense that its forgetting meets the \emph{cyclic} setting's lower bound
of \thmref{thm:worst_case_T3} 
(dashed \textcolor{orange}{orange}).
\vspace{-2mm}
\begin{figure}[h!]
  \begin{minipage}[c]{0.57\textwidth}
  \vspace{3mm}
    \caption{
    \small
    \textbf{Demonstrating the bounds from Thm~\ref{thm:worst_case_T3}~and~\ref{thm:expected_forgetting}.}
    \\
    The solid \textcolor{orange}{orange} curve shows the actual forgetting of
    the cyclic deterministic ordering 
    (the oscillations are formed, naturally, by the task cycles).
    The \textcolor{purple}{purple} solid curve shows the expected forgetting of
    the random ordering
    (averaged over {$5$} seeds).
    The \textcolor{purple}{purple} band indicates one standard deviation (over the $5$ seeds).
    Also plotted are the corresponding upper bounds of both settings (dotted).
    \\
    Notice how a random ordering behaves better than a cyclic one, both practically and analytically.
    }
    \label{fig:many_task_bound}
  \end{minipage}
  \hfill
  \begin{minipage}[c]{0.39\textwidth}
    \includegraphics[width=.97\textwidth]{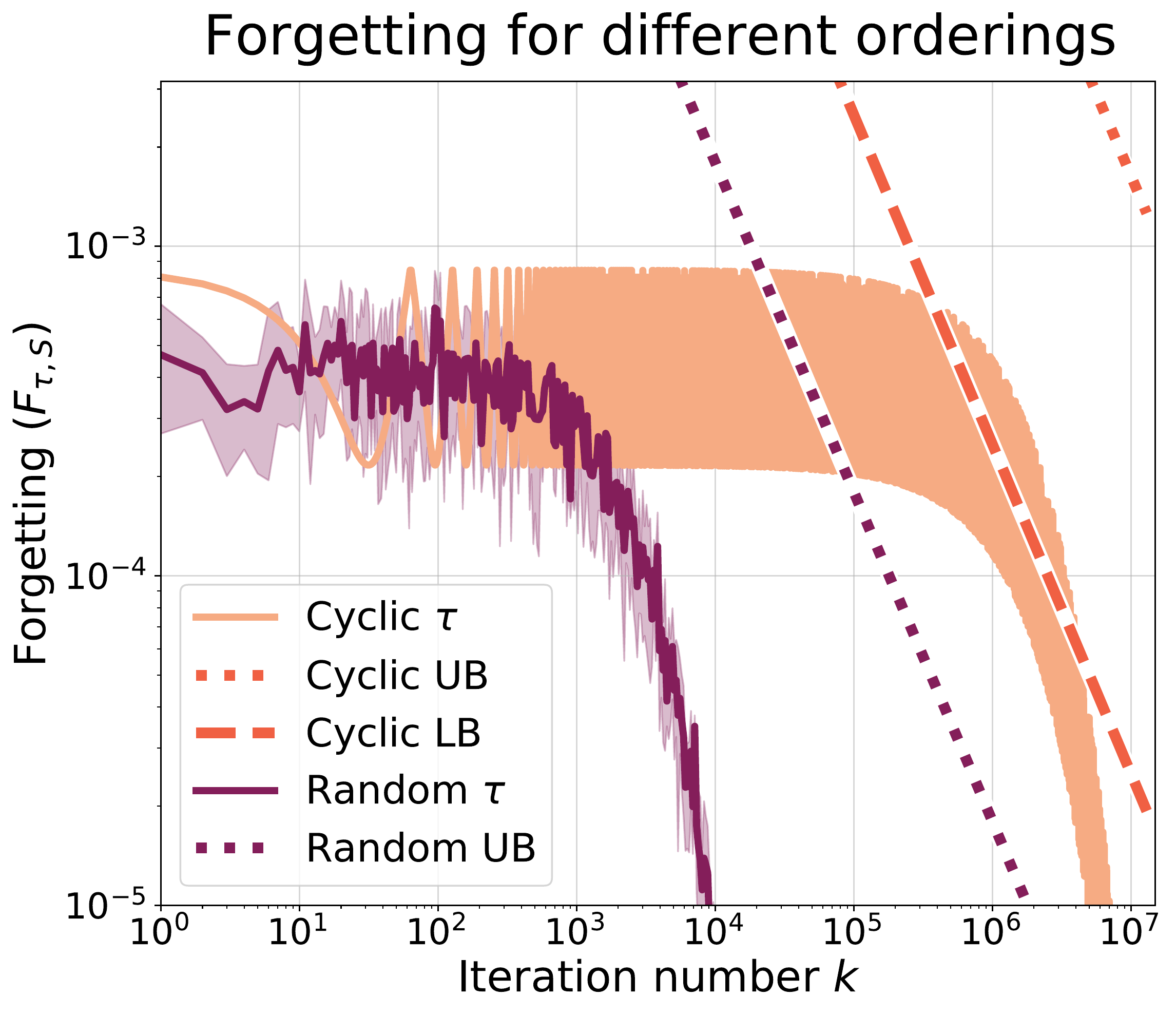}
  \end{minipage}
 \vspace{-4mm}
\end{figure}

\begin{remark}[Last iterate SGD results]
In the special case 
of rank-one tasks,
the iterative update rule in \eqref{eq:update_rule} reduces to a \emph{single} SGD step
with a $\tnorm{\x_{\tau(\itr)}}^{-2}$ step size, 
and the above bound becomes $\bigO(d/k)$.
A very recent work \citep{varre2021last} derived a
\emph{dimension-independent} bound 
of $\bigO(\nicefrac{\ln k}{k})$
for the \textit{last iterate} of SGD with a constant step size
in a similar rank-1 setting. 
Although our step size is different, 
we believe our bound's dependence on the dimension 
(\thmref{thm:expected_forgetting})
can probably be lifted.
\end{remark}

\begin{remark}[Average iterate analysis]
\label{rmrk:average_iterate}
We note that all our results become much tighter 
when working with the \emph{average iterate} 
$\overline{\w}_k \!\triangleq \!
\tfrac{1}{k}\tsum_{\itr=1}^{k}\w_{\itr}$
instead of the last one $\w_k$.
Then, it is easier to prove tighter bounds,
\ie
$\bigO(\nicefrac{T^2}{k})$ for cyclic orderings
and $\bigO(\nicefrac{1}{k})$ for random ones.
\appsref{app:average_iterate} provides a proof sketch.
\end{remark}

\section{Related Work}
\label{sec:related}
Many practical methods have been proposed over the last decade to address catastrophic forgetting in continual learning.
Clearly we cannot cover all methods here, but refer the reader to recent surveys 
(\eg \citep{parisi2019continual,qu2021recent})
and to \appsref{app:related}.
In a nutshell, algorithmic approaches to continual learning can be roughly divided into three: 
\emph{Memory-based replay approaches} (\eg \citep{robins1995rehearsal,shin2017generativeReplay}); \emph{Regularization approaches} (\eg \citep{kirkpatrick2017overcoming, zeno2021onlineVariational,li2017lwf, lopez2017GEM, benzing2021unifying_regularization}); 
and \emph{Parameter isolation approaches} 
(\eg \citep{aljundi2017expert,yoon2018lifelong,mallya2018packnet}).

\pagebreak

\paragraph{Theoretical results for catastrophic forgetting.}
We briefly discuss few related theoretical works.
\citet{doan2021NTKoverlap} analyzed forgetting in linear regression (or, more generally, in the NTK regime). 
Their derivations largely depend on a matrix
that captures principal angles between tasks,
very much like we do, but they focus on the smallest angle, \ie the Dixmier angle \citep{dixmier1949etude}.
They conclude that increased task similarity leads to more forgetting. 
In contrast, our $T\!=\!2$ analysis reveals the precise role of task similarity at different training stages (see \figref{fig:two_tasks_super}). 
Their work \citep{doan2021NTKoverlap} and others \citep{bennani2020generalisationWithOGD,farajtabar2020OGD} also analyzed the OGD algorithm, deriving generalization guarantees and improved variants.


\citet{lee2021taskSimilarity} considered a teacher-student two-layer setup with $2$ tasks, each having its own last layer.
They analyzed a different notion of task similarity 
(\ie teacher similarity) and studied how it affects forgetting in student models.
They found that intermediate (teacher-)similarity leads to the greatest forgetting.
In our linear setup, we find that when two consecutive tasks are shown \emph{once} (\ie in a single cycle),
intermediate principal angles cause higher forgetting
(see \figref{fig:exact_2_tasks}).
%
\citet{Asanuma_2021} considered a linear regression setting, as we do, yet only for $2$ tasks having a different teacher. 
They assumed specific axis-aligned input distributions,
and studied forgetting in the infinite-dimension limit, taking into account both the input-space and weight-space similarities.
Under a ``shared'' teacher, their results imply zero forgetting due to their assumptions on inputs.

\paragraph{Alternating projections (AP).}
Throughout our paper, we lay out connections between continual learning and the AP literature (see survey by \citep{ginat2018method}).
Our cyclic setting (\secref{sec:cyclic}) is a special case of cyclic AP (\citet{vonNeumann1949rings,halperin1962product}).
Studies from this field often bound the convergence to the subspace intersection either asymptotically or using notions of generalized angles between subspaces
(\eg \cite{kayalar1988error,deutsch1997rate,oppenheim2018angle}).
Most results in this area are uninformative at the worst case.

\paragraph{Kaczmarz method.}
Our fitting procedure (described in \secref{sec:fitting_procedure} and \ref{sec:forgetting_in_linear}) 
shares a great resemblance with 
Kaczmarz methods
for solving a system $\A\x\!=\!\vect{b}$.
The ``basic" method \citep{karczmarz1937angenaherte} 
corresponds to tasks of rank $r\!=\!1$,
while the block variant \citep{elfving1980block} corresponds to tasks having $r\!\ge\!2$. 
Traditionally, these methods used a deterministic cyclic ordering, like we do in \secref{sec:cyclic}.
There is also a randomized variant \citep{strohmer2009randomized,needell2010randomized,xiang2017randomized}, 
related
to our random task ordering in \secref{sec:random}, which often achieves faster convergence both empirically and theoretically \cite{sun2019worst}.
Studies of these methods often analyze convergence to a feasible set using the spectrum of $\A$ (\eg its condition number; see
\cite{needell2014stochastic, oswald2015convergence,haddock2021greed}).

\paragraph{Our convergence analysis is inherently different.}
Most convergence results from the AP and Kaczmarz research areas 
concentrate on bounding the convergence to the subspace intersection.
These results \emph{are} valuable for the continual setting,
\eg they help us bound $\tnorm{\w_k-\teacher}$ in \thmref{thm:2_tasks_distance}, which clearly upper bounds the forgetting.
%
However, most convergence results in these areas depend on the task specifics, precluding any informative \emph{worst-case} analysis.
In contrast, we analyze the convergence of the \emph{forgetting} (\ie projection residuals),
allowing us to derive \emph{uniform} worst-case guarantees.
In \secref{sec:two_tasks}, we demonstrate these differences by comparing the two quantities.
Finally, we believe our bounds are novel and can 
provide a new perspective in these areas, 
especially in the cyclic block Kaczmarz setting \cite{elfving1980block,needell2014paved},
\ie that \emph{worst-case} last-iterate analyses become feasible when examining the residual convergence instead of the distance from $\teacher$.

\paragraph{Normalized Least-Mean-Square (NLMS).}
The NLMS algorithm is a popular choice for adaptive filtering.
In its basic form,
it fits one random sample at each iteration using update rules identical to those of our stochastic setting and the randomized Kaczmarz method.
There are known convergence guarantees for the single-sample algorithm
\cite{slock1993convergence}
and its multiple-samples counterpart, the APA \cite{sankaran2000convergence},
but these are proven only under limiting assumptions on the input signals.

\pagebreak

\paragraph{Other optimization methods.}

It is interesting to note that when fitting $T$ tasks in a cyclic or random ordering, our update rule in \eqref{eq:update_rule} 
is equivalent to dual block coordinate descent on the dual of \eqref{eq:minimum-norm-solution}, 
\ie
 $   {\argmin}_{\vect{\alpha}\in\reals^{N}}
    \bigprn{\,
    \tfrac{1}{2}
    \bignorm{\X_{1:T}^{\top}\vect{\alpha}}^2
    -
    \,
    \y_{1:T}^{\top}\vect{\alpha}\,
    }$,
where $\X_{1:T}\!\in\!\reals^{N\times d}$ and
$\y_{1:T}\!\in\!\reals^{N}$ (for $N\!=\!\sum_{m=1}^T n_m$) are the concatenation of the data matrices and label vectors from all tasks.
The primal variable $\w$ and the dual one $\vect{\alpha}$ are related through $\w\!=\!\X_{1:T}^{\top}\vect{\alpha}$. 
\citet{DBLP:journals/jmlr/Shalev-Shwartz013} analyzed the stochastic dual coordinate ascent method for minimizing a \textit{regularized} loss,
and proved convergence rates for the primal suboptimality. However, when the regularization parameter vanishes, as in our case, their bounds tend to infinity. 
Others \cite{sun2019worst} analyzed the convergence of coordinate descent on quadratic functions like the dual above, and derived data-dependent bounds.


The stochastic proximal point algorithm (SPPA) \cite{RyuBoy:16,Bertsekas2011IncrementalPM,NonasymptoticPatrascu} follows an update rule (for a step size~$\eta$) 
of
${
\w_{\itr}\!=\!\argmin_{\w}\! \Bigprn{
\tfrac{1}{2}(\X_{\tau(\itr)}\w\!-\!\y_{\tau(\itr)})^2 + \tfrac{1}{2\eta}\left\Vert\w\!-\!\w_{\itr-1}\right\Vert^2
}
}
$,
equivalent when $\eta\!\to\!\infty$
to~our rule in \eqref{min_l2}.
As far as we know, no uniform convergence-rates were previously proved for
SPPA.

Minimizing forgetting over $T$ tasks can also be seen as a finite sum minimization problem \cite{WoodworthS16,negiar20a}.
Many algorithms have been proposed for this problem, some of which achieve better convergence rates than ours using
additional \textit{memory} (\eg for storing gradients \cite{NIPS2013_ac1dd209,allenzhu2018katyusha}). 
We derive forgetting convergence rates for the  fitting procedure in \eqref{min_l2}
that is equivalent to running (S)GD to convergence for each presented task, which is a natural choice for continual learning settings.

\paragraph{Forgetting vs. regret.}
\unnotice{
Our \defref{def:forgetting} of the forgetting should not be confused with the notion of regret, 
mainly studied in the context of online learning
(\eg \citep{shalev2012online}). 
Specifically, using our notations, the average regret after $k$ iteration is $\frac{1}{k}\tsum_{t=1}^{k}\tnorm{\X_{\tau(t)}\w_{t-1}-\y_{\tau(t)}}^2$. 
Comparing the two quantities, we note that forgetting quantifies degradation on \emph{previous} tasks, 
while regret captures the ability to predict \emph{future} tasks.
To illustrate the differences, consider a finite sequence of orthogonal tasks. 
\linebreak
In this case, the forgetting is $0$ (as discussed in \secref{sec:no_forgetting}), but the regret is large. Conversely, given a task sequence like in \figref{fig:worst_case}, when $k\to\infty$ the regret vanishes while forgetting goes to $1$ (see \thmref{thm:worst_case}).
Nevertheless, in the cyclic and random settings, 
both quantities will go to $0$ when $k\to\infty$ since we converge to an offline solution. 
However, their rates of convergence may differ.
}

\section{Conclusion}
\label{sec:conclusion}
Catastrophic forgetting is not yet fully understood theoretically. 
Therefore, one must first study it in the simplest model exhibiting this phenomenon --- linear regression.
In this setting, we provide sharp uniform worst-case bounds.
Most of our analysis does not depend on task specifics, 
but only on the number of tasks $T$, 
the number of iterations $k$, and the task ordering. 
On the one hand, we prove that for an arbitrary ordering, forgetting \emph{can} be catastrophic.
On the other hand, for cyclic orderings, 
we prove forgetting \emph{cannot} be catastrophic
and that even in the worst-case it vanishes at most as 
$\frac{T^2}{\sqrt{k}}$ or $\frac{T^2d}{k}$.
Lastly, we prove worst-case bounds for random orderings,
independent of $T$.

Our bounds complement existing Kaczmarz and alternating projection bounds, which focus on convergence to a feasible set.
Unlike ours, their bounds strictly depend on task specifics 
(\eg principal angles or condition numbers)
and
become trivial in worst-case analysis.

There are many intriguing directions for future research.
These include extending our results to the non-realizable case,
analyzing forgetting in classification tasks,
and deriving optimal presentation task orderings. 
Lastly, one can try to extend our results to more complex models
(\eg deep models)
and other training algorithms. 
We hope a thorough theoretical understanding of catastrophic forgetting, will facilitate the development of new practical methods to alleviate it.


\acks{We would like to thank Suriya Gunasekar {(Microsoft Research)} for her insightful comments and suggestions.
We would also like to thank 
Simeon Reich {(Technion)}, 
Rafal Zalas {(Technion)},
and 
Timur Oikhberg {(Univ.~of Illinois)}
for~their fruitful discussions.
R.~Ward was partially supported by {AFOSR~MURI~FA9550-19-1-0005, NSF~DMS~1952735, NSF~HDR1934932, and NSF~2019844}.
N.~Srebro was partially supported by NSF~IIS~award~\#1718970 and the NSF-Simons~Funded~Collaboration on the Mathematics of Deep Learning.
D.~Soudry was supported by the Israel Science Foundation {(Grant No.~1308/18)}
and the Israel Innovation Authority {(the Avatar Consortium)}.
}


\bibliography{99_biblio}

\newpage

\newpage
\appendix
\onecolumn


\section{Preliminary notations and lemmas}

\subsection{Additional notations for the appendices}
\label{app:additional_notations}

We start by adding several notations to those we defined in \secref{sec:setting}.

\begin{enumerate}
\item Like in the main text, $\norm{\cdot}$ denotes either the Euclidean $\ell^2$-norm of a vector or the spectral norm of a matrix.
\item We denote the \textbf{conjugate transpose} (\ie Hermitian transpose) of complex vectors by $\vv^\hop$. 
\item 
We use the \textbf{singular value decomposition}
of \emph{real} matrices,
\ie
$$\X=\U\mSigma\V^{\top}\in\reals^{N\times d},$$
where 
$\U\in\reals^{N\times N}$,
$\V\in\reals^{d\times d}$
are two orthonormal matrices
and
$\mSigma\in\reals^{N\times d}$ has the same rank and dimensions as $\X$. 
We assume w.l.o.g. that $\sigma_1 \ge \sigma_2 \ge \dots \ge \sigma_{\rank\X} \ge \underbrace{~0~=~\dots~=~ 0~}_{\left(d-\rank{\X}\right)\text{ times}}$.
In~addition, 
we often decompose $\V$ 
into 
$\V=
\big[
\underbrace{~~\V^{r}~~}_{d\times \rank{\X}}
~\big|~
\underbrace{~~\V^\perp~~}_{d\times (d-\rank{\X})}
\big]$
to distinguish between the columns of $\V$ that span the range of $\X$ from the ones that span its orthogonal complement.
\end{enumerate}

\vspace{2cm}

\subsection{Useful general properties}
\label{sec:useful_properties}

Following are several 
useful inequalities and properties that will facilitate our proofs. 
We only work with \emph{real} matrices, so we often use the transpose and the Hermitian transpose interchangeably.

\bigskip

\begin{lemma}
\label{lem:square_ineq}
 For any $\x_{1},\dots,\x_{T}\!\in\!\mathbb{C}^{d}$, it holds that $\left\Vert \x_{1}\!+\!\dots\!+\!\x_{T}\right\Vert ^{2}
 \le
 T\!\left(\left\Vert \x_{1}\right\Vert ^{2}+\dots+\left\Vert \x_{T}\right\Vert ^{2}\right)$.
\end{lemma}
\begin{proof}
Notice that
\begin{align*}
    0\le\left\Vert \x_{i}-\x_{j}\right\Vert ^{2}	
    =\left\Vert \x_{i}\right\Vert ^{2}-2\text{Re}\left(\x_{i}^{*}\x_{j}\right)+\left\Vert \x_{j}\right\Vert ^{2}\iff2\text{Re}\left(\x_{i}^{*}\x_{j}\right)\le\left\Vert \x_{i}\right\Vert ^{2}+\left\Vert \x_{j}\right\Vert ^{2}\,.
\end{align*}

Now, we see use the Cauchy-Schwarz inequality and show
\begin{align*}
    \left\Vert \x_{1}+\dots+\x_{T}\right\Vert ^{2}	
    &=\sum_{i=1}^{T}\left\Vert \x_{i}\right\Vert ^{2}+\sum_{i>j}2\text{Re}\left(\x_{i}^{*}\x_{j}\right)\le\sum_{i=1}^{T}\left\Vert \x_{i}\right\Vert ^{2}+\sum_{i>j}\left(\left\Vert \x_{i}\right\Vert ^{2}+\left\Vert \x_{j}\right\Vert ^{2}\right)
    \\
\left[\text{each appears \ensuremath{T} times}\right]
&=
T\sum_{i=1}^{T}\left\Vert \x_{i}\right\Vert ^{2}=T\left(\left\Vert \x_{1}\right\Vert ^{2}+\dots+\left\Vert \x_{T}\right\Vert ^{2}\right)~.
\end{align*}
\end{proof}

\newpage

\begin{property}[Spectral norm properties]
\label{prop:norms}
Let $\A\in\reals^{m\times n}$ and
$\B\in\reals^{n\times p}$.
Then, the spectral norm holds the following properties:
\begin{enumerate}
    \item \textbf{Definition.}
    $\norm{\A}
    =\norm{\A}_2
    \triangleq
    \max_{\vv\in\reals^{n}\setminus\left\{0\right\}}
    \frac{\norm{\A\vv}_2}{\norm{\vv}_2}
    =
    \max_{\substack{\vv\in\reals^{n}:\\
    \norm{\vv}_2=1}}
    \norm{\A\vv}_2
    =\sigma_{1} \left(\A\right)$~;
    \item \textbf{Invariance to transposition.}
    $\norm{\A}=\Vert{\A^\top}\Vert$~;
    \item \textbf{Triangle inequality.}
    $\norm{\A+\B}\le\norm{\A}+\norm{\B}$~;
    \item \textbf{Squared norm of matrix sum.}
    $\norm{\A_1 + \dots + \A_T}^2
    \le
    T \sum_{i=1}^{T} \norm{\A_i}^2$~;
    \item \textbf{Multiplicative norm inequality.}
    $\norm{\A\B}\le\norm{\A}\norm{\B}$~;
    \item \textbf{Invariance to rotations.}
    Let $\V_1\in\reals^{m'\times m}$
    and $\V_2\in\reals^{n'\times n}$ be matrices with orthonormal \emph{columns} ($n'\ge n, m'\ge m$).
    Then,
    $\norm{\V_1\A}=\norm{\A}=\Vert{\A\V_2^\top}\Vert=\Vert{\V_1\A\V_2^\top}\Vert$~.
\end{enumerate}
\end{property}
See Chapter~5.2 in \citet{meyer2000matrix} for the proofs and for more such properties.
The squared norm inequality follows immediately from
\lemref{lem:square_ineq} and the definition of the spectral norm.

\hspace{0.5cm}

\begin{property}[Orthogonal projection properties]
\label{prop:projections}
Let $\mP$ be a \emph{real} orthogonal-projection linear-operator that projects onto a linear subspace $H\subseteq\mathbb{R}^{d}$.
Let $\vv\in \reals^{d}$ be an arbitrary vector.
Then, $\mP$ holds the following properties: 
\begin{enumerate}
    \item 
    \textbf{Geometric definition.}
    $\norm{\vv-\mP\vv}=
    \inf_{\x\in H}\norm{\vv-\x}$~;
    \item \textbf{Symmetry.} $\mP=\mP^\top$~;
    \item \textbf{Idempotence.}
    $\mP^2=\mP$~;
    \item $\I-\mP$ is also a projection operator, projecting onto the subspace orthogonal to $H$.
    
    Consequentially, $\left(\I-\mP\right)\mP=\0$~;
    \item
    Using the above,
    we get 
    $\norm{\left(\I-\mP\right)\vv}^2
    = \vv^*\left(\I-\mP\right)^2\vv
    = \vv^*\left(\I-\mP\right)\vv
    =\norm{\vv}^2-\norm{\mP\vv}^2$~;
    \item \textbf{Contraction.} 
    $\norm{\mP\vv} \le \norm{\vv}$, holding in equality if and only if $\mP\vv=\vv$~;
    \item \textbf{Singular values.} 
    All the singular values of $\mP$ are in $\left\{0,1\right\}$, implying that 
    $\norm{\mP}\!=\!1 \!\iff\! \mP\!\neq\!\0$.
\end{enumerate}
\end{property}
See \citet{zarantonello1971projections} for the proofs and for more properties.

\hspace{0.5cm}

The next corollary stems directly from the properties above
(see Chapter~5.13 in \citet{meyer2000matrix}).
\begin{corollary}[Projection onto solution spaces]
    Let $\X_{m}$ be the matrix of task $m$. 
    Then, the projection onto its solution space,
    or equivalently onto its null space,
    is given by:
    $
        \mP_{m}
        \!=\!
        \I \!-\! 
        \X_{m}^{+}
        \X_{m}
        \!=\!
        \V_{m}
        \left(\I\!-\!\mSigma_{m}^{+}\mSigma_{m}\right)
        \V_{m}^\top
        \!=\!
        \V_{m}^{\perp}\V^{\perp^\top}_{m}
   $.
\end{corollary}

\newpage

\subsection{Principal angles between two tasks}
\label{app:principal_angles}
We continue our discussion from \secref{sec:forgetting_in_linear} and present principal angles more thoroughly for completeness.
We mostly follow the definitions in \citep{bjorck1973numerical,bargetz2020angles,meyer2000matrix}.

\begin{definition}[Principal angles between two tasks]
Let $\X_1\!\in\!\reals^{n_1 \times d},\X_2\!\in\!\reals^{n_2 \times d}$
be two data matrices, corresponding to two given tasks.
Let  $r_{\min} = \min(\rank(\X_1),\rank(\X_2))$
be their minimal rank.
The principal angles $\theta_1, \dots, \theta_{r_{\min}}\in\left[0,\nicefrac{\pi}{2}\right]$
between the two tasks,
\ie the angles between the row spaces of $\X_1,\X_2$,
are recursively defined by
\begin{align*}
\vu_{i},\!
\vv_{i}\!=&
\argmax
\abs{\vu^{\hop}_{i} \vv_{i}}~~
\\
&\hskip25pt \emph{s.t.}~
\vu_i\in\range(\X_1^\top),~
\vv_i\in\range(\X_2^\top),~
\norm{\vu_i}=\norm{\vv_i}=1,~~
\\
&\hskip43pt
\forall j\!\in\!\cnt{i-1}\!:~
\vu_{i}\!\perp\!\vu_{j},~
\vv_{i}\!\perp\!\vv_{j}
\\
\theta_i = &
\arccos\abs{\vu^{\hop}_{i} \vv_{i}}~.
\end{align*}
%
\end{definition}

\medskip

Notice that according to our definition,
the principal angles hold
$\nicefrac{\pi}{2}\ge\theta_{r_{\min}}\ge\dots\ge\theta_1\ge 0$~. 
Important for our analysis is the fact that 
unlike the principal \emph{vectors}
$\left\{\vu_i, \vv_i\right\}_i$,
the principal \emph{angles} between two subspaces are uniquely defined
\citep{bjorck1973numerical}.
Two fundamental principal angles in the field of alternating projections are the minimal principal angle, \ie the Dixmier angle \citep{dixmier1949etude};
and the minimal non-zero principal angle,
\ie the Friedrichs angle \citep{friedrichs1937angles}.

\begin{figure}[h!]
  \begin{minipage}[c]{0.51\textwidth}
    \caption{
    \small
    \textbf{
    Principal angles between subspaces.
    }
    Since the two hyperplanes share an intersecting direction, $\vu_1, \vv_1$ are chosen inside this intersection, 
    meaning that
    ${\theta_1\!=\!\arccos\abs{\vu^{\hop}_{1} \vv_{1}}\!=\!\arccos(1)\!=\!0}$.
    In this case, $\theta_1$ is the Dixmier angle.
    From the remaining directions that are orthogonal to $\vu_1,\vv_1$, the recursive definition chooses two unit vectors $\vu_2,\vv_2$, forming a non-zero angle ${\theta_2\!=\!\arccos\abs{\vu^{\hop}_{2} \vv_{2}}}$.
    In this case, $\theta_2$ is the Friedrichs angle.
    }
  \end{minipage}
  \hfill
  \begin{minipage}[c]{0.48\textwidth}
    \includegraphics[width=.99\textwidth]{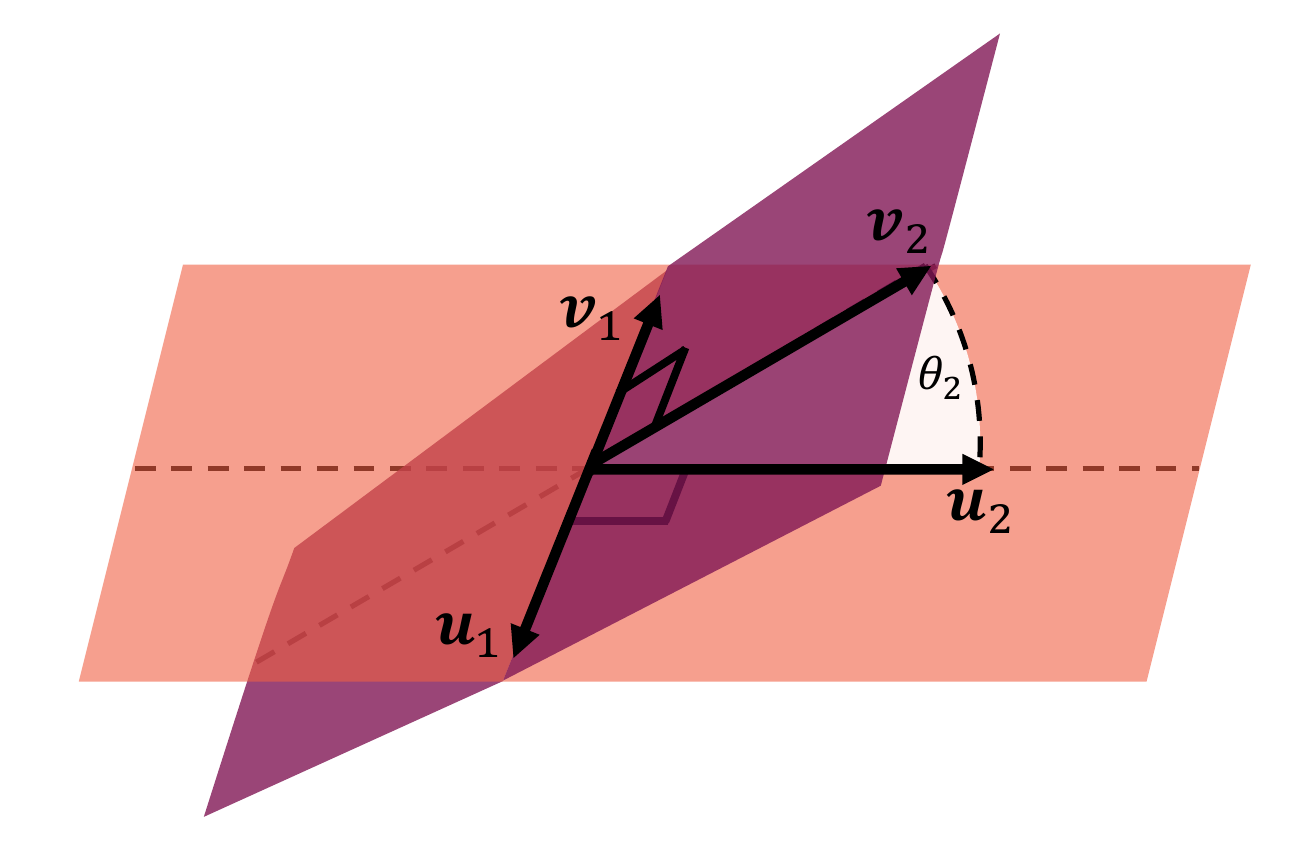}
  \end{minipage}
\end{figure}

\begin{claim}[Principal angles between "equivalent" subspaces]
\label{clm:same_angles}
Let $\bigprn{\X_1,\y_1}, \bigprn{\X_2, \y_2}$ be two tasks.
Then, the sets of \emph{non-zero} principal angles between the following pairs of subspaces, are all the same:
\begin{enumerate}
    \item The data row-spaces, \ie $\range(\X_1^\top)$ and $\range(\X_2^\top)$ ;
    \item The null spaces, \ie $\kernel(\X_1)$ and $\kernel(\X_2)$ ;
    \item The affine solution spaces, \ie 
    $\sol{1}=\teacher+\kernel(\X_1)$ and $\sol{2}=\teacher+\kernel(\X_1)$~. 
\end{enumerate}
\end{claim}
\begin{proof}
The equivalence between (1) and (2) is proven by
Theorem~2.7 of \citet{knyazev2006principal}
(stating that the \emph{non-zero} principal angles between 
two subspaces are essentially the same as those between
their orthogonal complements).
Moreover, since the solution spaces $\sol{1},\sol{2}$ are merely affine subspaces of the null spaces themselves, they induce the exact same principal angles.

It \emph{is} possible however that these pairs of subspaces \emph{do} have a different amount of zero principal angles between them, but this generally does not interfere with our analyses in the appendices.
\end{proof}

\newpage

We now use principal angles to prove a lemma that will facilitate our proofs for Sections~\ref{sec:arbitrary}~and~\ref{sec:cyclic}.

\begin{lemma}
\label{lem:principal_form}
Let $\X_1,\X_2$ be two data matrices of two tasks
and let the corresponding orthogonal projections onto their null spaces
be
${\mP_1=\I-\X_1^+\X_1,~
\mP_2=\I-\X_2^+\X_2}$.
Then, for any $n\in\naturals$ we have that
\begin{align*}
\norm{\left(\I-\mP_{1}\right)
	\left(\mP_{2}\mP_{1}\right)^{n}}^{2}
=
\max_{i}
\left\{
\left(\cos^2 \theta_i\right)^{2n-1}(1-\cos^2 \theta_i)
\right\}~,
\end{align*}
where 
$\{\theta_i\}_i \subseteq
\left(0, \tfrac{\pi}{2}\right]$
are the non-zero principal angles between the two tasks in $\coll$,
\ie between $\range(\X_1^\top)$ and $\range(\X_2^\top)$ or equivalently
between $\kernel(\X_1)$ and $\kernel(\X_2)$
or between the affine solution spaces $\sol{1}$ and $\sol{2}$. 
\end{lemma}

\begin{proof}
We use the SVD notation defined in \appref{app:additional_notations}
and denote by $\V_1^\perp,\,\V_2^\perp$ the matrices whose orthonormal columns span $\kernel(\X_1),\,\kernel(\X_2)$ respectively.
Thus, we can express the projections as 
$\mP_1=\V_1^\perp\V_1^{\perp^\top}$ and $\mP_2=\V_2^\perp\V_2^{\perp^\top}$.

Now, we show that
\begin{align*}
\begin{split}
    \norm{\left(\I-\mP_{1}\right)
	\left(\mP_{2}\mP_{1}\right)^{n}}^{2}
	&=
	\norm{\left(\I-\mP_{1}\right)
	\left(
	\V_{2}^{\perp}\V_{2}^{\perp^\top}
	\V_{1}^{\perp}\V_{1}^{\perp^\top}
	\right)^{n}}^{2}
	\\
	&
	=
	\norm{\left(\I-\mP_{1}\right)
	\V_{2}^{\perp}\V_{2}^{\perp^\top}
	\V_{1}^{\perp}
	\left(
	\V_{1}^{\perp^\top}
	\V_{2}^{\perp}\V_{2}^{\perp^\top}
	\V_{1}^{\perp}
	\right)^{n-1}
	\V_{1}^{\perp^\top}}^{2}
	\\
	\explain{\substack{\text{spectral norm}\\ \text{properties}}}
	&
	=
	\norm{\left(\I-\mP_{1}\right)
	\V_{2}^{\perp}\V_{2}^{\perp^\top}
	\V_{1}^{\perp}
	\left(
	\V_{1}^{\perp^\top}
	\V_{2}^{\perp}\V_{2}^{\perp^\top}
	\V_{1}^{\perp}
	\right)^{n-1}}^{2}~.
\end{split}
\end{align*}
We now notice that the idempotence of $\mP_1=\V_{1}^{\perp}\V_{1}^{\perp^\top}$
and
$\mP_2 = \V_{2}^{\perp}\V_{2}^{\perp^\top}$
implies
\begin{align*}
	&
	\V_{1}^{\perp^\top}
	\V_{2}^{\perp}\V_{2}^{\perp^\top}
	\left(\I-\mP_{1}\right)^\top
	\left(\I-\mP_{1}\right)
	\V_{2}^{\perp}\V_{2}^{\perp^\top}
	\V_{1}^{\perp}
	=
	\V_{1}^{\perp^\top}
	\V_{2}^{\perp}\V_{2}^{\perp^\top}
	\left(\I-\mP_{1}\right)
	\V_{2}^{\perp}\V_{2}^{\perp^\top}
	\V_{1}^{\perp}
	\\
	&\hspace{1cm}=
	\V_{1}^{\perp^\top}
	\V_{2}^{\perp}\V_{2}^{\perp^\top}
	\V_{1}^{\perp}
	-
	\V_{1}^{\perp^\top}
	\V_{2}^{\perp}\V_{2}^{\perp^\top}
	\V_{1}^{\perp}\V_{1}^{\perp^\top}
	\V_{2}^{\perp}\V_{2}^{\perp^\top}
	\V_{1}^{\perp}
	\\
	&\hspace{1cm}=
	\V_{1}^{\perp^\top}
	\V_{2}^{\perp}\V_{2}^{\perp^\top}
	\V_{1}^{\perp}
	-
	\prn{
	\V_{1}^{\perp^\top}
	\V_{2}^{\perp}\V_{2}^{\perp^\top}
	\V_{1}^{\perp}
	}^2~.
\end{align*}
Since for the spectral norm of any real matrix $\A$ it holds that
$\norm{\A}^2=\norm{\A^\top \A}$,
we get that
\begin{align*}
	&\norm{\left(\I-\mP_{1}\right)
	\V_{2}^{\perp}\V_{2}^{\perp^\top}
	\V_{1}^{\perp}
	\left(
	\V_{1}^{\perp^\top}
	\V_{2}^{\perp}\V_{2}^{\perp^\top}
	\V_{1}^{\perp}
	\right)^{n-1}}^{2}
	\\
	&
	=
	\norm{
	\left(
	\V_{1}^{\perp^\top}
	\V_{2}^{\perp}\V_{2}^{\perp^\top}
	\V_{1}^{\perp}
	\right)^{n-1}
	\V_{1}^{\perp^\top}
	\V_{2}^{\perp}\V_{2}^{\perp^\top}
	\left(\I-\mP_{1}\right)^2
	\V_{2}^{\perp}\V_{2}^{\perp^\top}
	\V_{1}^{\perp}
	\left(
	\V_{1}^{\perp^\top}
	\V_{2}^{\perp}\V_{2}^{\perp^\top}
	\V_{1}^{\perp}
	\right)^{n-1}
	}
	\\
	&
	=
	\norm{
	\left(
	\V_{1}^{\perp^\top}
	\V_{2}^{\perp}\V_{2}^{\perp^\top}
	\V_{1}^{\perp}
	\right)^{2n-1}
	-
	\left(
	\V_{1}^{\perp^\top}
	\V_{2}^{\perp}\V_{2}^{\perp^\top}
	\V_{1}^{\perp}
	\right)^{2n}
	}~.
\end{align*}

\newpage

Denote the spectral decomposition of the 
Gram matrix 
$\V_{1}^{\perp^\top}
    \V_{2}^{\perp}
    \V_{2}^{\perp^\top}
\V_{1}^{\perp}\succeq\0$
as
$\Q\mLambda \Q^\top$,
with~its (non-negative) eigenvalues ordered in a non-ascending order on the diagonal of $\mLambda$ and $\Q$ being some orthonormal matrix.
The upper bound thus becomes
\begin{align*}
	\norm{
	\bigprn{
    \Q\mLambda \Q^\top
	}^{2n-1}
	-
	\bigprn{
	\Q\mLambda \Q^\top
	}^{2n}
	}
    &=
	\norm{\mLambda^{2n-1} - \mLambda^{2n}}~.
    \\
    &
=
\max_{i}
\lambda_i^{2n-1}(1-\lambda_i)
=
\max_{i}
\left\{
\left(\cos^2 \theta_i\right)^{2n-1}(1-\cos^2 \theta_i)
\right\}~,
\end{align*}
where
$\{\theta_i\}_i \subseteq
\left[0, \tfrac{\pi}{2}\right]$
are all the principal angles (zeros included)
between $\kernel(\X_1)$ and $\kernel(\X_2)$.
The last equality stems from a known analysis result
(Theorem~2.1 in \citet{risteski2001principal}) 
relating the principal angles 
between $\kernel(\X_1)$ and $\kernel(\X_2)$
to the eigenvalues of the Gram matrix we defined,
\ie
${\V_{1}^{\perp^\top}
    \V_{2}^{\perp}
    \V_{2}^{\perp^\top}
\V_{1}^{\perp}=\Q\mLambda \Q^\top}$. 
More formally, the result states 
that $\forall i$, $\lambda_i = \cos^2 \theta_i$.
See also Chapter~5.15 in \citet{meyer2000matrix} for more detailed explanations on this relation.

\bigskip

We conclude this proof by using
Claim~\ref{clm:same_angles} showing that the \emph{non-zero} principal angles between 
$\kernel(\X_1)$ and $\kernel(\X_2)$
are essentially the same as those between
their orthogonal complements, 
\ie
$\range(\X^\top_1)$ and $\range(\X^\top_2)$.
\end{proof}

\newpage
\subsubsection{Auxiliary lemmas on forgetting}

We first state an auxiliary lemma for deriving lower bounds on the forgetting of tasks of rank $d\!-\!1$.

\begin{lemma}[Lower bound on forgetting in the $r=d-1$ case]
\label{lem:forgetting_lower_bound}
Let $\coll\in\Coll$ be a task collection with $T$ data matrices of rank $r=d-1$.
Let $\vv_1,\dots,\vv_T$
be the normalized vectors spanning the $T$ rank-\emph{one} solution spaces $\sol{1},\dots,\sol{T}$.
Let $\ang{i}{j}\in\left[0,\pi/2\right]$
be the angle between $\vv_{\tau(i)}$ and $\vv_{\tau(j)}$, $\forall i,j\in\cnt{k}$,
meaning that 
$\cos^2 \ang{i}{j} 
\!= \!
\bigprn{\vv_{\tau(i)}^\top
\vv_{\tau(j)}}^2$.
Finally, let $\teacher$ be the minimum norm offline solution of $\coll$.
Then, the forgetting on $\coll$ after $k$ iterations is lower bounded by
\begin{align*}
    F_{\tau,\coll}(k)
    \ge
    \min_{m\in\cnt{T}}
    \big\{\sigma_{\min}^2 (\X_m)\big\}
    \left(
        \vv_{\tau(1)}^\top
        \teacher
    \right)^2
    \frac{1}{k}
    \left(
        \prod_{i=1}^{k-1}
        \cos^2 \ang{i}{i+1} 
    \right)
    \sum_{j=1}^{k-1}
    \left(1-
        \cos^2 \ang{j}{k} 
    \right)
    ~,
\end{align*}
\end{lemma}
where $\sigma_{\min}^2 (\X_m)$ is the smallest squared non-zero singular value of $\X_{m}$.

\bigskip

\begin{proof}[Proof for \lemref{lem:forgetting_lower_bound}]
We start from \eqref{eq:linear_forgetting} 
and perform similar derivations to the ones that lead us to \eqref{eq:forgetting-projection}.
We have,
\begin{align*}
\begin{split}
    F_{\tau,S}\left(k\right)
    &=
    \frac{1}{k}
    \sum_{\itr=1}^{k}
    \norm{
    \X_{\tau(\itr)} {\mP_{\tau(k)}\cdots\mP_{\tau(1)}}
    \w^{\star}}^2
    \\
    \explain{\text{pseudo-inverse}}
    &=
    \frac{1}{k}
    \sum_{\itr=1}^{k}
    \bignorm{
    \X_{\tau(\itr)}
    \bigprn{\X_{\tau(\itr)}^+\X_{\tau(\itr)}}
    {\mP_{\tau(k)}\cdots\mP_{\tau(1)}}
    \w^{\star}}^2
    \\
    \explain{\text{spectral}\\\text{norm}\\\text{properties}}
    &\ge
    \frac{1}{k}
    \sum_{\itr=1}^{k}
    \sigma_{\min}^2 (\X_{\tau(\itr)})
    \bignorm{
    {\bigprn{\X_{\tau(\itr)}^+\X_{\tau(\itr)}}}
    {\mP_{\tau(k)}\cdots\mP_{\tau(1)}}
    \teacher}^2
    \vspace{-1mm}
    \\
    &\ge
    \min_{m\in\cnt{T}}
    \big\{\sigma_{\min}^2 (\X_m)\big\}
    \frac{1}{k}
    \sum_{\itr=1}^{k}
    \norm{
    \bigprn{\I-\mP_{\tau(\itr)}}
    {\mP_{\tau(k)}\cdots\mP_{\tau(1)}}
    \teacher}^2
    \\
    \explain{r=d-1}
    &
    =
    \min_{m\in\cnt{T}}
    \big\{\sigma_{\min}^2 (\X_m)\big\}
    \frac{1}{k}
    \sum_{\itr=1}^{k-1}
    \norm{
    \bigprn{\I-
    \vv_{\tau(\itr)}\vv_{\tau(\itr)}^\top}
    \vv_{\tau(k)}\vv_{\tau(k)}^\top
    \cdots
    \vv_{\tau(1)}\vv_{\tau(1)}^\top
    \teacher}^2~,
\end{split}
\end{align*}
where we used the rank-1 SVD of the data matrices,
\ie $\X_m = \vv_m \vv_m^{\top}$.

Now, we denote
$\forall i,j: 
\abs{\vv_{\tau(i)}^\top
\vv_{\tau(j)}}
=
\theta_{i,j}$
and get 
(notice the signs do not matter in the norm),
\begin{align*}
    &=\!
    \min_{m\in\cnt{T}}\!\!
    \big\{\sigma_{\min}^2 (\X_m)\big\}
    \frac{1}{k}
    \sum_{\itr=1}^{k-1}
    \biggnorm{
    \bigprn{\I-
    \vv_{\tau(\itr)}\vv_{\tau(\itr)}^\top}
        \prn{
        \prod_{i=1}^{k-1}
        \cos \ang{i}{i+1} 
        \!}
        \vv_\perp^{\tau(k)}
        \vv_\perp^{\tau(1)^\top}
        \teacher
    }^2
    \\
    &=\!
    \min_{m\in\cnt{T}}\!\!
    \big\{\sigma_{\min}^2 (\X_m)\big\}
    \frac{\tprn{
        \vv_{\tau(1)}^\top
        \teacher
    }^2}{k}\!
    \prn{
        \prod_{i=1}^{k-1}
        \cos^2 \ang{i}{i+1} 
        \!}
    \sum_{j=1}^{k-1}
    \norm{
        \left(\I\!-\! \vv_\perp^{\tau(j)}\vv_\perp^{\tau(j)^\top}\right)
        \vv_\perp^{\tau(k)}
    }^2
    \\
    \explain{\text{idempotence}}
    &=\!
    \min_{m\in\cnt{T}}\!\!
    \big\{\sigma_{\min}^2 (\X_m)\big\}
    \frac{\tprn{
        \vv_{\tau(1)}^\top
        \teacher
    }^2}{k}\!
    \prn{
        \prod_{i=1}^{k-1}
        \cos^2 \ang{i}{i+1} 
        \!}
    \sum_{j=1}^{k-1}\!
    \!\vv_\perp^{\tau(k)^\top}
    \!\!\!
    \left(\I\!-\!\vv_\perp^{\tau(j)}\vv_\perp^{\tau(j)^\top}\right)\!
    \vv_\perp^{\tau(k)}
    ~
    \\
    &=\!
    \min_{m\in\cnt{T}}\!\!
    \big\{\sigma_{\min}^2 (\X_m)\big\}
    \frac{\tprn{
        \vv_{\tau(1)}^\top
        \teacher
    }^2}{k}\!
    \prn{
        \prod_{i=1}^{k-1}
        \cos^2 \ang{i}{i+1} 
        }
    \sum_{j=1}^{k-1}\!
    \left(1\!-\!\cos^2 \ang{j}{k} 
    \right)
    ~.
\end{align*}
\vspace{-4mm}
\end{proof}

\newpage

\section{Supplementary material: Arbitrary task orderings (\secref{sec:arbitrary})}
\label{app:proofs1}

\subsection{No forgetting cases (\secref{sec:no_forgetting})}
\label{app:no_forgetting}

Consider learning two tasks sequentially:
$\big(\X_{1},\y_{1}\big)$
and then $\big(\X_{2},\y_{2}\big)$.
Right after learning the second task,
we have
$\mathcal{L}_{2}(\w_2)
\!=\!
\tnorm{\y_2 \!-\! \X_2 \w_2}^2 
\!=\! 0$.
Thus, the forgetting from \eqref{eq:linear_forgetting} becomes 
${
    F_{\tau,S}(2)
    =
    \tfrac{1}{2}
    \norm{
    \X_1
    \mP_{2}\mP_{1}
    \w^\star}^2
}
$.
We now derive sufficient and necessary conditions for no forgetting.

\medskip

\begin{recall}[\thmref{thm:noforgetting}]
Let
$\coll\!=\!\left\{\big(\X_{1},\y_{1}\big),\big(\X_{2},\y_{2}\big)\right\}
\!\in\!\mathcal{S}_{T=2}$ 
be a task collection with $2$ tasks,
fitted under an identity ordering $\tau$,
\ie $\big(\X_{1},\y_{1}\big)$~and then $\big(\X_{2},\y_{2}\big)$.
Then the following conditions are equivalent:
\begin{enumerate}
    \item 
    For \emph{any} labeling $\y_1,\y_2$
    (or equivalently, any minimum norm solution $\teacher$),
    after fitting the second task, the model does not ``forget'' the first one.
    That is,
    $F_{\tau,\coll}(2)=0$.
    \item
    It holds that
    $\X_1 \mP_2 \mP_1 = 
    \0_{n_1 \times d}$.
    \item Each principal angle between the tasks, \ie $\range({\X_1^\top})$ and
    $\range({\X_2^\top})$, is either $0$ or $\nicefrac{\pi}{2}$.
\end{enumerate}
\end{recall}

\medskip

\subsubsection{Example: Sufficient conditions for no forgetting}
\label{sec:example_conditions}
Before we prove the theorem above,
we exemplify some of its implications by showing clear and simple sufficient conditions for holding the conditions of the theorem.

\begin{enumerate}[label=(\alph*),leftmargin=2cm]\itemsep2pt
\item $\mathrm{range}\big(\X_{2}^\top\big)\subseteq \mathrm{range}\big(\X_{1}^\top\big)$; 
or
$\mathrm{range}\big(\X_{2}^\top\big)\supseteq \mathrm{range}\big(\X_{1}^\top\big)$; 
or
\item $\mathrm{range}\big(\X_{2}^\top\big)\subseteq \kernel\big(\X_{1}\big)$; 
or
$\mathrm{range}\big(\X_{2}^\top\big)\supseteq \kernel\big(\X_{1}\big)$.
\end{enumerate}
These conditions can help understand that maximal task (=sample) similarity or dissimilarity can help prevent forgetting in the linear setting.


\bigskip

\subsubsection{Proving the theorem}

\begin{proof}
\label{prf:noforgetting}
Like we explain in \secref{sec:no_forgetting},
the forgetting after learning the second task is equal to 
${
    F_{\tau,S}(2)
    =
    \mathcal{L}_{1}(\w_2)   
    =
    \tfrac{1}{2}
    \norm{
    \X_1
    \mP_{2}\mP_{1}
    \w^\star}^2
}
$.
We notice that $\mP_1=\I-\X_1^+ \X_1$
and
$\mP_2=\I-\X_2^+ \X_2$,
meaning that the labels $\y_1,\y_2$ do not have any effect on the matrix 
$\X_1 \mP_2 \mP_1$.
However, these labels \emph{do} effect the minimum norm solution $\teacher$.

Here, we briefly discuss the relation between $\teacher$ and $\y_1,\y_2$,
so as to facilitate our proof below.
\paragraph{Relating the minimum norm solution and the labelings.}
Recall the constraints between the offline solution and the labels, \ie $\X_1\teacher=\y_1$ and $\X_2\teacher=\y_2$.
Also recall that under Assumption~\ref{assume:realizability}, we have that $\teacher\in\ball{d}$.
Note that a minimum norm solution \emph{must} lie in the row span of \emph{both} data matrices,
\ie
${\teacher\in\range({\X_1^\top}) \cup
\range({\X_2^\top})}$,
since any contributions from the nullspaces will not affect its predictions $\X_1 \teacher$ and $\X_2 \teacher$ but \emph{will} increase its norm.

Moreover,
notice that the $\y_1,\y_2$ can yield \emph{any} minimum norm solution that is inside $\range(\X_1^\top)$, since we could just
choose an arbitrary vector $\teacher\in\range(\X_1^\top)$
and set $\y_2=\X_2 \teacher$ (we do not have restrictions on the labelings).

We are now ready to complete our proof.

\newpage

\paragraph{Condition (1) $\iff$ Condition (2).}
Clearly, since ${
    F_{\tau,S}(2)
    \!=\!
    \tfrac{1}{2}
    \norm{
    \X_1
    \mP_{2}\mP_{1}
    \w^\star}^2
}
$,
we have that 
$$\underbrace{\X_1\mP_{2}\mP_{1}=\0}_{\text{(2)}}
~\Longrightarrow~
\underbrace{
\forall \y_1, \y_2, \teacher:~
F_{\tau,S}(2)=
    \tfrac{1}{2}
    \norm{
    \X_1
    \mP_{2}\mP_{1}
    \w^\star}^2=0}_{\text{(1)}}~.$$
Since $\mP_1$ is a symmetric operator projecting onto the row span of $\X_1$, we have that
$\range(\mP_1)
\!=\!
\range(\mP_1^\top)
\!=\!
\range(\X^\top_1)$.
It is readily seen that
$$
\prn{\forall \teacher \!\in \range(\X^\top_1)
:\,
\norm{
    \X_1
    \mP_{2}\mP_{1}
    \teacher}\!=\!0
}
~\iff~
    \X_1
    \mP_{2}\mP_{1}
    =
    \0~.
$$

We explicitly denote the minimum norm solution that two labelings $\y_1,\y_2$ induce as
$\teacher(\y_1, \y_2)$.
As explained, $\y_1,\y_2$ can yield \emph{any} minimum solution inside $\range(\X_1^\top)$.
Assume $\forall \y_1,\y_2\!: \norm{
    \X_1
    \mP_{2}\mP_{1}
    \teacher(\y1,\y2)}\!=\!0$.
Then, it follows that 
$\forall \teacher \!\in \range(\X^\top_1)
:\,
\norm{
    \X_1
    \mP_{2}\mP_{1}
    \teacher}\!=\!0$.
In this case we proved that it follows that $\X_1\mP_2\mP_1=\0$.

Overall we showed:
$$
\underbrace{\prn{\forall \y_1,\y_2 \!\in \range(\X^\top_1)
:\,
\norm{
    \X_1
    \mP_{2}\mP_{1}
    \teacher}\!=\!0
}}_{\text{(1)}}
~\Longrightarrow~
    \underbrace{\X_1
    \mP_{2}\mP_{1}
    =
    \0}_{\text{(2)}}~,
$$
thus completing the proof for (1) $\iff$ (2).

\bigskip

\paragraph{Condition (2) $\iff$ Condition (3).}
First, we notice that
$$
\left\Vert\X_1\boldsymbol{P}_{2}\boldsymbol{P}_{1}\right\Vert 
=
\0
\iff
\left\Vert
\left(\boldsymbol{I}-\boldsymbol{P}_{1}\right)\boldsymbol{P}_{2}\boldsymbol{P}_{1}\right\Vert =\0$$
since simple norm properties give us:
\begin{align*}
    &\left\Vert \left(\boldsymbol{I}-\boldsymbol{P}_{1}\right)\boldsymbol{P}_{2}\boldsymbol{P}_{1}\right\Vert 	
    =\left\Vert \boldsymbol{V}_{1}\boldsymbol{\Sigma}_{1}^{+}\boldsymbol{\Sigma}_{1}\boldsymbol{V}_{1}^{\top}\boldsymbol{P}_{2}\boldsymbol{P}_{1}\right\Vert =\left\Vert \boldsymbol{\Sigma}_{1}^{+}\boldsymbol{\Sigma}_{1}\boldsymbol{V}_{1}^{\top}\boldsymbol{P}_{2}\boldsymbol{P}_{1}\right\Vert 
	\\
	&
	\le
	\underbrace{\left\Vert \boldsymbol{\Sigma}_{1}^{+}\right\Vert}_{>0}
	\left\Vert \boldsymbol{\Sigma}_{1}\boldsymbol{V}_{1}^{\top}\boldsymbol{P}_{2}\boldsymbol{P}_{1}\right\Vert 
	\propto
	\left\Vert \boldsymbol{\Sigma}_{1}\boldsymbol{V}_{1}^{\top}\boldsymbol{P}_{2}\boldsymbol{P}_{1}\right\Vert 
	=\left\Vert \boldsymbol{U}_{1}\boldsymbol{\Sigma}_{1}\boldsymbol{V}_{1}^{\top}\boldsymbol{P}_{2}\boldsymbol{P}_{1}\right\Vert 
	=\left\Vert \boldsymbol{X}_{1}\boldsymbol{P}_{2}\boldsymbol{P}_{1}\right\Vert~,
\end{align*}
and
\begin{align*}
    &\left\Vert \boldsymbol{X}_{1}\boldsymbol{P}_{2}\boldsymbol{P}_{1}\right\Vert 	
    =
    \left\Vert \boldsymbol{U}_{1}\boldsymbol{\Sigma}_{1}\boldsymbol{V}_{1}^{\top}\boldsymbol{P}_{2}\boldsymbol{P}_{1}\right\Vert =\left\Vert \boldsymbol{\Sigma}_{1}\boldsymbol{V}_{1}^{\top}\boldsymbol{P}_{2}\boldsymbol{P}_{1}\right\Vert 
	=\left\Vert \boldsymbol{\Sigma}_{1}\left(\boldsymbol{\Sigma}_{1}^{+}\boldsymbol{\Sigma}_{1}\right)\boldsymbol{V}_{1}^{\top}\boldsymbol{P}_{2}\boldsymbol{P}_{1}\right\Vert 
	\\
	&
	\le\underbrace{\left\Vert \boldsymbol{\Sigma}_{1}\right\Vert}_{>0}
	\left\Vert \boldsymbol{\Sigma}_{1}^{+}\boldsymbol{\Sigma}_{1}\boldsymbol{V}_{1}^{\top}\boldsymbol{P}_{2}\boldsymbol{P}_{1}\right\Vert 
	\\&
	\propto
	\left\Vert \boldsymbol{\Sigma}_{1}^{+}\boldsymbol{\Sigma}_{1}\boldsymbol{V}_{1}^{\top}\boldsymbol{P}_{2}\boldsymbol{P}_{1}\right\Vert =\left\Vert \boldsymbol{V}_{1}\boldsymbol{\Sigma}_{1}^{+}\boldsymbol{\Sigma}_{1}\boldsymbol{V}_{1}^{\top}\boldsymbol{P}_{2}\boldsymbol{P}_{1}\right\Vert =\left\Vert \left(\boldsymbol{I}-\boldsymbol{P}_{1}\right)\boldsymbol{P}_{2}\boldsymbol{P}_{1}\right\Vert. 
\end{align*}

Then, we use \lemref{lem:principal_form}
and get that
$$\left\Vert \left(\boldsymbol{I}-\boldsymbol{P}_{1}\right)\boldsymbol{P}_{2}\boldsymbol{P}_{1}\right\Vert 
= 
\max_{i}
\left\{
\left(\cos^2 \theta_i\right)(1-\cos^2 \theta_i)
\right\}
= 
\frac{1}{4}
\max_{i}
\left\{\sin^2 (2\theta_i)\right\}~,$$
where 
$\{\theta_i\}_i \subseteq
\left(0, \tfrac{\pi}{2}\right]$
are the non-zero principal angles between the two tasks,
\ie between $\range(\X_1^\top)$ and $\range(\X_2^\top)$.
Finally, it is now clear that
$$
\left\Vert\X_1\boldsymbol{P}_{2}\boldsymbol{P}_{1}\right\Vert 
=
\0
\iff
\left\Vert
\left(\boldsymbol{I}-\boldsymbol{P}_{1}\right)\boldsymbol{P}_{2}\boldsymbol{P}_{1}\right\Vert = \frac{1}{4}
\max_{i}
\left\{\sin^2 (2\theta_i)\right\} = \0
\iff
\forall i: \theta_i \in \left\{0,\frac{\pi}{2}\right\}
~.$$
\end{proof}


\newpage

\subsubsection{Comparison to \citet{doan2021NTKoverlap}}
\label{app:compare_to_doan}
\citet{doan2021NTKoverlap} also studied  forgetting in a linear setting where a series of tasks are learned sequentially by SGD on squared loss with ridge penalty. 
In the special case where only $T=2$ tasks are given
and no regularization is used (\ie $\lambda=0$), 
their expression for forgetting in Theorem~1  \citep{doan2021NTKoverlap} is equivalent to our derivation in \eqref{eq:linear_forgetting} up to scaling by $\frac{1}{2}$.
However, their subsequent upper bound stated in Corollary~1 is a looser characterization of forgetting, which can be paraphrased in terms of our notation and framework as follows 
(for convenience, we attach their notations beneath the last equation):
\begin{equation}
\begin{split}
    F_{\tau,S}(2)
    &=
    \frac{1}{2}
    \norm{
    \X_{1}
    \mP_2 \mP_1
    (\w_{0}-\w^\star)
    }^2
    =
    \frac{1}{2}
    \norm{
    \X_{1}
    \cdot
    \X_{1}^+
    \X_{1}
    \mP_2 \mP_1
    (\w_{0}-\w^\star)
    }^2
    \\
    &=
    \frac{1}{2}
    \norm{
    \X_{1}
    \tprn{\I-\mP_1}
    \mP_2 \mP_1
    (\w_{0}\!-\!w^\star)
    }^2
    \\
    &
    \overset{(*)}=
    \frac{1}{2}
    \|
    \X_{1}
    \tprn{\I-\mP_1}
    \tprn{\I-\mP_2}
    \mP^{(1)}_\perp(\w_{0}\!-\!\w^\star)\|^2
    \\
    &
    \le \frac{1}{2}\|\X_1\|^2
    \bignorm{
    \underbrace{
    \tprn{\I-\mP_1}
    \tprn{\I-\mP_2}}_{
    \Theta^{1\to 2}
    }}^2
    \bignorm{
    \underbrace{\mP_1(\w_{0}-\w^\star)}_{
    = \M_2 \tilde{\y}_2
    }}^2,
\end{split}
\label{doan_our}
\end{equation} 
where (*) follows from 
plugging in
$\mP_2 = \mP_2 - \I + \I$ and $(\I-\mP_1)\mP_{1}=0.$

Based on the above upper bound, the authors informally argue that higher similarity of the principal components between the source task
and target task leads to higher risk of forgetting. In contrast, 
the sufficient condition (a) in our \appref{sec:example_conditions}
shows that there is no forgetting when tasks have maximum overlap in the row spans of their inputs.

Concretely, consider the following two tasks that hold
the sufficient conditions of \thmref{thm:noforgetting}
presented in \appref{sec:example_conditions}:
$$\left(\X^{(1)}=[1,0,0,0],~\y^{(1)}=\frac{1}{\sqrt{2}}[1]\right),~~~
\left(\X^{(2)}=\left[\begin{array}{cccc}1&0&0&0\\0&1&0&0\end{array}\right],~\y^{(2)}=\frac{1}{\sqrt{2}}\left[\begin{array}{c}1\\1\end{array}\right]\right)~.$$
Note that $\w^\star=\frac{1}{\sqrt{2}}[1,1,0,0]^\top$ is a unit norm linear predictor that realizes both tasks. In this case, there is clearly no forgetting as the first task is also part of second task, \ie $F_{\tau,\coll}(2)=0$.
However, we can verify that $\|\X_1\|\|\mP^{(1)}\mP^{(2)}\|=1$ and $\|\mP_1(\w_{0}-\w^\star)\|=\nicefrac{1}{\sqrt{2}}$, thus the upper bound in \eqref{doan_our} evaluates to $F_{\tau,S}(2)\le\nicefrac{1}{4}$. This demonstrates the weakness of the upper bound in Corollary $1$ of \cite{doan2021NTKoverlap}.

This gap can also be seen from a principal angle perspective.
Their so-called overlap matrix,
\ie $\Theta^{1\to 2}$,
is a diagonal matrix holding the singular values 
of $\V_2^{\perp^\top}\V_1^{\perp}$.
As they explain in Corollary~1 and we explain in 
the proof of \lemref{lem:principal_form},
these singular values are actually connected to the principal angles between $\X_1$ and $\X_2$.
Thus, when they use the spectral norm $\norm{\Theta^{1\to 2}}$ which equals $1$, they are actually using \emph{only} the largest singular values, \ie the \emph{smallest} principal angle which is called the Dimixer angle \citep{dixmier1949etude}.
In the above example, this angle, which is the only principal angle, is zero (thus holding our conditions from \thmref{thm:noforgetting}).
In contrast, our analysis in \lemref{lem:2_tasks_forgetting} uses \emph{all}
principal angles, revealing more delicate effects of task similarity on forgetting dynamics. 
\newpage

\subsection{Maximal forgetting cases (\secref{sec:maximal_forgetting})}
\label{app:maximal_forgetting}

\begin{recall}[\thmref{thm:worst_case}]
When using the identity ordering 
(\ie $\tau\!\left(t\right)\!=t$),
thus seeing each task once,
the worst-case forgetting after $k$ iterations is arbitrarily bad,
\ie
\vspace*{-.1cm}
\begin{align*}
    1-\!
    \sup_{
    \substack{
        \coll\in\mathcal{S}_{T=k}
    }
    }
    \!\!
    F_{\tau,\coll}(k)
    \le
    \bigO\left(\nicefrac{1}{\sqrt{k}}\right)~.
\end{align*}
\end{recall}

\bigskip

\begin{figure}[h!]
\centering
\begin{minipage}{.8\textwidth}
    \begin{centering}
    \caption*{\textbf{Recall \figref{fig:worst_case}}.
    Sequence of tasks where $F_{\tau,\coll}(k) \rightarrow 1$.
    Each black arrow represents a solution space of a rank $1$ task in $\reals^2$.
    There are
    $k_1$ tasks between $[0,\theta]$
    and $k_2$ tasks between $[\theta,\pi/2]$. 
    The green diamond shows the overall contraction after fitting all tasks.
    The red arrows show projections back onto the solution spaces.
    The mean squared length of these arrows is the forgetting.
    %
    }
    \end{centering}
\end{minipage}
    \centering
    \includegraphics[width=0.5\columnwidth]{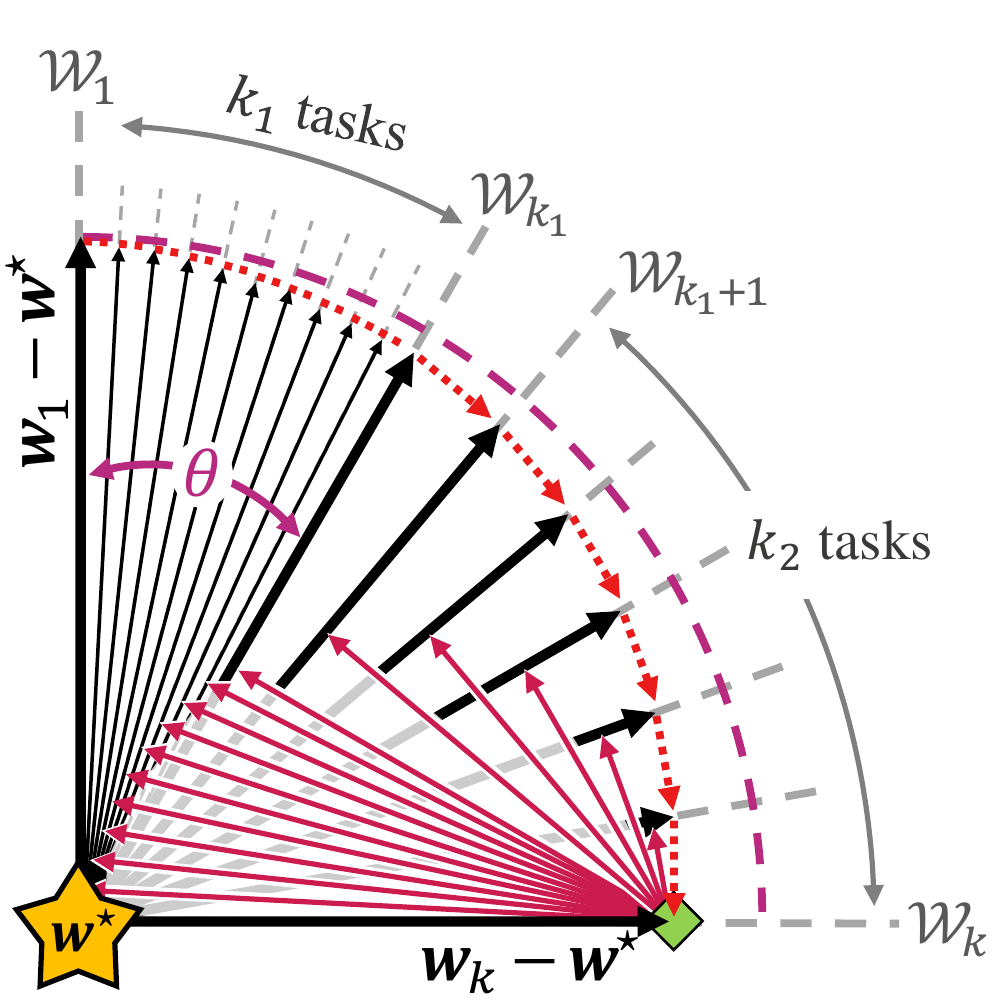}
\end{figure}

\begin{proof-sketch}
We show that for any $\epsilon>0$ 
there exists a task collection $\coll$ 
of $T=k=\mathcal{O}({1}/{\epsilon^2})$ tasks, 
such that under the identity ordering $\tau$,
the forgetting is $F_{\tau,\coll} (k) >1-\epsilon$.

We construct a task sequence as follows 
(illustrated in the figure in 2 dimensions): 
\begin{enumerate}
    \item For some small angle $\theta\sim\sqrt{\epsilon}$ we define $k_1\sim {1}/{\epsilon^2}$ tasks uniformly in $\left[0,\theta\right]$~;
    
    \item We add another $k_2\sim{1}/{\epsilon}$ tasks uniformly in $\left(\theta,\pi/2\right]$~.
\end{enumerate}
With this construction we show that on the one hand there is almost no contraction, but on the other hand there is a large forgetting especially on the first $k_1$ tasks, since they are almost orthogonal to the last task and the projection is large. 
\end{proof-sketch}

We note in passing 
a related phenomenon in quantum physics,
known as 
the \emph{quantum Zeno effect}
\citep{hacohen2018zeno}, 
where the state of a quantum system, described by a vector, can be manipulated by applying infinitesimally-spaced measurements, which act on it as orthogonal projections.

\bigskip

The full proof of \thmref{thm:worst_case} is given below.

\pagebreak

\begin{proof}[Proof of \thmref{thm:worst_case}] 
\label{prf:worst_case}

We show that for any $\epsilon\in\left(0,1\right)$
there exists a task collection $S$ of $k=\mathcal{O}\left({1}/{\epsilon^{2}}\right)$
tasks such that under the identity ordering,
$F_{\tau,\coll}\left(k\right)>1-\epsilon$.
%
From \lemref{lem:forgetting_lower_bound} we know
that for any choice of a task collection $\coll$ of rank $d-1$,
\begin{align}
F_{\tau,\coll}\left(k\right)
\!\ge\!
\frac{1}{k}\!
\left(
\prod_{m=1}^{k-1}
\cos^{2}\ang{m}{m+1}
\right)\!
\sum_{m=1}^{k-1}\!\left(1\!-\!\cos^{2}\ang{m}{k}\right) 
 \!=\!
 \frac{1}{k}\!
 \left(\prod_{m=1}^{k-1}\cos^{2}\theta_{m,m+1}\right)
 \!\!\sum_{m=1}^{k-1}\sin^{2}\theta_{m,k}.
\label{forgetting_k1k2}
\end{align}
We construct a sequence with $T=k=\nicefrac{72}{\epsilon^2}+1$ tasks as following: let $\theta=\theta\left(\epsilon\right)\in\left(0,\frac{\pi}{2}\right)$
and assume we have $k_{1}=k_{1}\left(\epsilon\right)$ tasks uniformly
in $\left[0,\theta\right]$ and $k_{2}=k_{2}\left(\epsilon\right)$
tasks uniformly in $\left(\theta,\frac{\pi}{2}\right]$. Note that we require
$k_{1}+k_{2}=k$.

Then from \eqref{forgetting_k1k2} we have
\begin{align*}
F_{\tau,\coll}\left(k\right) 
&\ge
\frac{1}{k_{1}+k_{2}}
\cos^{2(k_{1}-1)}\left(\frac{\theta}{k_{1}-1}\right)
\cos^{2k_{2}}\left(\frac{\frac{\pi}{2}-\theta}{k_{2}}\right)
\Bigg({
\sum_{i=1}^{k_{1}-1}
\sin^{2}\!\!\underbrace{\theta_{i,k}}_{\ge\nicefrac{\pi}{2}-\theta}
+
\underbrace{\sum_{i=k_{1}}^{k-1}\sin^{2}\theta_{i,k}}_{\ge 0}
}
\Bigg)
\\
 & 
 \ge
 \frac{1}{k_{1}+k_{2}}
 \cos^{2k_{1}-2}
 \prn{\frac{\theta}{k_{1}-1}}
 \cos^{2k_{2}}\left(\frac{\pi-2\theta}{2k_{2}}\right)
 \sum_{i=1}^{k_{1}-1}
 \sin^{2}\!\left(\frac{\pi}{2}-\theta\right)
 \\
 & 
 =\frac{1}{k_{1}+k_{2}}
 \cos^{2k_{1}-2}
 \prn{\frac{\theta}{k_{1}-1}}
 \cos^{2k_{2}}\left(\frac{\pi-2\theta}{2k_{2}}\right)
 \left(\prn{k_{1}-1}\sin^{2}\left(\frac{\pi}{2}-\theta\right)\right)
 \\
 & 
 =
 \underbrace{\cos^{2k_{1}-2}\left(\frac{\theta}{k_{1}-1}\right)}_{\triangleq F_{1}}~\cdot~
 \underbrace{\cos^{2k_{2}}\left(\frac{\pi-2\theta}{2k_{2}}\right)}_{\triangleq F_{2}}~\cdot~
 \underbrace{\frac{k_{1}-1}{k_{1}+k_{2}}\cos^{2}\theta}_{\triangleq F_{3}}~.
\end{align*}
We show that for any $\epsilon\in\left(0,1\right)$
we can choose $k_{1},k_{2},\theta$ such that
$$
F_{\tau,\coll}\left(k\right)\ge F_{1}\left(k_{1},\theta\right)\cdot F_{2}\left(k_{2},\theta\right)\cdot F_{3}\left(k_{1},k_{2},\theta\right)>1-\epsilon~
.$$
Specifically, let
$
k_{1} =\frac{72-12\epsilon}{\epsilon^{2}}+1,~~~~~~~
k_{2} = \frac{12}{\epsilon},~~~~~~~
\theta = \sqrt{\frac{\epsilon}{6}}
$.

Then, we 
use the inequality $\cos x\ge1-x^{2}/2$,
and get
\begin{align*}
F_{1}\!\left(k_{1},\theta\right) & 
\!=\!
\cos^{2(k_{1}-1)}\left(\frac{\theta}{k_{1}-1}\right)
\ge
\cos^{2(k_{1}-1)}\left(\frac{1}{k_{1}-1}\right)
\!\ge\!
 \left(1-\frac{1}{2(k_{1}-1)^{2}}\right)^{2k_{1}-2}
 \\
 &
 \!\ge\!
 1\!-\!\frac{1}{k_{1}-1}
 =
 1-\frac{\epsilon^{2}}{72-12\epsilon}\ge1-\frac{\epsilon}{3}
 \\
F_{2}\!\left(k_{2},\theta\right) & \!=\!\cos^{2k_{2}}\left(\frac{\pi-2\theta}{2k_{2}}\right)
\!\ge\!
\cos^{2k_{2}}\left(\frac{2}{k_{2}}\right)
\ge
\left(1-\frac{2}{k_{2}^{2}}\right)^{2k_{2}}\!\ge\!
1-\frac{4}{k_{2}}=1-\frac{\epsilon}{3}
\\
F_{3}\!\left(k_{1},k_{2},\theta\right) & 
\!=\!
\frac{k_{1}-1}{k_{1}+k_{2}}\!
\cos^{2}\theta
\ge
\frac{\frac{72-12\epsilon}{\epsilon^{2}}}{\frac{72-12\epsilon}{\epsilon^{2}}+\frac{12}{\epsilon}+1}
\left(1\!-\!\frac{\epsilon}{6}\right)
\!=\!
\frac{72-12\epsilon}{72+\epsilon^2}
\left(1\!-\!\frac{\epsilon}{6}\right)
\ge
1\!-\!\frac{\epsilon}{3}~.
\end{align*}
Therefore, we get that
$
F_{\tau,S}\left(k\right)\ge F_{1}\left(k_{1},\theta\right)\cdot F_{2}\left(k_{2},\theta\right)\cdot F_{3}\left(k_{1},k_{2},\theta\right)\ge\left(1-\frac{\epsilon}{3}\right)^{3}\ge1-\epsilon$~.

 Finally note that $k=k_{1}+k_{2}=\nicefrac{72}{\epsilon^{2}}+1$, which
concludes the proof.
\end{proof}

\newpage

\section{Supplementary material: Cyclic task orderings (\secref{sec:cyclic})}
\label{app:proofs2}

\subsection{Forgetting with $T=2$ tasks (\secref{sec:two_tasks})}
\label{app:two_tasks}

\begin{recall}[\lemref{lem:2_tasks_forgetting}]
For any task collection $\coll\in\mathcal{S}_{T=2}$ of two tasks,
the forgetting after $k=2n$ iterations (\ie $n$ cycles) is tightly upper bounded by
\begin{align*}
\begin{split}
F_{\tau, \coll}\left(k\right)
\le~
&
\frac{1}{2}
\max_{i}
\left\{
\left(\cos^2 \theta_i\right)^{k-1}
\left(1-\cos^2 \theta_i\right)
\right\}~,
\end{split}
\end{align*}
where 
$\{\theta_i\}_i \subseteq
\left(0, \tfrac{\pi}{2}\right]$
are the non-zero principal angles between the two tasks in $\coll$.
Moreover, the above inequality saturates when all non-zero singular values of the first task
(\ie of $\X_{1}$)
are $1$s. 
\end{recall}

\bigskip

\begin{proof}[Proof for \lemref{lem:2_tasks_forgetting}] 
\label{prf:worst_general_T_2}
We notice that at the end of each cycle we perfectly fit the second task,
thus having forgetting only on the first one.
In the cyclic case, from \eqref{eq:cyclic-forgetting} we have
\begin{align*}
\begin{split}
    F_{\tau,S}\left(k\right)
    &\le
    \frac{1}{2}
	\norm{\left(\I-\mP_{1}\right)
	\left(\mP_{2}\mP_{1}\right)^{n}}^{2}~.
\end{split}
\end{align*}
We now apply \lemref{lem:principal_form} (recall that $k=2n$) and conclude that:
\begin{align*}
\begin{split}
    F_{\tau,S}\left(k\right)
    &\le
    \frac{1}{2}
	\norm{\left(\I-\mP_{1}\right)
	\left(\mP_{2}\mP_{1}\right)^{n}}^{2}
	=
\frac{1}{2}
\max_{i}
\left\{
\left(\cos^2 \theta_i\right)^{k-1}(1-\cos^2 \theta_i)
\right\}~,
\end{split}
\end{align*}
where 
$\{\theta_i\}_i \subseteq
\left(0, \tfrac{\pi}{2}\right]$
are the non-zero principal angles between the two tasks.
Finally, we show that when all non-zero singular values of $\X_1$ are $1$s it holds that 
$\mSigma_1=\mSigma_1^+
\mSigma_1$
and we get 
\begin{align*}
F_{\tau,\coll}(k)
&=
\frac{1}{2}
\tsum_{m=1}^{T}
\bignorm{
\X_m
\bigprn{\mP_2 \mP_1}^{n}
\teacher
}^2
=
\frac{1}{2}
\biggprn{
\bignorm{
\X_1
\bigprn{\mP_2 \mP_1}^{n}
\teacher
}^2
+
\underbrace{\bignorm{
\X_2
\bigprn{\mP_2 \mP_1}^{n}
\teacher
}^2
}_{=0}
}
\\
&
=
\frac{1}{2}
\bignorm{
\U_1 \mSigma_1 \V_1^\top
\bigprn{\mP_2 \mP_1}^{n}
\teacher
}^2
\\
\explain{\text{\propref{prop:norms}}}
&
=
\frac{1}{2}
\bignorm{
\mSigma_1 \V_1^\top
\bigprn{\mP_2 \mP_1}^{n}
\teacher
}^2
=
\frac{1}{2}
\bignorm{
\mSigma_1^+
\mSigma_1 \V_1^\top
\bigprn{\mP_2 \mP_1}^{n}
\teacher
}^2
\\
&
=
\frac{1}{2}
\bignorm{
(\I-\mP_1)
\bigprn{\mP_2 \mP_1}^{n}
\teacher
}^2~,
\end{align*}
proving the inequality saturates in this case.
\end{proof}

\hspace{1.5pt}
\newpage

\begin{recall}[\thmref{thm:2_tasks_worst_forgetting}]
For~a~cyclic ordering $\tau$ of $2$ tasks, 
the worst-case forgetting after $k=2n\ge2$ iterations
(\ie $n\ge 1$ cycles),
is
\begin{align*}
\sup_{
    \substack{
        \coll\in\mathcal{S}_{T=2}
    }
    }
    \!\!
    F_{\tau,\coll}
    \left(k\right)
=
\frac{1}{2e\left(k-1\right)} -
\frac{1}{4e\left(k-1\right)^2} +
\bigO\left(\frac{1}{k^3}\right)~.
\end{align*}
\end{recall}
\begin{proof}
Following our previous lemma, 
the key to deriving the worst-case bound is to find the maximum of 
$$\frac{1}{2}
\left(\cos^2 \theta\right)^x \left(1-\cos^2 \theta\right)~,~~~\forall x\in\mathbb{N}\cup \left\{0\right\}~.$$
For $x=0$, the expression above is maximized by $\theta=0$ and equals $\nicefrac{1}{2}$.
Generally, one could show that for any integer $x\in\naturals$, the angle that maximizes the expression holds $\sin^2 \theta = \frac{1}{x+1}$.
Plugging that solution into the expression, we get:
$$
\frac{1}{2}
\max_{\theta\in\left[0,\pi\right]}
\left(\cos^2 \theta\right)^x \left(1-\cos^2 \theta\right)
=
\frac{1}{2}
\left(1-\frac{1}{x+1}\right)^x 
\frac{1}{x+1}
=
\frac{1}{2ex} - \frac{1}{4ex^2} +
\bigO\left(\frac{1}{x^3}\right)~.$$
Using \lemref{lem:2_tasks_forgetting},
we conclude that 
\begin{align*}
\sup_{\coll \in \mathcal{S}_{T=2}}F_{\tau,\coll}\left(k\right)
=
\frac{1}{2e\left(k-1\right)} -
\frac{1}{4e\left(k-1\right)^2} +
\bigO\left(\frac{1}{k^3}\right)~.
\end{align*}
\end{proof}

\newpage

\subsection{Forgetting with $T\ge3$ tasks (\secref{sec:forgetting_many_tasks})}
\label{app:proofs_many_tasks}

\begin{recall}[\thmref{thm:worst_case_T3}]
For~any number of tasks $T\ge 3$ under a cyclic ordering $\tau$, the worst-case forgetting
after $k=nT\ge T^2$ iterations (\ie $n\ge T$ cycles), is
\begin{align*}
    \frac{T^2}{24ek}
    ~~\le
    \!\!\!\!\!\!\!\!\!\!
    \sup_{
    \substack{
        \coll\in\mathcal{\coll}_{T\ge 3}: \\
        \forall m:\,\rank(\X_m) \le \maxrank
    }
    }
    \!\!\!\!\!\!\!\!\!\!\!\!\!\!\!\!
    F_{\tau,\coll}(k)
    ~~\le~~
    {
        \min\left\{
        \frac{T^2}{\sqrt{k}}
        ,~~
        \frac{T^2\prn{d-\maxrank}}{2k}
        \right\} 
    }~.
\end{align*}
%


Moreover, 
if the cyclic operator $\bigprn{\mP_T \cdots \mP_1}$ from \eqref{eq:cyclic-forgetting} is symmetric (\eg in a \emph{back-and-forth} setting  where tasks $m$ and $(T\!-\!m)$ are identical $\forall m\!\in\!\cnt{T}$),
then the worst-case forgetting is \emph{sharply}
${
\sup_{
\substack{
    \coll\in\mathcal{\coll}_{T\ge 3}
}
}
F_{\tau,\coll}(k)
=
\Theta\left(\nicefrac{T^2}{k}\right)
}
$.
\end{recall}

\subsubsection{Proving the upper bound}

We prove the upper bound using the two following lemmas. The proofs of the lemmas are given on the following pages.

Generally, the proofs revolve around the quantity $\left\Vert \M^{n}\vu\right\Vert _{2}^{2}
-
\left\Vert \M^{n+1}\vu\right\Vert _{2}^{2}$, 
that we bound both with and without using the dimension of the tasks.

\begin{lemma}[{Dimension-independent upper bounds}]
\label{lem:dimension_independent}
Let $\mP_{1},\dots,\mP_{T}\in \reals^{d\times d}$
be $T$ orthogonal projection operators forming a cyclic operator $\M=\mP_{T}\cdots\mP_{1}\in \reals^{d\times d}$.
Then:
\begin{enumerate}[label=(\textbf{Lemma~\ref*{lem:dimension_independent}\alph*)},leftmargin=2.5cm]\itemsep6pt
    \item 
    \label{lem:bound_on_projection_to_m_task}
    For any $\vv\in \complex^{d}$, $m\in\cnt{T-1}$, it holds that
    $$\left\Vert \left(\I-\mP_{m}\right)
    \vv\right\Vert _{2}^{2}
        \le 
        m
        \left(\left\Vert \vv\right\Vert _{2}^{2}-
        \left\Vert \M\vv\right\Vert _{2}^{2}\right)~;$$
    \item 
    \label{lem:bound_of_one_cycle}
    For any $\vv\in \complex^{d}$, $m\in\cnt{T-1}$, it holds that
    $\left\Vert \left(\I\!-\!\M\right)\vv\right\Vert _{2}^{2}
        \le 
        T
        \!
        \left(\left\Vert \vv\right\Vert _{2}^{2}-\left\Vert \M\vv\right\Vert _{2}^{2}\right)~\!;$
    \item 
    \label{lem:first_task_forgetting}
    For any vector $\vu\in \complex^{d}$ holding
    $\norm{\vu}\le 1$,
    after $n\ge 1$ cycles, it holds that
    \begin{align*}
        \left\Vert \M^{n}\vu\right\Vert _{2}^{2}
        -
        \left\Vert \M^{n+1}\vu\right\Vert _{2}^{2}
            \le 
        2 \norm{\M^{n} - \M^{n+1}}~;
    \end{align*}
    \item \label{lem:m_task_forgetting}
    The forgetting on task $m\in\cnt{T-1}$ after $n\ge 1$ cycles is upper bounded by
    \begin{align*}
        \norm{\left(\I-\mP_{m}\right)\M^{n}}^2 
        \le 
        2m 
        \norm{\M^{n} - \M^{n+1}}
        ~;
    \end{align*}
    \item 
    \label{lem:T_over_n_norm}
    For any number of cycles $n\ge 1$, it holds that
    $\norm{\M^{n-1} - \M^{n}} \le \sqrt{\nicefrac{T}{n}}$~;
    
    Moreover, when $\M$ is \emph{symmetric},
    we have
    $\norm{\M^{n-1} - \M^{n}} 
    \le \nicefrac{1}{e(n-1)}$~.
\end{enumerate}
\end{lemma}

\bigskip

\begin{lemma}[Rank-dependent upper bound]
\label{lem:rank_dependent_bound}
For any vector $\vu\in \complex^{d}$ holding
$\norm{\vu}\le 1$,
and for any non-expansive operator $\M\in\mathbb{R}^{d\times d}$,
\ie~$\left\Vert \M\right\Vert _{2}\le1$, 
it holds that 
\begin{align*}
\left\Vert \M^{n}\vu\right\Vert _{2}^{2}
-
\left\Vert \M^{n+1}\vu\right\Vert _{2}^{2}
~\le~
\frac{\rank\left(\M\right)}{n}
~\le~
\frac{d}{n}
~.
\end{align*}
\end{lemma}

\newpage

\begin{proof}[Proof for \thmref{thm:worst_case_T3} -- upper bound]

First, we remind the reader that the forgetting after $n$ cycles is upper bounded by
\begin{align}
    \frac{1}{T}
    \left\Vert 
    \left[\begin{array}{c}
    I-\mP_{1}
    \\
    \vdots\\
    I-\mP_{T-1}
    \end{array}\right]
    \left(\mP_{T}\cdots\mP_{1}\right)^{n}\right\Vert _{2}^{2}	
    &\le
    \frac{1}{T}\sum_{m=1}^{T-1}\left\Vert \left(\I-\mP_{m}\right)
    \left(\mP_{T}\cdots\mP_{1}\right)^{n}
    \right\Vert _{2}^{2}~.
\end{align}

The dimension-independent bound directly follows 
from \ref{lem:m_task_forgetting} and
\ref{lem:T_over_n_norm}:
\begin{align*}
\begin{split}
    \frac{1}{T}\sum_{m=1}^{T-1}\left\Vert \left(\I-\mP_{m}\right)\left(\mP_{T}\cdots\mP_{1}\right)^{n}\right\Vert _{2}^{2}
    &
    \le
    \underbrace{\frac{1}{T}
    \sum_{m=1}^{T-1}2m}_{=\prn{T-1}}
        \sqrt{\frac{T}{n}}
    	=
    	\prn{T-1}\sqrt{\frac{T}{n}}
    	\le
    	T\sqrt{\frac{T}{n}}
    	=
    	\frac{T^2}{\sqrt{k}}
    	~.
\end{split}
\end{align*}

\bigskip
\paragraph{A symmetric cyclic operator 
$\M=\left(\mP_{T}\cdots\mP_{1}\right)$.}
Consider a symmetric cyclic operator
appearing for instance in back-and-forth settings where
$\forall m\in\cnt{T}$ we have
$\sol{m}=\sol{T-m}$.
Then, \ref{lem:T_over_n_norm} gives a tighter bound, which in turn (assuming $n\ge2$) yields an overall bound of
$$\frac{1}{T}\sum_{m=1}^{T-1}\left\Vert \left(\I-\mP_{m}\right)\left(\mP_{T}\cdots\mP_{1}\right)^{n}\right\Vert _{2}^{2}
\le
\frac{T-1}{e(n-1)}
\le
\frac{T}{n}
=
\frac{T^2}{k}~.
$$

Finally, since we prove the theorem's
$\nicefrac{T^2}{24ek}$
lower bound in \appref{app:lower_bound_proof} using a back-and-forth task collection,
\ie using a symmetric cyclic operator,
we get that in these back-and-forth settings we have a sharp
worst-case behavior of $\Theta(\nicefrac{T^2}{k})$

\bigskip
\bigskip
\bigskip

Finally, we notice that \ref{lem:bound_on_projection_to_m_task} implies that
\begin{align*}
    \left\Vert \left(\I-\mP_{m}\right)
    \!
    \M^{n}
    \right\Vert _{2}^{2}
    &\triangleq\!\!\!
    \max_{\vu:\left\Vert \vu\right\Vert _{2}=1}
    \left\Vert \left(\I-\mP_{m}\right)
    \!
    \M^{n}\vu
    \right\Vert_{2}^{2}
    \le
    m\cdot\!\!\!\!\!\max_{\vu:\left\Vert \vu\right\Vert _{2}=1}\!
    \left(
        \norm{\M^{n}\vu}^2
        -
        \norm{\M^{n+1}\vu}^2
    \right)\!,
\end{align*}
and by applying \lemref{lem:rank_dependent_bound} on $\M\triangleq\mP_{T}\cdots\mP_{1}$, 
we conclude the dimension-dependent bound:
\begin{align*}
\begin{split}
    \frac{1}{T}\sum_{m=1}^{T-1}\left\Vert \left(\I-\mP_{m}\right)\left(\mP_{T}\cdots\mP_{1}\right)^{n}\right\Vert _{2}^{2}
    &\le
    \underbrace{\frac{1}{T}
    \sum_{m=1}^{T-1}m}_{=\prn{T-1}/{2}}
    \frac{\rank\left(\M\right)}{n}
    \le
    \frac{T\left(d-\maxrank\right)}{2n}
    =
    \frac{T^2\left(d-\maxrank\right)}{2k}~,
\end{split}
\end{align*}
where we used our notation of
$\maxrank\triangleq 
\max_{m\in\cnt{T}} \rank \X_m
=
\max_{m\in\cnt{T}} \left(d-\rank \mP_m\right)$
and the fact that
$\forall m\in\cnt{T}: 
\rank\left(\M\right)
=
\rank\left(\mP_{T}\cdots\mP_{1}\right)
\le 
\min_{m\in\cnt{T}} \rank \mP_m$.
\end{proof}

\newpage

Before we prove the lemmas above,
we state an auxiliary claim (and a corollary).

\begin{claim}
\label{clm:identity_trick}
For any $m\in\cnt{T}$, it holds that
\begin{align*}
\begin{split}
\I
&=
\left(
\I
-
\mP_{1}
\right)
+
\sum_{\ell=1}^{m-1}
\left(
\I
-
\mP_{\ell+1}
\right)
\left(
\mP_{\ell}
\cdots
\mP_{1}
\right)
+
\mP_{m}
\cdots
\mP_{1}~.
\end{split}
\end{align*}
\end{claim}
\begin{proof}
Recursively, we show that
\begin{align*}
\begin{split}
&
{\I
    -
    \mP_{m}
    \cdots
    \mP_{1}
}
\\
&=
{
    \left(
    \I
    -
    \mP_{1}
    \right)
    +
    \left(
    \mP_{1}
    -
    \mP_{2}
    \mP_{1}
    \right)
    +
    \left(
    \mP_{2}
    \mP_{1}
    -
    \mP_{3}
    \mP_{2}
    \mP_{1}
    \right)
    +
    \dots
    +
    \left(\mP_{m-1}
    \cdots
    \mP_{1}
    -
    \mP_{m}
    \cdots
    \mP_{1}
    \right)
}
\\
&=
{
    \left(
    \I
    -
    \mP_{1}
    \right)
    +
    \left(\I-\mP_{2}\right)
    \mP_{1}
    +
    \left(\I-\mP_{3}\right)
    \left(
    \mP_{2}
    \mP_{1}
    \right)
    +
    \dots
    +
    \left(\I-\mP_{m}\right)
    \left(
    \mP_{m-1}
    \cdots
    \mP_{1}
    \right)
}
\\
&=
\left(
\I
-
\mP_{1}
\right)
+
\sum_{\ell=1}^{m-1}
\left(
\I
-
\mP_{\ell+1}
\right)
\left(
\mP_{\ell}
\cdots
\mP_{1}
\right)~,
\end{split}
\end{align*}
which proves our claim.
\end{proof}

\bigskip
\bigskip

\begin{corollary}
\label{cor:vec_property}
For any $m\in\cnt{T}$, it holds that,
\begin{align*}
    \left(\I-\mP_{m}\right)
    =
    \left(\I-\mP_{m}\right)
    \left(\I-\mP_{1}\right)
    +
    \sum_{\ell=1}^{m-1}
    \left(\I-\mP_{m}\right)
    \left(\I-\mP_{\ell+1}\right)\left(\mP_{\ell}\cdots\mP_{1}\right)~.
\end{align*}
\end{corollary}
\begin{proof}
We use \clmref{clm:identity_trick} and get that
\begin{align*}
    &\left(\I-\mP_{m}\right)
    =
    \left(\I-\mP_{m}\right)
    \I
    \\
    &=
    \left(\I-\mP_{m}\right)
    \left[
    \left(\I-\mP_{1}\right)
    +
    \sum_{\ell=1}^{m-1}
    \left(\I-\mP_{\ell+1}\right)\left(\mP_{\ell}\cdots\mP_{1}\right)
    +
    \mP_{m}
    \cdots
    \mP_{1}
    \right]
    \\
    &=
    \left(\I-\mP_{m}\right)
    \left[
    \left(\I-\mP_{1}\right)
    +
    \sum_{\ell=1}^{m-1}
    \left(\I-\mP_{\ell+1}\right)\left(\mP_{\ell}\cdots\mP_{1}\right)
    \right]
    +
    \underbrace{
        \left(\I-\mP_{m}\right)
        \mP_{m}
    }_{=\0,~\text{by \propref{prop:projections}}}
    \mP_{m-1}
    \cdots
    \mP_{1}
    ~.
\end{align*}
\end{proof}

\newpage

\paragraph{Proof for \lemref{lem:dimension_independent}}.
We now prove our dimension-independent upper bounds, step by step, by proving all the statements of \lemref{lem:dimension_independent}.

\bigskip

\begin{proof}[Proof for \ref{lem:bound_on_projection_to_m_task}] 
For $m=1$, one can show that:
\begin{align*}
    \left\Vert \left(\I-\mP_{1}\right)\vv\right\Vert _{2}^{2}
    \stackrel{\text{idempotence}}{=}
    \left\Vert \vv\right\Vert _{2}^{2}
    -
    \left\Vert \mP_{1}\vv\right\Vert _{2}^{2}
    \stackrel{\substack{\text{projections}\\ \text{contract}}}{\le}
    1\cdot
    \left(
    \left\Vert \vv\right\Vert _{2}^{2}
    -
    \left\Vert \left(\mP_{T}\cdots\mP_{1}\right)\vv\right\Vert _{2}^{2}
    \right)~.
\end{align*}

Then, for $m=2,3,\dots, T-1$, we apply \corref{cor:vec_property} and obtain
\begin{align*}
\begin{split}
    &\left\Vert \left(\I-\mP_{m}\right)\vv\right\Vert _{2}^{2}
    \\
    \explain{\text{\corref{cor:vec_property}}}
    &=
    \norm{
    \left(\I-\mP_{m}\right)
    \left(\I-\mP_{1}\right)\vv
    +
    \sum_{\ell=1}^{m-1}
    \left(\I-\mP_{m}\right)
    \left(\I-\mP_{\ell+1}\right)\left(\mP_{\ell}\cdots\mP_{1}\right)\vv
    }^2
    \\
    \explain{\text{\lemref{lem:square_ineq}}}
    &\le 
    m
    \left(
    \norm{
    \left(\I-\mP_{m}\right)
    \left(\I-\mP_{1}\right)\vv
    }^2
    +
    \sum_{\ell=1}^{m-1}
    \norm{
    \left(\I-\mP_{m}\right)
    \left(\I-\mP_{\ell+1}\right)\left(\mP_{\ell}\cdots\mP_{1}\right)\vv
    }^2
    \right)
    \\
    \explain{\text{projections}\\ \text{contract}}
    &\le 
    m
    \left(
    \norm{
    \left(\I-\mP_{1}\right)\vv
    }^2
    +
    \sum_{\ell=1}^{m-1}
    \norm{
    \left(\I-\mP_{\ell+1}\right)\left(\mP_{\ell}\cdots\mP_{1}\right)\vv
    }^2
    \right)
    \\
    \explain{\text{idempotence}}
    &= 
    m
    \left(
    \norm{\vv}^2
    -
    \norm{
    \mP_{1}\vv
    }^2
    +
    \sum_{\ell=1}^{m-1}
    \left(
    \norm{
    \left(\mP_{\ell}\cdots\mP_{1}\right)\vv
    }^2
    -
    \norm{
    \left(\mP_{\ell+1}\cdots\mP_{1}\right)\vv
    }^2
    \right)
    \right)
    \\
    \explain{\text{telescoping}}
    &=
    m
    \left(
    \norm{\vv}^2
    -
    \norm{
    \left(\mP_{m}\cdots\mP_{1}\right)\vv
    }^2
    \right)
    \le
    m
    \left(
    \norm{\vv}^2
    -
    \norm{
    \left(\mP_{T}\cdots\mP_{1}\right)\vv
    }^2
    \right)
    ~,
\end{split}
\end{align*}
which completes our proof.
\end{proof}
\paragraph{Remark.}
We note in passing that this dependence on $m$ can be further improved.
Using similar techniques,
one can also prove that
$\left\Vert 
\left(\I-\mP_{m}\right)\M\vv\right
\Vert_{2}^{2}
\le
\bigprn{T-m}
\left(
\norm{\vv}^2
-
\norm{
\M\vv
}^2
\right)$.
This in turn can help tighten the upper bound in
\thmref{thm:worst_case_T3}
by a multiplicative factor of $2$, but yields a slightly less elegant expression.

\newpage

\begin{proof}[Proof for \ref{lem:bound_of_one_cycle}] 
Notice that \clmref{clm:identity_trick} (with $m=T$) implies that
\begin{align*}
\begin{split}
\I -
\mP_{T}
\cdots
\mP_{1}
&=
\left(
\I
-
\mP_{1}
\right)
+
\sum_{\ell=1}^{T-1}
\left(
\I
-
\mP_{\ell+1}
\right)
\left(
\mP_{\ell}
\cdots
\mP_{1}
\right)~.
\end{split}
\end{align*}
Then, we prove our lemma:
\begin{align*}
\begin{split}
\norm{\vv
    -
    \mP_{T}
    \cdots
    \mP_{1}
    \vv
}^2
&=
\norm{
    \left(
    \I
    -
    \mP_{1}
    \right)\vv
+
\sum_{\ell=1}^{T-1}
    \left(
    \I
    -
    \mP_{\ell+1}
    \right)
    \left(
    \mP_{\ell}
    \cdots
    \mP_{1}
    \right)\vv
}^2
\\
\explain{\text{\lemref{lem:square_ineq}}}
&\le
T
\left(
\norm{
    \left(
    \I
    -
    \mP_{1}
    \right)\vv
}^2
+
\sum_{\ell=1}^{T-1}
\norm{
    \left(
    \I
    -
    \mP_{\ell+1}
    \right)
    \left(
    \mP_{\ell}
    \cdots
    \mP_{1}
    \right)\vv
}^2
\right)
\\
\explain{\text{idempotence}}
&=
T
\left(
\norm{\vv}^2
-
\norm{
    \mP_{1}\vv
}^2
+
\sum_{\ell=1}^{T-1}
\left(
\norm{
    \left(
    \mP_{\ell}
    \cdots
    \mP_{1}
    \right)\vv
}^2
-
\norm{
    \left(
    \mP_{\ell+1}
    \cdots
    \mP_{1}
    \right)\vv
}^2
\right)
\right)
\\
\explain{\text{telescoping}}&=
T
\left(
\norm{\vv}^2
-
\norm{
    \left(
    \mP_{T}
    \cdots
    \mP_{1}
    \right)\vv
}^2
\right)~.
\end{split}
\end{align*}
\end{proof}
\vspace{-13pt}
\paragraph{Remark.}
We note in passing that after completing our proof above, 
we found that a similar proof was already presented in \citet{netyanun2006iterated} while discussing Kakutani's lemma.
We still brought our proof here for the sake of completeness.
Moreover, they showed that for a non-expansive \emph{self-adjoint positive semi-definite} operator $\M$, the factor $T$ can be alleviated from the inequality. 
That is,
$
\norm{\vv
    -
    \mP_{T}
    \cdots
    \mP_{1}
    \vv
}^2
\le
\cancel{T}
\left(
\norm{\vv}^2
-
\norm{
    \left(
    \mP_{T}
    \cdots
    \mP_{1}
    \right)\vv
}^2
\right)
$.
But clearly this does not suit our general cyclic operators which are not necessarily self-adjoint.
Indeed, our proof for \ref{lem:T_over_n_norm} yields a similar conclusion for \emph{symmetric} cyclic operators.

\bigskip
\bigskip
\bigskip

\begin{proof}[Proof for \ref{lem:first_task_forgetting}] 
Let $\vu\in\complex^d$. 
%
First, we notice that since $\M$
is a non-expansive operator,
we have that 
$\left\Vert \M^{n}\vu\right\Vert_{2}^{2}
\ge
\left\Vert \M^{n+1}\vu\right\Vert_{2}^{2}$.
In turn, this means
\begin{align*}
0
\le
\left\Vert \M^{n}\vu\right\Vert_{2}^{2}
-
\left\Vert \M^{n+1}\vu\right\Vert_{2}^{2}
=
\vu^\hop
\Bigprn{
\underbrace{
\left(\M^{n}\right)^{\hop}
\M^{n}
-
\left(\M^{n+1}\right)^{\hop}
\M^{n+1}
}_{\succeq{\0}~\text{because it is clearly symmetric as well}}
}
\vu~.
%
%
\\
\end{align*}
Then, we use the fact that for any positive semi-definite matrix $\A\succeq\0$, 
we have that
$\vu^\hop \A \vu \le 
\lambda_1 \left(\A\right)
=\norm{\A}$,
and get
\begin{align*}
\left\Vert \M^{n}\vu\right\Vert_{2}^{2}
-
\left\Vert \M^{n+1}\vu\right\Vert_{2}^{2}
&\le
\norm{
\left(\M^{n}\right)^{\hop}
\M^{n}
-
\left(\M^{n+1}\right)^{\hop}
\M^{n+1}
}
\\
\explain{
\A^\hop \A -\B^\hop \B
=
\prn{\A^\hop -\B^\hop}\A
+
\B^\hop\prn{\A -\B}
}
&=
\norm{
\left(\M^{n}\right)^{\hop}
\prn{
\M^{n}
-
\M^{n+1}
}
+
\prn{
\left(\M^{n}\right)^{\hop}
-
\left(\M^{n+1}\right)^{\hop}
}
\M^{n}
}
\\
\explain{\text{triangle inequality}}
&\le
\norm{
\left(\M^{n}\right)^{\hop}
\prn{
\M^{n}
-
\M^{n+1}
}
}
+
\norm{
\prn{
\left(\M^{n}\right)^{\hop}
-
\left(\M^{n+1}\right)^{\hop}
}
\M^{n}
}
\\
\explain{~\norm{\A^\hop}=\norm{\A}~}
&\le
2\norm{
\left(\M^{n}\right)^{\hop}
\prn{
\M^{n}
-
\M^{n+1}
}
}
\\
\explain{\text{contraction}}
&\le
2\norm{
\M^{n}
-
\M^{n+1}
}~.
%
\end{align*}
\end{proof}

\newpage

\begin{proof}[Proof for \ref{lem:m_task_forgetting}] 
This lemma is a direct corollary of
\ref{lem:bound_on_projection_to_m_task}
and
\ref{lem:first_task_forgetting}.

First, we set
$\vv=\M^{n}\vu$ and apply
\ref{lem:bound_on_projection_to_m_task}:
\begin{align*}
\left\Vert 
\left(\I-\mP_{m}\right)\M^{n}
\right\Vert^2
=
\max_{\vu:\left\Vert \vu\right\Vert _{2}=1}
\left\Vert \left(\I-\mP_{m}\right)
\M^{n} \vu
\right\Vert^2
\le
m\cdot\!\!
\max_{\vu:\left\Vert \vu\right\Vert _{2}=1}
\prn{
\left\Vert 
\M^{n} \vu\right\Vert_{2}^{2}
-
\left\Vert \M^{n+1} \vu
\right\Vert_{2}^{2}
}~.
\end{align*}
Then, we apply \ref{lem:first_task_forgetting}
to get:
\begin{align*}
\left\Vert 
\left(\I-\mP_{m}\right)\M^{n}
\right\Vert^2
\le
m\cdot\!\!
\max_{\vu:\left\Vert \vu\right\Vert _{2}=1}
\prn{
\left\Vert 
\M^{n} \vu\right\Vert_{2}^{2}
-
\left\Vert \M^{n+1} \vu
\right\Vert_{2}^{2}
}
\le
2m \norm{\M^{n} - \M^{n+1}}~.
\end{align*}
\end{proof}

\bigskip
\bigskip

\begin{proof}[Proof for \ref{lem:T_over_n_norm}] 
We use \ref{lem:bound_of_one_cycle} 
to show that
\begin{align*}
\begin{split}
\sum_{t=0}^{n-1} \norm{\left(\M^t-\M^{t+1}\right)\vv}^2
=
\sum_{t=0}^{n-1} \norm{
\left(\I-\M\right)\M^t\vv}^2
&
\le
T\sum_{t=0}^{n-1}
\left(
\norm{\M^t\vv}^2
-
\norm{\M^{t+1}\vv}^2
\right)
\\
\explain{\text{telescoping}}
&
=
T
\left(
\norm{\vv}^2
-
\norm{\M^{n}\vv}^2
\right)
\le
T \norm{\vv}^2~.
\end{split}
\end{align*}
Since
$\norm{\left(\M^t-\M^{t+1}\right)\vv}^2
=
\norm{\M\left(\left(\M^{t-1}-\M^{t}\right)\vv\right)}^2$
and since the cyclic operator $\M$ is a contraction operator,
we get that the series $\left\{\norm{\left(\M^t-\M^{t+1}\right)\vv}^2\right\}_t$
is monotonic non-increasing.
This in turn means that
\begin{align*}
    \forall{\vv \in \complex^d}:
    \norm{\left(\M^{n-1}-\M^{n}\right)\vv}^2
    &=\!\!
    \min_{t=0,\dots,n-1} \norm{\left(\M^t-\M^{t+1}\right)\vv}^2
    \\&
    \le 
    \frac{1}{n}
    \sum_{t=0}^{n-1} \norm{\left(\M^t-\M^{t+1}\right)\vv}^2
    \le
    \frac{T}{n} \norm{\vv}^2~,
\end{align*}
which finally implies that
$\norm{\M^{n-1} - \M^{n}}^2 \le \nicefrac{T}{n}$
as required.

\bigskip
\bigskip

\paragraph{A symmetric cyclic operator 
$\M=\left(\mP_{T}\cdots\mP_{1}\right)$.}
When $\M$ is Hermitian matrix,
its spectral decomposition $\M=\Q\mLambda\Q^\hop$
reveals a tighter bound:
\begin{align*}
    \norm{\M^{n-1} \!-
    \!\M^{n}}
    =
    \norm{\Q(\mLambda^{n-1} 
    \!-\!
    \mLambda^{n})\Q^\hop}
    =
    \norm{\mLambda^{n-1} 
    \!-\! \mLambda^{n}}
    =
    \max_{i}
    \prn{\lambda_i^{n-1}
    \prn{1 \!-\! \lambda_i}}
    \le
    \frac{1}{e(n-1)}\,.
\end{align*}
\end{proof}

\newpage

\begin{proof}[Proof for \lemref{lem:rank_dependent_bound}] 
We follow a proof by
\citet{mathoverflowbound}.

Define the series of matrices $\left\{\boldsymbol{B}_n\right\}$,
where $\boldsymbol{B}_{n}=\M^{n^{\top}}\M^{n}-\M^{n+1^{\top}}\M^{n+1}$.
Notice that it holds that $\boldsymbol{B}_{n}\succeq0$ (since it is symmetric and $\forall\vv:\vv^{*}\boldsymbol{B}_{n}\vv=\left\Vert \M^{n}\vv\right\Vert _{2}^{2}-\left\Vert \M^{n+1}\vv\right\Vert _{2}^{2}\ge0$).


Furthermore, 
Von Neumann's trace inequality
and the fact that $\M$ is a contraction operator, imply that
\begin{align*}
\tr\left(\boldsymbol{B}_{n+1}\right)
=
\tr
\bigprn{
\M^{\top}\boldsymbol{B}_{n}\M
}
=
\tr
\bigprn{
\underbrace{\M\M^{\top}}_{\succeq 0}
\underbrace{\boldsymbol{B}_{n}}_{\succeq 0}
}
\le
\bignorm{\M\M^{\top}}_2
\tr\left(\boldsymbol{B}_{n}\right)
\le
\tr\left(\boldsymbol{B}_{n}\right)~,
\end{align*}
meaning that the sequence 
$\left\{\tr\left(\boldsymbol{B}_{n}\right)\right\}_n$
is monotonically non-increasing. 
Hence,  
we get that
\begin{align*}
\begin{split}
    n\tr\left(\boldsymbol{B}_{n}\right)	
    &
    \le\sum_{\ell=1}^{n}\tr\left(\boldsymbol{B}_{\ell}\right)
    =\sum_{\ell=1}^{n}\tr\left(\M^{\ell^{\top}}\M^{\ell}-\M^{\ell+1^{\top}}\M^{\ell+1}\right)
    \\
    &=\sum_{\ell=1}^{n}\left(\tr\left(\M^{\ell^{\top}}\M^{\ell}\right)-\tr\left(\M^{\ell+1^{\top}}\M^{\ell+1}\right)\right)
	\\
    \explain{\text{telescoping}}
	&
	=\tr\left(\M^{^{\top}}\M\right)-\underbrace{\tr\left(\M^{n+1^{\top}}\M^{n+1}\right)}_{\ge0}
	\le\tr\left(\M^{^{\top}}\M\right)
	\stackrel{\text{contraction}}{\le}\rank{\left(\M\right)}
	\\
    \Longrightarrow 
    \tr\left(\boldsymbol{B}_{n}\right)
    &
    \le\frac{\rank{\left(\M\right)}}{n}~.
    \end{split}
\end{align*}
We are now ready to conclude this lemma,
\begin{align*}
    \begin{split}
	\max_{\vv:\left\Vert \vv\right\Vert _{2}=1}\left(\left\Vert \M^{n}\vv\right\Vert _{2}^{2}-\left\Vert \M^{n+1}\vv\right\Vert _{2}^{2}\right)
	&=
	\max_{\vv:\left\Vert \vv\right\Vert _{2}=1}
	\vv^{*}\left(\M^{n^{\top}}\M^{n}-\M^{n+1^{\top}}\M^{n+1}\right)\vv
	\\
	&=\max_{\vv:\left\Vert \vv\right\Vert _{2}=1}\vv^{*}\underbrace{\boldsymbol{B}_{n}}_{\succeq0}\vv
	=\underbrace{\lambda_{1}\left(\boldsymbol{B}_{n}\right)}_{\in\mathbb{R}_{+}}
	\le\tr\left(\boldsymbol{B}_{n}\right)\le\frac{\rank{\left(\M\right)}}{n}~.
    \end{split}
\end{align*}
\end{proof}
\paragraph{Remarks.}
For the sake of completeness, we briefly discuss our result above.
First, this result is useful for our derivations since our cyclic operator $\left(\mP_{T}\cdots\mP_{1}\right)$ is essentially a contraction operator.
In fact, being a product of projections,
our cyclic operator has great expressiveness.
It was shown by \citet{oikhberg1999products} that \emph{any} contraction operator can be decomposed into a product of sufficiently,
though sometimes infinitely, many projections.
This explains our interest in analyzing ``general'' contraction operators.

Finally, we should mention that the bound derived here is sharp (up to a constant).
For instance, for the Toeplitz operator $\boldsymbol{M}=
\left[
\begin{smallmatrix}
 & 1\\
 &  & \ddots\\
 &  &  & 1\\
1-\epsilon
\end{smallmatrix}
\right]$, 
a proper choice of $\epsilon$ can yield a rate of
$\nicefrac{d}{en}$ for the quantity we bound in the lemma above.
However, such an operator cannot be expressed as a product of a \emph{finite} number of projection operators. 
One can approximate this operator 
arbitrarily-well
by replacing the $1$s with $\beta\in\prn{0,1}$.
However, we were not able to use this approximation to improve our lower bound
in \thmref{thm:worst_case_T3}, since 
$\beta<1$ implies a geometric contraction of
\emph{all} elements at every cycle and requires a very large number of constructing projections (\ie tasks).


\newpage

\subsubsection{Proving the lower bound}
\label{app:lower_bound_proof}
\begin{proof}[Proof for \thmref{thm:worst_case_T3} -- lower bound] 
\label{prf:worst_case_T3}

To prove the lower bound, we show a construction of a task collection $\coll\in\Coll$ with $T$ tasks whose forgetting is
\begin{align*}
    \frac{T^2}{24ek}
    \le
    F_{\tau,\coll}(k)~.
\end{align*}

\begin{figure}[b!]
   \centering
  %
  %
    \caption{Construction of $T=6$ "back-and-forth" tasks for the lower bound proof in  \thmref{thm:worst_case_T3}.}
  \includegraphics[width=.5\textwidth]{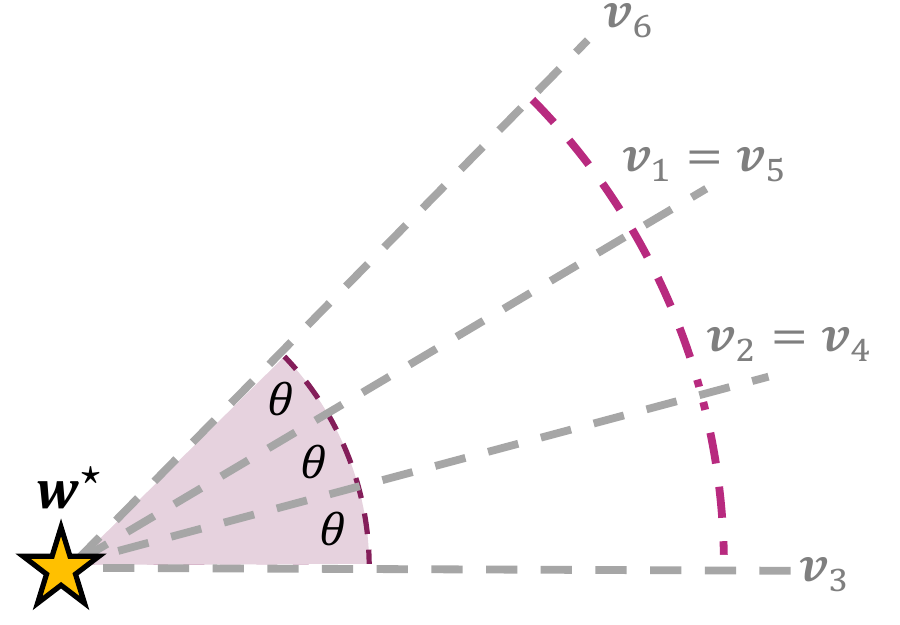}
    %
    \label{fig:lower_bound}
\end{figure}


We construct $\coll$ using data matrices of rank $r=d-1$ whose non-zero singular values all equal to $1$.
Let $\vv_1,\dots,\vv_T$
be the normalized vectors spanning the $T$ rank-\emph{one} solution spaces $\sol{1},\dots,\sol{T}$.
We also choose to generate labels using a \emph{unit norm} offline solution, \ie $\norm{\w^\star}=1$.
We spread $\vv_1,\dots,\vv_T$
on a 2-dimensional hyperplane
such that the tasks "go" back-and-forth, \ie switch direction at the middle task (see illustration in \figref{fig:lower_bound}).

\bigskip\bigskip

Using \lemref{lem:forgetting_lower_bound} we have
%
%
\begin{align*}
&F_{\tau,\coll}\left(k\right) 
\geq
\frac{1}{k}\left(\prod_{i=1}^{k-1}\cos^{2}\theta_{\tau\left(i,i+1\right)}\right)\sum_{m=1}^{k-1}\left(1-\cos^{2}\theta_{\tau\left(m,k\right)}\right)
 \\&=
 \frac{1}{k}\!
 \left(\prod_{i=1}^{k-1}\cos^{2}\theta_{\tau\left(i,i+1\right)}\right)\frac{k}{T}\sum_{m=1}^{T}\left(1\!-\!\cos^{2}\theta_{m,T}\right)
 \!=\!
 \frac{1}{T}\!
 \left(\prod_{i=1}^{k-1}\cos^{2}\theta_{\tau\left(i,i+1\right)}\!\right)^{\!k-1}
 \!\!\!
 \sum_{m=1}^{T}\!\left(1\!-\!\cos^{2}\theta_{m,T}\right).
\end{align*}
We consider separately the case of even $T$ and the case of odd $T$.

\newpage

\paragraph{Even $T$: }
In this case, we set 
$
\forall i\in\left[T-1\right]:\,\,\theta_{i,i+1}=\begin{cases}
\theta & i<\frac{T}{2}\\
-\theta & i\ge\frac{T}{2}
\end{cases}
~$,
thus obtaining

\[
\frac{1}{T}\!
\left(\prod_{i=1}^{k-1}\cos^{2}\theta_{\tau\left(i,i+1\right)}\right)^{k-1}
\!\!\!
\sum_{m=1}^{T}\left(1\!-\!\cos^{2}\theta_{m,T}\right)
=
\frac{1}{T}\!\left(\cos^{2}\theta\right)^{k-1}
\sum_{m=1}^{T}
\!\left(1\!-\!
\cos^{2}\left(\left(\frac{T}{2}\!-\!m\right)\theta\right)\right)~.
\]
We choose $\theta=\sqrt{\nicefrac{1}{k-1}}$ and get
\begin{align*}
&\frac{1}{T}\left(\cos^{2}\theta\right)^{k-1}\sum_{m=1}^{T}\left(1-\cos^{2}\left(\left(\frac{T}{2}-m\right)\theta\right)\right)
\\
&
=\frac{1}{T}\left(\cos^{2}\sqrt{\frac{1}{k-1}}\right)^{k-1}\sum_{m=1}^{T}\left(1-\cos^{2}\left(\left(\frac{T}{2}-m\right)\sqrt{\frac{1}{k-1}}\right)\right).
\end{align*}
Next, using the inequality $1-\cos^{2}x\ge x^{2}-\frac{1}{3}x^{4}$, we get
\begin{align*}
&\frac{1}{T}\left(\cos^{2}\theta\right)^{k-1}\sum_{m=1}^{T}\left(1-\cos^{2}\left(\left(\frac{T}{2}-m\right)\theta\right)\right)
\\
&\ge\frac{1}{T}\left(\cos^{2}\sqrt{\frac{1}{k-1}}\right)^{k-1}\sum_{m=1}^{T}\left(\left(\frac{T}{2}-m\right)^{2}\frac{1}{k-1}-\frac{1}{3}\left(\frac{T}{2}-m\right)^{4}\frac{1}{\left(k-1\right)^{2}}\right)
\\
&\ge\frac{1}{T}\frac{1}{k-1}\left(\cos^{2}\sqrt{\frac{1}{k-1}}\right)^{k-1}\sum_{m=1}^{T}\left(\left(\frac{T}{2}-m\right)^{2}-\left(\frac{T}{2}-m\right)^{4}\frac{1}{3\left(k-1\right)}\right)~.
\end{align*}
By using $\forall x\ge1:\frac{1}{x}
\left(
\cos^{2}\sqrt{\nicefrac{1}{x}}\right)^{x}\ge\nicefrac{1}{e\left(x+1\right)}$
we get
\begin{align*}
&\frac{1}{T}\left(\cos^{2}\theta\right)^{k-1}\sum_{m=1}^{T}\left(1-\cos^{2}\left(\left(\frac{T}{2}-m\right)\theta\right)\right)
\\
&\ge\frac{1}{Tek}\sum_{m=1}^{T}\left(\left(\frac{T}{2}-m\right)^{2}-\left(\frac{T}{2}-m\right)^{4}\frac{1}{3\left(k-1\right)}\right)
\\
&=\frac{1}{Tek}\left(\frac{1}{12}T\left(T^{2}+2\right)-
\frac{T\left(3T^{4}+20T^{2}-8\right)}{240\cdot 3(k-1)}\right)
=\frac{1}{12ek}\left(T^{2}+2-\frac{3T^{4}+20T^{2}-8}{60\left(k-1\right)}\right)~.
\end{align*}
For $k\ge T^{2}$ we get
\begin{equation}
\frac{1}{T}\left(\cos^{2}\theta\right)^{k-1}\sum_{m=1}^{T}\left(1-\cos^{2}\left(\left(\frac{T}{2}-m\right)\theta\right)\right)\ge\frac{1}{12ek}\left(T^{2}+2-\frac{3T^{4}+20T^{2}-8}{60\left(T^{2}-1\right)}\right)~.\label{eq:evenT}
\end{equation}
Note that for $T\ge4$ we have
$
\frac{1}{12ek}\left(T^{2}+2-\frac{3T^{4}+20T^{2}-8}{60\left(T^{2}-1\right)}\right)\ge\frac{T^{2}}{13ek}$.

Therefore for even $T\ge4$ we have
\[
F_{\tau,\coll}\left(k\right)\ge\frac{1}{T}\left(\cos^{2}\theta\right)^{k-1}\sum_{m=1}^{T}\left(1-\cos^{2}\left(\left(\frac{T}{2}-m\right)\theta\right)\right)\ge\frac{T^{2}}{13ek}~.
\]

\newpage

\paragraph{Odd $T$:}
In this case we set 
$
\forall i\in\left[T-1\right]:\,\,
\theta_{i,i+1}=\begin{cases}
\theta & i<\frac{T-1}{2}\\
-\frac{\theta}{2} & i=\frac{T-1}{2}\,\,\textrm{or}\,\,i=\frac{T+1}{2}\\
-\theta & i>\frac{T+1}{2}
\end{cases}~
$,
and since $\cos\left(\nicefrac{\theta}{2}\right)\ge\cos\left(\theta\right)$
(for small $\theta$, as we choose later) we have

\[
\frac{1}{T}\left(\prod_{i=1}^{k-1}\cos^{2}\theta_{\tau\left(i,i+1\right)}\right)^{k-1}\sum_{m=1}^{T}\left(1-\cos^{2}\theta_{m,T}\right)\ge\frac{1}{T}\left(\cos^{2}\theta\right)^{k-1}\sum_{m=1}^{T}\left(1-\cos^{2}\theta_{m,T}\right)~.
\]
Next,
\[
\sum_{m=1}^{T}\left(1-\cos^{2}\theta_{m,T}\right)=1-\cos^{2}\left(\left(\frac{T-2}{2}\right)\theta\right)+\sum_{m=1}^{T-1}\left(1-\cos^{2}\left(\left(\frac{T-1}{2}-m\right)\theta\right)\right)~.
\]
Therefore we get
\begin{align}
 & \frac{1}{T}\left(\cos^{2}\theta\right)^{k-1}\sum_{m=1}^{T}\left(1-\cos^{2}\theta_{m,T}\right)\nonumber \\
= & \frac{1}{T}\left(\cos^{2}\theta\right)^{k-1}\left(1-\cos^{2}\left(\left(\frac{T-2}{2}\right)\theta\right)\right)\nonumber \\
 & +\frac{1}{T}\left(\cos^{2}\theta\right)^{k-1}\sum_{m=1}^{T-1}\left(1-\cos^{2}\left(\left(\frac{T-1}{2}-m\right)\theta\right)\right)\nonumber \\
= & \frac{1}{T}\left(\cos^{2}\theta\right)^{k-1}\left(1-\cos^{2}\left(\left(\frac{T-2}{2}\right)\theta\right)\right)\nonumber \\
 & +\frac{T-1}{T}\frac{1}{T-1}\left(\cos^{2}\theta\right)^{k-1}\sum_{m=1}^{T-1}\left(1-\cos^{2}\left(\left(\frac{T-1}{2}-m\right)\theta\right)\right)~.\label{eq:odd1}
\end{align}
We choose again $\theta=\sqrt{\nicefrac{1}{k-1}}$. For the first term
in Eq. (\ref{eq:odd1}) we have
\begin{align*}
 & \frac{1}{T}\left(\cos^{2}\theta\right)^{k-1}\left(1-\cos^{2}\left(\left(\frac{T-2}{2}\right)\theta\right)\right)\\
 & 
 =
 \frac{1}{T}\left(\cos^{2}\sqrt{\frac{1}{k-1}}\right)^{k-1}\left(1-\cos^{2}\left(\left(\frac{T-2}{2}\right)\sqrt{\frac{1}{k-1}}\right)\right)
 \\
 & 
 \ge
 \frac{1}{T}\left(\cos^{2}\sqrt{\frac{1}{k-1}}\right)^{k-1}\left(1-\cos^{2}\left(\frac{1}{2}\sqrt{\frac{1}{k-1}}\right)\right)
 \ge
 \frac{1}{4Tek}~.
\end{align*}
For the second term in Eq. (\ref{eq:odd1}) we use Eq. (\ref{eq:evenT})
and get
\begin{align*}
\frac{T-1}{T}\frac{1}{T-1}\left(\cos^{2}\theta\right)^{k-1}\sum_{m=1}^{T-1}\left(1-\cos^{2}\left(\left(\frac{T-1}{2}-m\right)\theta\right)\right)\\
\ge\frac{T-1}{T}\frac{1}{12ek}\left(\left(T-1\right)^{2}+2-\frac{3\left(T-1\right)^{4}+20\left(T-1\right)^{2}-8}{60\left(\left(T-1\right)^{2}-1\right)}\right)~.
\end{align*}
Therefore in Eq. (\ref{eq:odd1}) we have
\begin{align*}
 & \frac{1}{T}\left(\cos^{2}\theta\right)^{k-1}\sum_{m=1}^{T}\left(1-\cos^{2}\theta_{m,T}\right)\\
 & \ge\frac{1}{4Tek}+\frac{T-1}{T}\frac{1}{12ek}\left(\left(T-1\right)^{2}+2-\frac{3\left(T-1\right)^{4}+20\left(T-1\right)^{2}-8}{60\left(\left(T-1\right)^{2}-1\right)}\right)\\
 & =\frac{1}{ek}\left(\frac{1}{4T}+\frac{T-1}{12T}\left(\left(T-1\right)^{2}+2-\frac{3\left(T-1\right)^{4}+20\left(T-1\right)^{2}-8}{60\left(\left(T-1\right)^{2}-1\right)}\right)\right)\\
\left[T\ge3\right] & \ge\frac{T^{2}}{24ek}~,
\end{align*}
and thus for odd $T\ge3$ we have $F_{\tau,\coll}\left(k\right)\ge\nicefrac{T^{2}}{24ek}$,
which concludes the proof.

\bigskip
\bigskip
\bigskip

\paragraph{Regarding the rank.}
Notice that the lower bound we derived
used a construction consisting of tasks
of rank $d-1$.
On the other hand, 
the lower bound in
\thmref{thm:worst_case_T3}
should be correct even when $\maxrank < d-1$.
In such a case, we can always use an identical task collection 
using $T$ tasks spread on a $2$-d hyperplane,
and add to the tasks' data matrices any number of directions which are orthogonal to that hyperplane, but appear identically on \emph{all} tasks.
Then, any expression of the form 
$\norm{\bigprn{\I-\mP_m}\bigprn{\mP_T \cdots \mP_1}^n}$ will disregard these added directions, since they are shared across all tasks (hence they do not appear in $\I\!-\!\mP_m$).
The remaining hyperplane's behavior will remain unchanged as in the analysis above.
Thus our analysis above is valid for any task rank.
\end{proof}

\newpage

\section{Supplementary material: Random task orderings (\secref{sec:random})}
\label{app:proofs6}

\begin{recall}[\thmref{thm:expected_forgetting}]
Under the uniform i.i.d. task ordering $\tau$,
the worst-case expected forgetting  
after $k$ iterations is
\begin{align*}
    \sup_{
    \substack{
        \coll\in\mathcal{\coll}_{T}: \\
        \tfrac{1}{T}
        \tsum_{m=1}^{T} \rank(\X_m) = \avgrank
    }
    }
    \!\!\!\!\!\!\!\!\!\!\!\!\!\!\!\!\!\!\!
    \expfor_{\tau,\coll}
    \left(k\right)
    \le
    \frac{
    9\left(d-\avgrank\right)}{k}~.
\end{align*}
%
\end{recall}

\bigskip

\begin{proof}[Proof for \thmref{thm:expected_forgetting}] 
\label{prf:expected_forgetting}
Let $\coll=\left\{ \mP_{1},\dots,\mP_{T}\right\}$  be a given collection of $T$ arbitrary projection matrices sampled i.i.d. at each iteration from a uniform distribution.

As in \eqref{eq:linear_forgetting},
we first express the expected forgetting
on a task collection $\coll\in\Coll$,
defined 
in \defref{def:expected-forgetting},
in terms of the projection matrices
induced by the $T$ data matrices in $\coll$.
That is, we bound the expected forgetting by
\begin{align}
\expfor_{\tau,\coll} \left(k\right)
\!
\triangleq
\expectation_{\tau} 
\!
\bigg[
\frac{1}{k}
\sum_{\itr=1}^{k}
\!
\Bignorm{\X_{\tau(\itr)}\w_k-\y_{\tau(\itr)}}^2
\bigg]
\!
\le
\expectation_{\tau} 
\bigg[
\frac{1}{k}
\sum_{\itr=1}^{k}
\!
\Bignorm{
 \left(\I-\mP_{\tau(\itr)}\right)
{
\mP_{\tau\left(k\right)}\cdots
\mP_{\tau\left(1\right)}
}
}^2
\bigg]~.
\label{eq:expected-forgetting-projections}
\end{align}
\subsection{Detour: Rephrasing and bounding a more natural expression}
\subsubsection{Rephrasing}
Before we will bound the upper bound above,
we start by bounding a slightly different 
quantity which is easier to work with and can be seen as a straightforward MSE loss over $T$ tasks (as studied in NLMS papers, \eg \cite{slock1993convergence,sankaran2000convergence}).
Instead of averaging the forgetting (\ie residuals) over \emph{previously-seen} tasks,
we average over \emph{all} tasks in the collection $\coll$.
That is, we start by bounding.
\begin{align}
    \expectation_{\tau} 
    \bigg[
    \frac{1}{T}
    \sum_{m=1}^{T}
    \!
    \left\Vert \left(\I-\mP_{m}\right)
    \mP_{\tau\left(k\right)}\cdots\mP_{\tau\left(1\right)}
    \right\Vert_{2}^{2}\bigg]
    =
    \expectation_{\tau,\mP} 
    \left\Vert \left(\I-\mP\right)
    \mP_{\tau\left(k\right)}\cdots\mP_{\tau\left(1\right)}
    \right\Vert_{2}^{2}~,
    \label{eq:expected-average}
\end{align}
where $\mathbb{E}_{\mP}$ means averaging over all $T$ projection matrices w.r.t. a uniform distribution.

To further ease reading,
we slightly abuse notation
such that instead of explicitly sampling a task ordering $\tau$,
we simply uniformly sample $k+1$ orthogonal projection operators, $\mP_1,\mP_2,\dots,\mP_k$ and $\mP$ from 
some i.i.d. uniform distribution 
(\eg uniformly from a finite set of $T$ such operators).
Finally, we obtain the following expression:
\begin{align}
\mathbb{E}_{
\mP_{1},\dots,\mP_{k},\mP}\left\Vert \left(\I-\mP\right)\mP_{k}\cdots\mP_{1}\right\Vert _{2}^{2}~.
\label{eq:random-projections}
\end{align}

\bigskip
Next, we will bound \eqref{eq:random-projections} which is equivalent to \eqref{eq:expected-average}.
Then, we will use \eqref{eq:expected-average}
to bound the right hand side of
\eqref{eq:expected-forgetting-projections}
and conclude an upper bound for the expected forgetting, as required.

\newpage

\subsubsection{Bounding \eqref{eq:random-projections}}
We begin by using a known inequality between the spectral norm and the Frobenius norm, that is,
${\left\Vert \boldsymbol{A}\right\Vert _{2}^{2}\!=\!\lambda_{1}\left(\boldsymbol{A}^{*}\boldsymbol{A}\right)\le\trace\left(\boldsymbol{A}^{*}\boldsymbol{A}\right)\!=\left\Vert \boldsymbol{A}\right\Vert _{F}^{2}}$, 
and present an easier, but perhaps looser, surrogate quantity $f(k)$ which we will bound:
\begin{align}
\expectation_{\mP_{1},\dots,\mP_{k},\mP}\left\Vert \left(\I-\mP\right)\mP_{k}\cdots\mP_{1}\right\Vert _{2}^{2}	
\le
\expectation_{\mP_{1},\dots,\mP_{k},\mP}\trace\left[\mP_{1}\cdots\mP_{k}\left(\I-\mP\right)\mP_{k}\cdots\mP_{1}\right]
\triangleq
f\!\left(k\right)~.
\label{eq:trace_bound}
\end{align}
Key to our following derivations, is the fact that all the projection matrices including $\mP$ are identically distributed, allowing us to freely change variables. This property, together with the linearity of the expectation and trace operators, facilitates our derivations.

We now notice that the bound we got is \emph{non-increasing} in $k$, \ie
\begin{align*}
 f\!\left(k-1\right)-f\!\left(k\right)
 =&\expectation_{\mP_{1},\dots,\mP_{k-1},\mP}\tr\left[\mP_{1}\cdots\mP_{k-1}\left(\I-\mP\right)\mP_{k-1}\cdots\mP_{1}\right]
 -
\\
&\hskip 0cm
 \expectation_{\mP_{1},\dots,\mP_{k},\mP}\tr\left[\mP_{1}\cdots\mP_{k}\left(\I-\mP\right)\mP_{k}\cdots\mP_{1}\right]
 \\
\\
\left[\substack{\text{identically}\\
\text{distributed}
}
\right]	
=&\!\!\!\expectation_{\mP_{2},\dots,\mP_{k+1}}\tr\left[\mP_{2}\cdots\mP_{k}\left(\I-\mP_{k+1}\right)\mP_{k}\cdots\mP_{2}\right]
 -
\\
&\hskip 0cm
\expectation_{\mP_{1},\dots,\mP_{k+1}}\tr\left[\mP_{1}\cdots\mP_{k}\left(\I-\mP_{k+1}\right)\mP_{k}\cdots\mP_{1}\right]
\\
\\
\left[\text{linearity}\right] 
=&
\expectation_{\mP_{1},\dots,\mP_{k+1}}
\Bigg[\tr\bigg[\underbrace{\mP_{2}\cdots\mP_{k}\left(\I-\mP_{k+1}\right)\mP_{k}\cdots\mP_{2}}_{\triangleq\boldsymbol{M}}\bigg]
-
\\
&\hskip 1.95cm
\tr\bigg[\mP_{1}\underbrace{\mP_{2}\cdots\mP_{k}\left(\I-\mP_{k+1}\right)\mP_{k}\cdots\mP_{2}}_{\triangleq\boldsymbol{M}}\mP_{1}\bigg]
\Bigg]
\\
\\
&
=\expectation_{\mP_{1},\dots,\mP_{k},\mP}
\bigg[\underbrace{\tr\left[\boldsymbol{M}\right]-\tr\left[\mP_{1}\boldsymbol{M}\mP_{1}\right]}_{\ge0}\bigg]\ge0~,
\end{align*}
where in the last inequality we (again) used Von Neumann's trace inequality
to show that
$$
\tr\left(\mP_1\M\mP_1\right)
=
\tr\left(\mP_1\mP_1\M\right)
=
\tr
\bigprn{
\underbrace{\mP_1}_{\succeq 0}
\underbrace{\M}_{\succeq 0}
}
\le
\bignorm{\mP_1}_2
\tr\left(\M\right)
=
\tr\left(\M\right)~.
$$

Since the bounds are non-increasing,
we bound $f\!\left(k\right)$ at iteration k, using 
$f\!\left(1\right),\dots,f\!\left(k\right)$:
\begin{align*}
f\left(k\right)
&=
\min_{\itr\in\left[k\right]}f\left(\itr\right)
\le
\frac{1}{k}\sum_{\itr=1}^{k}
f\left(\itr\right)=
\frac{1}{k}\sum_{\itr=1}^{k}
\expectation_{\mP_{1},\dots,\mP_{\itr},\mP}\trace\left[\mP_{1}\cdots\mP{}_{\itr}\left(\I-\mP\right)\mP_{\itr}\cdots\mP_{1}\right]
\\
&=
\frac{1}{k}
\sum_{\itr=1}^{k}\left(\expectation_{\mP_{1},\dots,\mP_{\itr}}\trace\left[\mP_{1}\cdots\mP_{\itr}\cdots\mP_{1}\right]
-\!\!\!
\expectation_{\mP_{1},\dots,\mP_{\itr},\mP}\trace\left[\mP_{1}\cdots\mP{}_{\itr}\mP\mP_{\itr}\cdots\mP_{1}\right]\right)
\\
\left[\substack{\text{identically}\\
\text{distributed}
}
\right]	
&=
\frac{1}{k}\sum_{\itr=1}^{k}\left(\expectation_{\mP_{1},\dots,\mP_{\itr}}\trace\left[\mP_{1}\cdots\mP_{\itr}\cdots\mP_{1}\right]
-\!\!\!
\expectation_{\mP_{1},\dots,\mP_{\itr},\mP_{\itr+1}}\trace\left[\mP_{1}\cdots\mP{}_{\itr+1}\cdots\mP_{1}\right]\right)
\\
\left[\text{telescoping}\right]	&=
\frac{1}{k}\expectation_{\mP_{1}}\trace\left[\mP_{1}\right]-\underbrace{\expectation_{\mP_{1},\dots\mP_{k+1}}\trace\left[\mP_{1}\cdots\mP{}_{k+1}\cdots\mP_{1}\right]}_{\ge0,\text{ because \ensuremath{\mP_{1}\cdots\mP{}_{k+1}\cdots\mP_{1}\succeq\boldsymbol{0}}}}
\\
&
\le
\frac{1}{k}\expectation_{\mP}\rank\left(\mP\right)
=
\frac{1}{k}
\cdot
\frac{1}{T}
\sum_{m=1}^{T}
(d-\rank\left(\X_{m}\right))
\triangleq
\frac{d-\avgrank}{k}~,
\end{align*}
and conclude that
\begin{align*}
    \expectation_{\tau} 
    \bigg[
    \frac{1}{T}
    \sum_{m=1}^{T}
    \!
    \left\Vert \left(\I-\mP_{m}\right)
    \mP_{\tau\left(k\right)}\cdots\mP_{\tau\left(1\right)}
    \right\Vert_{2}^{2}\bigg]
    &
    \stackrel{
    \ref{eq:expected-average} \& \ref{eq:random-projections}
    }{\equiv}
\!\!\!
\expectation_{
\mP_{1},\dots,\mP_{k},\mP}\left\Vert \left(\I-\mP\right)\mP_{k}\cdots\mP_{1}\right\Vert _{2}^{2}
\\&
\stackrel{
\ref{eq:trace_bound}
}{\le}
f\!\left(k\right)
\le
\frac{d-\avgrank}{k}~.
\end{align*}

\newpage

\subsection{Back to the expected forgetting}
Recall
that our initial goal in \eqref{eq:expected-forgetting-projections} was to bound 
$
\expectation_{\tau} 
\Big[
\frac{1}{k}
\sum_{\itr=1}^{k}
\!
\bignorm{
 \left(\I-\mP_{\tau(\itr)}\right)
{
\mP_{\tau\left(k\right)}\cdots
\mP_{\tau\left(1\right)}
}
}^2
\Big]$.
Instead, we started by bounding the slightly different
quantities from \eqref{eq:expected-average}
and
(\ref{eq:random-projections}), 
\ie
we showed that
$
\expectation_{\tau,\mP} 
\left\Vert \left(\I-\mP\right)
\mP_{\tau\left(k\right)}\cdots\mP_{\tau\left(1\right)}
\right\Vert_{2}^{2}
\le
\tfrac{d-\avgrank}{k}$.
We now use the bound we proved to bound the original quantity of interest.
\vspace{-1mm}
\begin{align*}
&
	\mathbb{E}_{\tau}\!\left[\left\Vert \left(\boldsymbol{I}-\boldsymbol{P}_{\tau\left(t\right)}\right)\boldsymbol{P}_{\tau\left(k\right)}\cdots\boldsymbol{P}_{\tau\left(t\right)}\cdots\boldsymbol{P}_{\tau\left(1\right)}\right\Vert _{2}^{2}\right]
	\\
&
	=
	\mathbb{E}_{\tau}\!\left[\left\Vert \left(\boldsymbol{I}-\boldsymbol{P}_{\tau\left(t\right)}\right)\boldsymbol{P}_{\tau\left(k\right)}\cdots\boldsymbol{P}_{\tau\left(t+1\right)}\left(\boldsymbol{P}_{\tau\left(t\right)}+\boldsymbol{I}-\boldsymbol{I}\right)\boldsymbol{P}_{\tau\left(t-1\right)}\cdots\boldsymbol{P}_{\tau\left(1\right)}\right\Vert _{2}^{2}\right]
\\
\left[{\text{\propref{prop:norms}}}\right]
&
\le
2\bigg(\mathbb{E}_{\tau}\!\left[\left\Vert \left(\boldsymbol{I}-\boldsymbol{P}_{\tau\left(t\right)}\right)\boldsymbol{P}_{\tau\left(k\right)}\cdots\boldsymbol{P}_{\tau\left(t+1\right)}\left(\boldsymbol{I}-\boldsymbol{P}_{\tau\left(t\right)}\right)\boldsymbol{P}_{\tau\left(t-1\right)}\cdots\boldsymbol{P}_{1}\right\Vert _{2}^{2}\right]+
\\
&
	\,\,\,\,\,\,\,\,\,\,\,\,\,\,+\mathbb{E}_{\tau}\!\left\Vert \left(\boldsymbol{I}-\boldsymbol{P}_{\tau\left(t\right)}\right)\boldsymbol{P}_{\tau\left(k\right)}\cdots\boldsymbol{P}_{\tau\left(t+1\right)}\boldsymbol{P}_{\tau\left(t-1\right)}\cdots\boldsymbol{P}_{1}\right\Vert _{2}^{2}\bigg)
\\
\left[\text{i.i.d.}\right]	
&
=
2\bigg(\mathbb{E}_{\tau}\!\left\Vert \left(\boldsymbol{I}-\boldsymbol{P}_{\tau\left(t\right)}\right)\boldsymbol{P}_{\tau\left(k\right)}\cdots\boldsymbol{P}_{\tau\left(t+1\right)}\left(\boldsymbol{I}-\boldsymbol{P}_{\tau\left(t\right)}\right)\boldsymbol{P}_{\tau\left(t-1\right)}\cdots\boldsymbol{P}_{1}\right\Vert _{2}^{2}
+\\
&
	\,\,\,\,\,\,\,\,\,\,\,\,\,\,+
\underbrace{\mathbb{E}_{\tau,\boldsymbol{P}}\!\left\Vert \left(\boldsymbol{I}-\boldsymbol{P}\right)\boldsymbol{P}_{\tau\left(k-1\right)}\cdots\boldsymbol{P}_{1}\right\Vert _{2}^{2}}_{\le\nicefrac{(d-\avgrank)}{\left(k-1\right)}}\bigg)
\\
\left[\text{norms}\right]
&
\le
2
\left(\mathbb{E}_{\tau}\!\left[\left\Vert \left(\boldsymbol{I}-\boldsymbol{P}_{\tau\left(t\right)}\right)\boldsymbol{P}_{\tau\left(k\right)}\cdots\boldsymbol{P}_{\tau\left(t+1\right)}\right\Vert _{2}^{2}\left\Vert \left(\boldsymbol{I}-\boldsymbol{P}_{\tau\left(t\right)}\right)\boldsymbol{P}_{\tau\left(t-1\right)}\cdots\boldsymbol{P}_{1}\right\Vert _{2}^{2}\right]+\frac{d-\avgrank}{k-1}\right)
\\
&
=2\left(\mathbb{E}_{\tau,\boldsymbol{P}}\!\left[\left\Vert \left(\boldsymbol{I}-\boldsymbol{P}\right)\boldsymbol{P}_{\tau\left(k\right)}\cdots\boldsymbol{P}_{\tau\left(t+1\right)}\right\Vert _{2}^{2}\left\Vert \left(\boldsymbol{I}-\boldsymbol{P}\right)\boldsymbol{P}_{\tau\left(t-1\right)}\cdots\boldsymbol{P}_{1}\right\Vert _{2}^{2}\right]+\frac{d-\avgrank}{k-1}\right)~.
\end{align*}

Recall that both the spectral norms above are upper bounded by 1. We upper bound the norm consisting of less projections by 1 and keep the other one.

Assume $t>\nicefrac{\left(k+1\right)}{2}$, and thus $t-1>k-t$.
We get,	
$\mathbb{E}_{\tau,\boldsymbol{P}}\!\left\Vert \left(\boldsymbol{I}-\boldsymbol{P}\right)\boldsymbol{P}_{\tau\left(t-1\right)}\cdots\boldsymbol{P}_{1}\right\Vert _{2}^{2}\le\frac{d-\avgrank}{t-1}$.
Similarly, when $t\le\nicefrac{\left(k+1\right)}{2}$ and $t-1\le k-t$, we get
$
\mathbb{E}_{\tau,\boldsymbol{P}}\!\left\Vert \left(\boldsymbol{I}-\boldsymbol{P}\right)\boldsymbol{P}_{\tau\left(k\right)}\cdots\boldsymbol{P}_{\tau\left(t+1\right)}\right\Vert _{2}^{2}\le\frac{d-\avgrank}{k-t}$.

\newpage

Overall, we get a final bound of
\vspace{-1mm}
\begin{align*}
&\frac{1}{k}\sum_{t=1}^{k-1}\mathbb{E}_{\tau}\left[\left\Vert \left(\boldsymbol{I}-\boldsymbol{P}_{\tau\left(t\right)}\right)\boldsymbol{P}_{\tau\left(k\right)}\cdots\boldsymbol{P}_{\tau\left(t\right)}\cdots\boldsymbol{P}_{\tau\left(1\right)}\right\Vert _{2}^{2}\right]
\\
&
\le\frac{2}{k}
\bigg(\left(k-1\right)\frac{d-\avgrank}{k-1}+\sum_{t=1}^{\left\lfloor \left(k+1\right)/2\right\rfloor }\frac{d-\avgrank}{k-t}+\sum_{t=\left\lfloor \left(k+1\right)/2\right\rfloor +1}^{k-1}\frac{d-\avgrank}{t-1}\bigg)
\\
&
=\frac{2(d-\avgrank)}{k}
\bigg(1+
\!\!\!
\sum_{t=k-\left\lfloor \left(k+1\right)/2\right\rfloor }^{k-1}\!\!\!\nicefrac{1}{t}
+\!\!
\sum_{t=\left\lfloor \left(k+1\right)/2\right\rfloor }^{k-2}\!\!\!\nicefrac{1}{t}\bigg)
\le\frac{2(d-\avgrank)}{k}\bigg(1+2\!\sum_{t=\left\lfloor k/2\right\rfloor }^{k-1}\!\!\nicefrac{1}{t}\bigg)
\\
&
\stackrel{\left[k\ge2\right]}{\le}
\frac{2(d-\avgrank)}{k}\bigg(1+2+2\intop_{\left\lfloor k/2\right\rfloor }^{k-1}\frac{1}{t}dt\bigg)
\le
\frac{2(d-\avgrank)}{k}\left(3+2\left(\ln\left(k-1\right)-\ln\left(\left\lfloor k/2\right\rfloor \right)\right)\right)
\\
&
=
\frac{2(d-\avgrank)}{k}\left(3+2\ln\left(\frac{k-1}{\left\lfloor k/2\right\rfloor }\right)\right)\le
\frac{2(d-\avgrank)}{k}
\underbrace{\left(3+2\ln\left(2\right)\right)}_{\le 4.5}
\le\frac{9(d-\avgrank)}{k}~.
\end{align*}
\end{proof}

\newpage

\section{Extension to the average iterate (Remark~\ref{rmrk:average_iterate})}
\label{app:average_iterate}
Here we briefly demonstrate how our results from Sections~\ref{sec:cyclic} and \ref{sec:random} can be readily improved by considering the average iterate 
$\overline{\w}_{k}\triangleq\frac{1}{k}\sum_{\itr=1}^{k}\boldsymbol{w}_{\itr}$
instead of the last iterate $\w_k$.
The average iterate is known to be easier to analyze, and yields generally stronger results when analyzing the SGD algorithm \cite{varre2021last} and the Kaczmard method \cite{morshed2020accelerated}. 
Of course, working with the average iterate requires maintaining a running mean parameter vector, which is generally more memory consuming.

\subsection{Proof sketch for the cyclic setting}
As a proof of concept,
we prove better bounds for an ``easier'' iterate that averages only iterates at the end of cycles,
\ie for $k\!=\!nT$ we define
$\overline{\boldsymbol{w}}_{n}\!=\!\frac{1}{n}\sum_{n'=1}^{n}\!\boldsymbol{w}_{Tn'}$.
For the first task, we get
\begin{align*}
\left\Vert \boldsymbol{X}_{1}\left(\overline{\boldsymbol{w}}_{n}-\boldsymbol{w}^{\star}\right)\right\Vert ^{2}	
&
=
\left\Vert \boldsymbol{X}_{1}
\Bigprn{
\frac{1}{n}\sum_{n'=1}^{n}\boldsymbol{w}_{Tn'}-\boldsymbol{w}^{\star}
}
\right\Vert ^{2}
	=\frac{1}{n^{2}}\left\Vert \boldsymbol{X}_{1}\sum_{n'=1}^{n}\left(\boldsymbol{P}_{T}\cdots\boldsymbol{P}_{1}\right)^{n'}\left(\cancel{\boldsymbol{w}_{0}}-\boldsymbol{w}^{\star}\right)\right\Vert ^{2}
	\\
	&
	\le
	\frac{1}{n^{2}}\left\Vert \sum_{n'=1}^{n}\left(\boldsymbol{I}-\boldsymbol{P}_{1}\right)\left(\boldsymbol{P}_{T}\cdots\boldsymbol{P}_{1}\right)^{n'}\boldsymbol{w}^{\star}\right\Vert ^{2}
	\\
	\explain{\text{\lemref{lem:square_ineq}}}
	&
	\le\frac{1}{n}\sum_{n'=1}^{n}\left\Vert \left(\boldsymbol{I}-\boldsymbol{P}_{1}\right)\left(\boldsymbol{P}_{T}\cdots\boldsymbol{P}_{1}\right)^{n'}\boldsymbol{w}^{\star}\right\Vert ^{2}
	\\&
	=\frac{1}{n}\sum_{n'=1}^{n}\left(\left\Vert \left(\boldsymbol{P}_{T}\cdots\boldsymbol{P}_{1}\right)^{n'}\boldsymbol{w}^{\star}\right\Vert ^{2}-\left\Vert \boldsymbol{P}_{1}\left(\boldsymbol{P}_{T}\cdots\boldsymbol{P}_{1}\right)^{n'}\boldsymbol{w}^{\star}\right\Vert ^{2}\right)
	\\
	&
	\le\frac{1}{n}\sum_{n'=1}^{n}\left(\left\Vert \left(\boldsymbol{P}_{T}\cdots\boldsymbol{P}_{1}\right)^{n'}\boldsymbol{w}^{\star}\right\Vert ^{2}-\left\Vert \left(\boldsymbol{P}_{T}\cdots\boldsymbol{P}_{1}\right)^{n'+1}\boldsymbol{w}^{\star}\right\Vert ^{2}\right)
	\\
	&
	=\frac{1}{n}\Bigprn{ \underbrace{\left\Vert\left(\boldsymbol{P}_{T}\cdots\boldsymbol{P}_{1}\right)\boldsymbol{w}^{\star}\right\Vert ^{2}}_{\le 1}
	-
	\underbrace{\left\Vert \left(\boldsymbol{P}_{T}\cdots\boldsymbol{P}_{1}\right)^{n+1}\boldsymbol{w}^{\star}\right\Vert ^{2}}_{\ge 0}
	}
	\le\frac{1}{n}~.
\end{align*}
Similarly, for a general task $m\in\cnt{T}$, we have
\begin{align*}
\left\Vert \boldsymbol{X}_{m}\left(\overline{\boldsymbol{w}}_{n}-\boldsymbol{w}^{\star}\right)\right\Vert ^{2}	
&
	\le\frac{1}{n}\sum_{n'=1}^{n}\left\Vert \left(\boldsymbol{I}-\boldsymbol{P}_{m}\right)\left(\boldsymbol{P}_{T}\cdots\boldsymbol{P}_{1}\right)^{n'}\boldsymbol{w}^{\star}\right\Vert ^{2}
	\\
	&
	=\frac{1}{n}\sum_{n'=1}^{n}\left(\left\Vert \left(\boldsymbol{P}_{T}\cdots\boldsymbol{P}_{1}\right)^{n'}\boldsymbol{w}^{\star}\right\Vert ^{2}-\left\Vert \boldsymbol{P}_{m}\left(\boldsymbol{P}_{T}\cdots\boldsymbol{P}_{1}\right)^{n'}\boldsymbol{w}^{\star}\right\Vert ^{2}\right)
	\\
\left[\text{\ref{lem:bound_on_projection_to_m_task}}\right]
	&	
	\le\frac{m}{n}\sum_{n'=1}^{n}\left(\left\Vert \left(\boldsymbol{P}_{T}\cdots\boldsymbol{P}_{1}\right)^{n'}\boldsymbol{w}^{\star}\right\Vert ^{2}-\left\Vert \left(\boldsymbol{P}_{T}\cdots\boldsymbol{P}_{1}\right)^{n'+1}\boldsymbol{w}^{\star}\right\Vert ^{2}\right)
	\\
	&
	=\frac{m}{n}\left(\left\Vert \left(\boldsymbol{P}_{T}\cdots\boldsymbol{P}_{1}\right)\boldsymbol{w}^{\star}\right\Vert ^{2}-\left\Vert \left(\boldsymbol{P}_{T}\cdots\boldsymbol{P}_{1}\right)^{n+1}\boldsymbol{w}^{\star}\right\Vert ^{2}\right)
	\le\frac{m}{n}~.
\end{align*}
Overall, we get the following bound on the forgetting in the cyclic setting:
\begin{align*}
\begin{split}
    F_{\tau,\coll}(\overline{\w}_{n})
    =
    \frac{1}{T}
    \sum_{m=1}^{T}
    \left\Vert \boldsymbol{X}_{m}\left(\overline{\boldsymbol{w}}_{n}-\boldsymbol{w}^{\star}\right)\right\Vert ^{2}	
    \le
    \frac{1}{T}
    \sum_{m=1}^{T-1}\frac{m}{n}
    =
    \frac{T-1}{2n}
    \le
    \frac{T^2}{2k}
    	~.
\end{split}
\end{align*}

\newpage

\subsection{Proof sketch for the random setting}
For this case, we will demonstrate how one can easily bound
a similar expected forgetting to the one in
\eqref{eq:random-projections},
by $\tfrac{1}{k}$ instead of $\tfrac{d-\avgrank}{k}$.

Plugging in the average iterate we get,
\begin{align*}
&
\mathbb{E}_{
\mP_{1},\dots,\mP_{k},\mP}
\Bignorm{\left(\I-\mP\right)
\bigprn{\overline{\w}_{k}-\teacher}
}_{2}^{2}
\\
&
=
\mathbb{E}_{
\mP_{1},\dots,\mP_{k},\mP}
\Bignorm{\left(\I-\mP\right)
\bigprn{
\teacher
-
\frac{1}{k}\sum_{\itr=1}^{k}\boldsymbol{w}_{\itr}}
}_{2}^{2}
\\
&
=
\frac{1}{k^2}
\mathbb{E}_{
\mP_{1},\dots,\mP_{k},\mP}
\Bignorm{\left(\I-\mP\right)
\sum_{\itr=1}^{k}\bigprn{\boldsymbol{w}_{\itr}-\teacher}
}_{2}^{2}
\\
\explain{\text{\lemref{lem:square_ineq}}}
&
\le
\frac{1}{k}
\sum_{\itr=1}^{k}
\mathbb{E}_{
\mP_{1},\dots,\mP_{k},\mP}
\Bignorm{\left(\I-\mP\right)
\bigprn{\boldsymbol{w}_{\itr}-\teacher}
}_{2}^{2}
\\
\explain{\text{\lemref{total_contraction}}}
&
=
\frac{1}{k}
\sum_{\itr=1}^{k}
\mathbb{E}_{
\mP_{1},\dots,\mP_{k},\mP}
\Bignorm{\left(\I-\mP\right)\mP_{t}\cdots\mP_{1}
\teacher
}_{2}^{2}
\\
\explain{\text{\propref{prop:projections}}}
&
=
\frac{1}{k}
\sum_{\itr=1}^{k}
\mathbb{E}_{
\mP_{1},\dots,\mP_{k},\mP}
\left[
\Bignorm{\mP_{t}\cdots\mP_{1}
\teacher
}_{2}^{2}
-
\Bignorm{\mP\mP_{t}\cdots\mP_{1}
\teacher
}_{2}^{2}
\right]
\\
\explain{\text{i.i.d.}}
&
=
\frac{1}{k}
\sum_{\itr=1}^{k}
\mathbb{E}_{
\mP_{1},\dots,\mP_{k}}
\left[
\Bignorm{\mP_{t}\cdots\mP_{1}
\teacher
}_{2}^{2}
-
\Bignorm{\mP_{t+1}\cdots\mP_{1}
\teacher
}_{2}^{2}
\right]
\\
\explain{\text{telescoping}}
&
=
\frac{1}{k}
\mathbb{E}_{
\mP_{1},\dots,\mP_{k}}
\left[
\Bignorm{\mP_{1}
\teacher
}_{2}^{2}
-
\Bignorm{\mP_{k+1}\cdots\mP_{1}
\teacher
}_{2}^{2}
\right]
\le
\frac{1}{k}
~.
\end{align*}

\newpage

\section{Additional related works}
\label{app:related}

\subsection{Continual learning}
\paragraph{Practical methods for continual learning.}

Like we explained in \secref{sec:related}, algorithmic approaches for preventing catastrophic forgetting roughly partition into three categories:
\begin{enumerate}
     \item \emph{Memory-based replay approaches} actively repeat examples from previous tasks to avoid forgetting. 
     Some store examples from observed tasks 
     (\eg \citet{robins1995rehearsal}),
     while others replay \emph{generated} synthetic data from previous tasks to the main model
     (\eg \citet{shin2017generativeReplay}).
    
 \item
\emph{Regularization approaches} limit the plasticity of the learned model in order to enhance its stability.
One can perform regularization in the \emph{parameter space}, 
to protect parameters that are important to previous tasks
(\eg \citet{kirkpatrick2017overcoming,nguyen2017variational, zeno2021onlineVariational}),
or in the \emph{function space}, 
to protect the \emph{outputs} of previous tasks
(\eg \citet{li2017lwf, lopez2017GEM, farajtabar2020OGD}).
Few recent works (\citet{benzing2021unifying_regularization},
\citet{lubana2021regularization}) focus on obtaining a better understanding of popular quadratic regularization techniques like EWC \citep{kirkpatrick2017overcoming},
MAS \citep{aljundi2018memory},
and SI \citep{zenke2017continual}.

\item \emph{Parameter isolation approaches} allocate different subsets of parameters to different tasks. 
This can be done by expanding the model upon seeing a new task (\eg \citet{yoon2018lifelong}),
or by compressing the existing architecture after each task, to free parameters for future tasks (\eg \citet{mallya2018packnet,schwarz2021powerpropagation}).
\end{enumerate}

\paragraph{Understanding customary training techniques}
Some recent papers also test the effect of common deep learning training techniques that change the optimization dynamics and/or have an implicit regularizing implications.
For instance, \citet{Mirzadeh2020regimes}
study how different training regimes affect the loss landscape geometry, thus influencing the overall degree of forgetting in continual learning settings.
\citet{goodfellow2013empirical} were the first to suggest that training with the \emph{dropout} technique can be beneficial to remedy \cf in continual learning settings.
\citet{delange2021continual} also show that dropout is fruitful in many continual learning methods (some of which are mentioned above).
They point out that it mainly improves the initial performance on learned tasks (by mitigating overfitting), but leads to an increased amount of forgetting when learning later tasks.

Optimization hyperparameters are known to change the geometry of the loss landscape and affect generalization (see for instance \citet{Jastrzebski2017minima}).
\citet{Mirzadeh2020regimes} point out that 
most empirical papers in this field use small batch sizes and SGD with learning rate decay.
They also discuss how a large learning rate increases the plasticity of deep models, thus having an ill-effect on forgetting. 

Finally, \citet{mirzadeh2021wide} pointed out that architectural decisions also have implications on catastrophic forgetting, and specifically observe that wide neural networks tend to forget less catastrophically.


\subsection{Wider scope}
Finally, we note that 
similar questions to those we ask here and that are asked generally in the continual learning paradigm,
are often asked on other fields, such as multi-task learning, meta learning, transfer learning, 
curriculum learning \cite{weinshall2020theory}, 
and online learning.
Despite the similarities, there are differences though.
For example, in multi-task learning, the data from all tasks are simultaneously available, in transfer learning the goal to adapt models for a new target tasks and preserving performance on old source task is not a priority, and in online learning, typically, training data for previously seen tasks are assumed to be available in sequentially adapting to data from new tasks. 
However, a more detailed review of these related areas is beyond the scope of this work.

\end{document}